\documentclass[english]{article}
\usepackage[T1]{fontenc}
\usepackage{textcomp}
\usepackage[utf8]{inputenc}
\usepackage{xcolor}
\usepackage{babel}
\usepackage{array}
\usepackage{float}
\usepackage{wrapfig}
\usepackage{booktabs}
\usepackage{units}
\usepackage{url}
\usepackage{bm}
\usepackage{multirow}
\usepackage{amsmath}
\usepackage{amsthm}
\usepackage{amssymb}
\usepackage{graphicx}
\usepackage{rotating}
\usepackage[authoryear,round]{natbib}
\PassOptionsToPackage{normalem}{ulem}
\usepackage{ulem}
\usepackage[pdfusetitle,
 bookmarks=true,bookmarksnumbered=false,bookmarksopen=false,
 breaklinks=false,pdfborder={0 0 1},backref=false,colorlinks=false]
 {hyperref}

\makeatletter

\let\SF@@footnote\footnote
\def\footnote{\ifx\protect\@typeset@protect
    \expandafter\SF@@footnote
  \else
    \expandafter\SF@gobble@opt
  \fi
}
\expandafter\def\csname SF@gobble@opt \endcsname{\@ifnextchar[
  \SF@gobble@twobracket
  \@gobble
}
\edef\SF@gobble@opt{\noexpand\protect
  \expandafter\noexpand\csname SF@gobble@opt \endcsname}
\def\SF@gobble@twobracket[#1]#2{}
\providecommand{\tabularnewline}{\\}
\floatstyle{ruled}
\newfloat{algorithm}{tbp}{loa}
\providecommand{\algorithmname}{Algorithm}
\floatname{algorithm}{\protect\algorithmname}

\theoremstyle{plain}
\newtheorem{lem}{\protect\lemmaname}
\theoremstyle{definition}
\newtheorem{defn}{\protect\definitionname}
\theoremstyle{plain}
\newtheorem{thm}{\protect\theoremname}
\theoremstyle{plain}
\newtheorem*{lem*}{\protect\lemmaname}
\theoremstyle{plain}
\newtheorem*{thm*}{\protect\theoremname}

\usepackage[accepted]{tmlr}

\usepackage{amsmath,amsfonts,bm}

\def\eqref#1{equation~\ref{#1}}

\def\1{\bm{1}}

\DeclareMathAlphabet{\mathsfit}{\encodingdefault}{\sfdefault}{m}{sl}
\SetMathAlphabet{\mathsfit}{bold}{\encodingdefault}{\sfdefault}{bx}{n}

\usepackage[utf8]{inputenc} 
\usepackage[T1]{fontenc}    
\usepackage{url}
\usepackage{amsfonts}
\usepackage{nicefrac}
\usepackage{xcolor}
\usepackage{verbatim}
\usepackage{pifont}
\usepackage{microtype}      
\usepackage{listings}

\title{Reinforcement Learning for Causal Discovery\\without Acyclicity Constraints}

\author{\name Bao Duong \email b.duong@deakin.edu.au \\
      \addr Applied Artificial Intelligence Institute (A$^2$I$^2$), Deakin University
      \AND
      \name Hung Le \email thai.le@deakin.edu.au \\
	  \addr Applied Artificial Intelligence Institute (A$^2$I$^2$), Deakin University
      \AND
      \name Biwei Huang \email bih007@ucsd.edu\\
      \addr University of California, San Diego
      \AND
      \name Thin Nguyen \email thin.nguyen@deakin.edu.au \\
	  \addr Applied Artificial Intelligence Institute (A$^2$I$^2$), Deakin University
}

\usepackage{amsmath}
\usepackage{amsthm}
\usepackage{algorithm}
\let\classAND\AND
\let\AND\relax
\usepackage{algorithmic}

\let\AND\classAND
\AtBeginEnvironment{algorithmic}{\let\AND\algoAND}

\definecolor{deepred}{rgb}{0.6,0,0}
\definecolor{codegreen}{rgb}{0,0.6,0}
\definecolor{codegray}{rgb}{0.5,0.5,0.5}
\definecolor{codepurple}{rgb}{0.58,0,0.82}
\definecolor{backcolour}{rgb}{0.95,0.95,0.92}
\definecolor{darkpastelgreen}{rgb}{0.01, 0.75, 0.24}

\lstdefinestyle{mystyle}{
    backgroundcolor=\color{backcolour},   
    commentstyle=\color{codegreen},
    keywordstyle=\color{magenta},
    numberstyle=\tiny\color{codegray},
    stringstyle=\color{codepurple},
    basicstyle=\ttfamily\scriptsize,
	emph={Vec2DAG},          
	emphstyle=\color{deepred},    
    breakatwhitespace=false,         
    breaklines=true,                 
    captionpos=b,                    
    keepspaces=true,                 
    numbers=left,                    
    numbersep=5pt,                  
    showspaces=false,                
    showstringspaces=false,
    showtabs=false,                  
    tabsize=4
}

\newcommand{\cmark}{{\color{darkpastelgreen}\ding{51}}}
\newcommand{\xmark}{{\color{deepred}\ding{55}}}
\newcommand{\indep}{\perp \!\!\! \perp}

\usepackage{etoc}
\etocsettocdepth{subsection}

\def\month{04}  
\def\year{2025} 
\def\openreview{\url{https://openreview.net/forum?id=sNzBi8rZTy}}

\makeatother

\providecommand{\definitionname}{Definition}
\providecommand{\lemmaname}{Lemma}
\providecommand{\theoremname}{Theorem}

\begin{document}
\maketitle

\global\long\def\ours{\mathbf{\mathrm{\mathbf{ALIAS}}}}%

\global\long\def\vecdag{\mathbf{\mathrm{\mathbf{Vec2DAG}}}}%

\global\long\def\pa{\mathrm{pa}}%

\global\long\def\de{\mathrm{de}}%

\global\long\def\an{\mathrm{an}}%

\begin{abstract}
Recently, reinforcement learning (RL) has proved a promising alternative
for conventional local heuristics in score-based approaches to learning
directed acyclic causal graphs (DAGs) from observational data. However,
the intricate acyclicity constraint still challenges the efficient
exploration of the vast space of DAGs in existing methods. In this
study, we introduce $\ours$ (reinforced d\uline{A}g \uline{L}earning
w\uline{I}thout \uline{A}cyclicity con\uline{S}traints), a novel approach
to causal discovery powered by the RL machinery. Our method features
an efficient policy for generating DAGs in just a single step with
an optimal quadratic complexity, fueled by a novel parametrization
of DAGs that directly translates a continuous space to the space of
all DAGs, bypassing the need for explicitly enforcing acyclicity constraints.
This approach enables us to navigate the search space more effectively
by utilizing policy gradient methods and established scoring functions.
In addition, we provide compelling empirical evidence for the strong
performance of $\ours$ in comparison with state-of-the-arts in causal
discovery over increasingly difficult experiment conditions on both
synthetic and real datasets. Our implementation is provided at \url{https://github.com/baosws/ALIAS}.

\end{abstract}

\section{Introduction\label{sec:Introduction}}

The knowledge of causal relationships is crucial to understanding
the nature in many scientific sectors \citep{Sachs_etall_05Causal,Hunermund_Bareinboim_2023Causal,Cao_etal_19Causal}.
This is especially relevant in intricate situations where randomized
experiments are impractical, and therefore, over the last decades,
it has motivated the development of causal discovery methods that
aim to infer cause-effect relationships from purely passive data.
Causal discovery is typically formulated as finding the directed acyclic
graph (DAG) representing the causal model that most likely generated
the observed data. Among the broad literature, score-based methods
are one of the most well-recognized approaches, which assigns each
possible DAG $\mathcal{G}$ a ``score'' $\mathcal{S}\left(\mathcal{D},\mathcal{G}\right)$
quantifying how much it can explain the observed data $\mathcal{D}$,
and then optimize the score over the space of DAGs: 
\begin{equation}
\mathcal{G}^{\star}=\underset{\mathcal{\mathcal{G}\in\text{DAGs}}}{\arg\max}\ \mathcal{S}\left(\mathcal{D},\mathcal{G}\right).\label{eq:score-based}
\end{equation}
 Solving this optimization problem is generally NP-hard \citep{Chickering_1996Learning},
due to the \textit{huge combinatorial search space} that grows super-exponentially
with the number of variables \citep{Robinson_77Counting} and the
\textit{intricate acyclicity constraint} that is difficult to characterize
and maintain efficiently because of its combinatorial nature. Most
methods therefore resort to local heuristics, such as GES \citep{Chickering_02Optimal}
which gradually adds edges into a graph one-by-one while laboriously
maintaining acyclicity. With the introduction of soft DAG characterizations
\citep{Zheng_etal_18DAGs,Yu_etal_2019Dag,Zhang_etal_2022Truncated,Bello_etal_22dagma},
the combinatorial optimization problem above is relaxed to a continuous
optimization problem, allowing for exploring graphs more effectively,
as multiple edges can be added or removed simultaneously in an update.
Alternatively, interventional causal discovery methods exploit available
interventional data to identify the causal graph \citep{Hauser_12Characterization,Brouillard_etal_20Differentiable,Lippe_etal_22Efficient}.
However, our focus in this study is the challenging observational
causal discovery setting where no interventional data is accessible.

Recently, reinforcement learning (RL) has emerged into score-based
causal discovery \citep{Zhu_etal_2020Causal,Wang_etal_2021Ordering,Yang_etal_2023Reinforcement,Yang_etal_2023Causal}
as the improved search strategy, thanks to its exploration and exploitation
abilities. However, existing RL-based methods handle acyclicity either
by fusing the soft DAG regularization from \citet{Zheng_etal_18DAGs}
into the reward \citep{Zhu_etal_2020Causal}, which wastes time for
exploring non-DAGs but still does not prohibit all cycles \citep{Wang_etal_2021Ordering},
or designing \textit{autoregressive} \textit{policies} \citep{Wang_etal_2021Ordering,Deleu_etal_2022Bayesian,Yang_etal_2023Reinforcement,Yang_etal_2023Causal,Deleu_etal_24Joint}
that hinder parallel DAG generation and necessitates learning the
transition policies over a multitude of discrete state-action combinations.

\begin{table}
\caption{\textbf{Positioning $\protect\ours$ among the score-based causal
discovery literature.}\label{tab:Conceptual-comparison}}

\begin{centering}
\resizebox{\columnwidth}{!}{%
\begin{tabular}{ccccccccc}
\toprule 
\textbf{Search} & \multirow{2}{*}{\textbf{Method (year)}} & \textbf{Search} & \textbf{Generation} & \textbf{Generation} & \multirow{2}{*}{\textbf{Constraint$^{\ddagger}$}} & \textbf{Acyclicity} & \textbf{Nonlinear} & \textbf{Differentiable score}\tabularnewline
\textbf{type$^{\dagger}$} &  & \textbf{space} & \textbf{steps} & \textbf{complexity} &  & \textbf{assurance$^{\flat}$} & \textbf{data} & \textbf{not required}\tabularnewline
\midrule
\multirow{6}{*}{\begin{turn}{90}
Local
\end{turn}} & GES (2002) \citep{Chickering_02Optimal} & \textbf{\color{darkpastelgreen}DAGs} & \multirow{6}{*}{-} & \multirow{6}{*}{-} & Hard & \cmark & \cmark & \cmark\tabularnewline
 & NOTEARS (2020) \citep{Zheng_etal_18DAGs,Zheng_etal_20Learning} & Graphs &  &  & Soft & \xmark & \cmark & \xmark\tabularnewline
 & NOCURL (2021) \citep{Yu_etal_2021Dags} & \textbf{\color{darkpastelgreen}DAGs} &  &  & \textbf{\color{darkpastelgreen}None} & \cmark & \xmark & \xmark\tabularnewline
 & DAGMA (2022) \citep{Bello_etal_22dagma} & Graphs &  &  & Soft & \xmark & \cmark & \xmark\tabularnewline
 & BaDAG (2023) \citep{Annadani_etal_2023Bayesdag} & \textbf{\color{darkpastelgreen}DAGs} &  &  & \textbf{\color{darkpastelgreen}None} & \cmark & \cmark & \xmark\tabularnewline
 & COSMO (2024) \citep{Massidda_etal_2023Constraint} & Graphs &  &  & \textbf{\color{darkpastelgreen}None} & \xmark & \cmark & \xmark\tabularnewline
\midrule
\multirow{7}{*}{\begin{turn}{90}
\textbf{\color{darkpastelgreen}Global}
\end{turn}} & RL-BIC (2020) \citep{Zhu_etal_2020Causal} & Graphs & \textbf{\color{darkpastelgreen}Single} & \textbf{\color{darkpastelgreen}Quadratic} & Soft & \xmark & \cmark & \cmark\tabularnewline
 & BCD-Nets (2021) \citep{Cundy_etal_2021Bcd} & \textbf{\color{darkpastelgreen}DAGs} & \textbf{\color{darkpastelgreen}Single} & Cubic & \textbf{\color{darkpastelgreen}None} & \cmark & \xmark & \xmark\tabularnewline
 & CORL (2021) \citep{Wang_etal_2021Ordering} & Orderings & Multiple (Autoregressive) & Cubic & Hard & \cmark & \cmark & \cmark\tabularnewline
 & DAG-GFN (2022) \citep{Deleu_etal_2022Bayesian} & \textbf{\color{darkpastelgreen}DAGs} & Multiple (Autoregressive) & Cubic & Hard & \cmark & \xmark & \cmark\tabularnewline
 & GARL (2023) \citep{Yang_etal_2023Causal} & Orderings & Multiple (Autoregressive) & Cubic & Hard & \cmark & \cmark & \cmark\tabularnewline
 & RCL-OG (2023) \citep{Yang_etal_2023Reinforcement} & Orderings & Multiple (Autoregressive) & Cubic & Hard & \cmark & \cmark & \cmark\tabularnewline
\cmidrule{2-9}
 & $\ours$ (Ours) & \textbf{\color{darkpastelgreen}DAGs} & \textbf{\color{darkpastelgreen}Single} & \textbf{\color{darkpastelgreen}Quadratic} & \textbf{\color{darkpastelgreen}None} & \cmark & \cmark & \cmark\tabularnewline
\midrule
 & \multicolumn{8}{l}{{\footnotesize$^{\dagger}$Local methods start with an initial graph
and update it every iteration, while global methods typically concern
with DAG generation parameters.}}\tabularnewline
 & \multicolumn{8}{l}{{\footnotesize$^{\ddagger}$Methods with Hard constraints explicitly
identify and discard the actions that lead to cycles, while Soft constraints
refer to the use of DAG regularizers.}}\tabularnewline
 & \multicolumn{8}{l}{{\footnotesize$^{\flat}$Methods that guarantee acyclicity only in
an annealing limit are considered as do not ensure acyclicity.}}\tabularnewline
\bottomrule
\end{tabular}}
\par\end{centering}
\vspace{-5mm}
\end{table}

In this study, we address the aforementioned limitations of score-based
causal discovery methods with a novel RL approach, named $\ours$
(reinforced d\textbf{\uline{A}}g Learning w\textbf{\uline{I}}thout
\textbf{\uline{A}}cyclicity con\textbf{\uline{S}}traints). Our approach
employs a generative policy that is capable of generating DAGs in
a single-step fashion without any acyclicity regularization or explicit
acyclicity maintenance. This enables us to effectively explore and
exploit the full DAG space with arbitrary score functions, rather
than the restricted ordering space. Specifically, we make the following
contributions in this study:

\begin{figure}[t]
\centering{}\resizebox{1\columnwidth}{!}{%
\begin{tabular}{>{\centering}p{0.5\textwidth}>{\centering}p{0.5\textwidth}}
\includegraphics[width=1\linewidth]{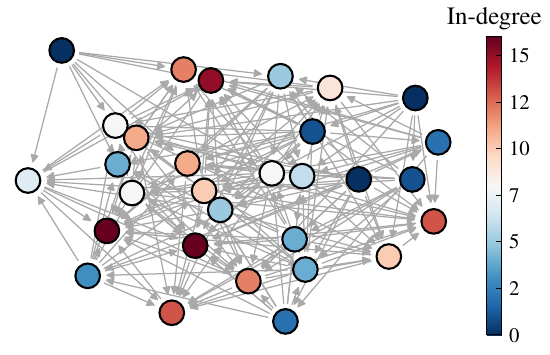} & \includegraphics[width=1\linewidth]{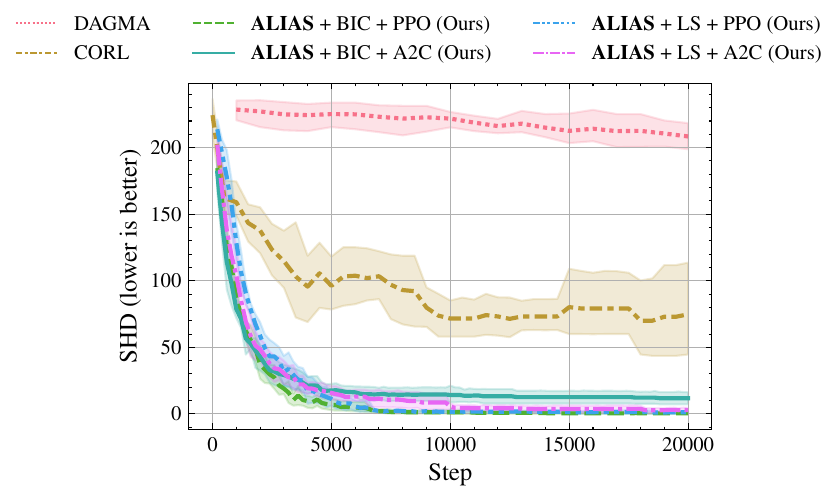}\tabularnewline
(a) & (b)\tabularnewline
\includegraphics[width=0.9\linewidth]{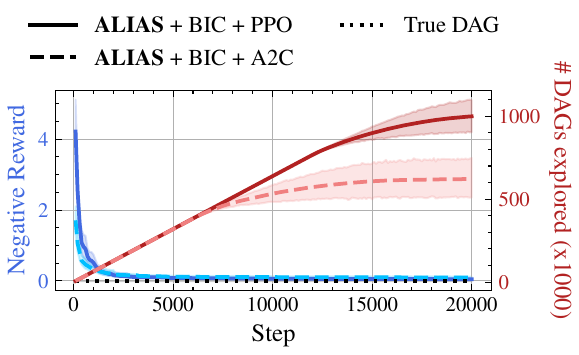} & \includegraphics[width=0.9\linewidth]{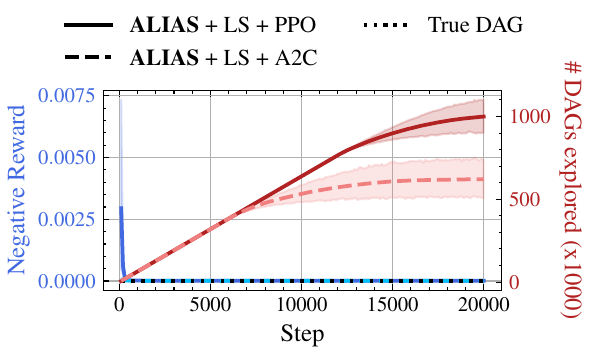}\tabularnewline
(c) & (d)\tabularnewline
\end{tabular}}\caption{Using merely observational data, the proposed $\protect\ours$ method
\textbf{correctly identifies all edges} of (a) a very complex causal
dataset with extremely dense connections (linear-Gaussian data with
Erd\H{o}s-R\'{e}nyi graph of 30 nodes and expected in-degree of
8). (b) We evaluate DAG learning performance in Structural Hamming
Distance (SHD, lower is better) on 5 of such datasets with respect
to the first $20\,000$ training steps of four $\protect\ours$ variants,
by combining scoring functions Bayesian Information Criterion (BIC)
\& Least Squares (LS) with RL methods PPO \citep{Schulman_etal_2017Proximal}
\& A2C \citep{Mnih_etal_2016Asynchronous}, in comparison with the
best baselines in this setting, namely CORL \citep{Wang_etal_2021Ordering}
and DAGMA \citep{Bello_etal_22dagma} (as evaluated in Section~\ref{subsec:Linear-Data}).
The best method in this scenario is our $\protect\ours+\text{BIC}+\text{PPO}$
variant with zero SHD at the end of the learning process. (c) \& (d)
For both scores, our method's rewards always approach those of the
ground truth DAGs very sharply, which is made possible largely thanks
to our efficient DAG parameterization, as well as the continuous exploitation
and exploration of the RL algorithms, especially PPO.\label{fig:learning-curve}}
\end{figure}

\vspace{-2mm}

\begin{enumerate}
\item At the core of $\ours$, taking inspirations from NoCurl \citep{Yu_etal_2021Dags}
and subsequent works \citep{Massidda_etal_2023Constraint,Annadani_etal_2023Bayesdag},
we design $\vecdag$, a \textit{surjective map} from a continuous
domain into the space of all DAGs. We prove that given a fixed number
of nodes, this function can translate an unconstrained real-valued
vector into a binary matrix that represents a valid DAG, and vice
versa--there always exists a vector mapped to every possible DAGs.\vspace{-2mm}
\item Thanks to $\vecdag$, we are able to devise a policy outputting actions
in the continuous domain that are directly associated with high-reward
DAGs. The policy is one-step, unconstrained, and costs only a quadratic
number of parallel operations w.r.t. the number of nodes, allowing
our agent to explore the DAG space very effectively with arbitrary
RL method and scoring function. To our knowledge, $\ours$ is the
first score-based causal discovery method based on RL that can explore
the exact space of DAGs with an efficient one-step generation, rendering
it an efficient realization of Eq.~(\ref{eq:score-based}).\vspace{-2mm}
\item We demonstrate the effectiveness of the proposed $\ours$ method in
comparison with various state-of-the-arts on a systematic set of numerical
evaluations on both synthetic and real-world datasets. Empirical evidence
shows that our method consistently surpasses all state-of-the-art
baselines under multiple evaluation metrics on varying degrees of
nonlinearity, dimensionality, graph density, and model misspecification.
For example, our method can achieve an $\text{SHD}=0.2\pm{\color{red}{\color{black}0.2}}$
on very dense graphs with 30 nodes and 8 parents per node on average,
and on large graphs with 200 nodes and 400 edges on average, $\ours$
can still obtain a very low SHD of $2.0\pm0.9$.
\end{enumerate}
We summarize the advantages of $\ours$ compared with the state-of-the-arts
in causal discovery in Table~\ref{tab:Conceptual-comparison}. In
addition, Figure~\ref{fig:learning-curve} shows a snapshot of $\ours$'s
strong performance in a case of highly complex structures, in which
our method can achieve absolute accuracy, while the best baselines
in this setting still struggle.

\section{Related Work\label{sec:Related-Work}}

\paragraph{Constraint-based\protect\footnote{Note that the term ``constraint'' here largely refers to statistical
constraints, such as conditional independence, while ``constraint''
in our method refers to the acyclicity enforcement.} methods}

like PC, FCI \citep{Spirtes_Glymour_91Algorithm,Spirtes_etal_00Causation}
and RFCI \citep{Colombo_etal_12Learning} form a prominent class of
causal discovery approaches. They first exploit conditional independence
relationships statistically exhibited in data via a series of hypothesis
tests to recover the skeleton, which is the undirected version of
the DAG, and then orient the remaining edges using probabilistic graphical
rules. However, their performance heavily relies on the quality of
the conditional independence tests, which can deteriorate rapidly
with the number of conditioning variables \citep{Ramdas_etal_2015Decreasing},
rendering them unsuitable for large or dense graphs.

\paragraph{Score-based methods}

is another major class of DAG learners, where each DAG is assigned
a properly defined score based on its compatibility with observed
data, then the DAG learning problem becomes the optimization problem
for the DAG yielding the best score.  Score-based methods can be
further categorized based on the search approach as follows.

\subparagraph{Combinatorial greedy search methods}

such as GES \citep{Chickering_02Optimal} and FGES \citep{Ramsey_17Million}
resort to greedy heuristics to reduce the search space and enforce
acyclicity by adding one edge at a time after explicitly checking
that it would not introduce any cycle, yet this comes at the cost
of the sub-optimality of the result.

\subparagraph{Continuous optimization methods}

improve upon combinatorial optimization methods in scalability by
the ingenious smooth acyclicity constraint, introduced and made popular
with NOTEARS \citep{Zheng_etal_18DAGs}, which turns said combinatorial
optimization into a continuous optimization problem. This enables
bypassing the adversary between combinatorial enumeration and scalability
to allow for exploring the DAG space much more effectively, where
multiple edges can be added or removed in an update. Following developments,
e.g., \citet{Yu_etal_2019Dag,Lee_etal_2019Scaling,Zheng_etal_20Learning,Ng_etal_2020Role,Yu_etal_2021Dags,Zhang_etal_2022Truncated,Wei_etal_2020Dags}
contribute to extending and improving the soft DAG characterization
in scalability and convergence. Notably, unconstrained DAG parameterizations
are also proposed by \citet{Yu_etal_2021Dags} and \citet{Massidda_etal_2023Constraint},
which simplify the optimization problem from a constrained to an unconstrained
problem. However, continuous optimization methods restrict the choices
of the score to be differentiable functions, which exclude many well-studied
scores such as BIC, BDe, MDL, or independence-based scores \citep{Buhlmann_etall_14Cam}..

\subparagraph{Reinforcement learning methods}

have emerged in recent years as the promising replacement for the
greedy search heuristics discussed so far, thanks to its search ability
via exploration and exploitation. As the pioneer in this line of work,
\citet{Zhu_etal_2020Causal} introduced the first RL agent that is
trained to generate high-reward graphs. To handle acyclicity, they
incorporate the soft DAG constraint from \citet{Zheng_etal_18DAGs}
into the reward function to penalize cyclic graphs. Unfortunately,
this may not discard all cycles in the solution, but also increase
computational cost drastically due to the unnecessary reward calculations
for non-DAGs. To mitigate this issue, subsequent studies \citep{Wang_etal_2021Ordering,Yang_etal_2023Causal,Yang_etal_2023Reinforcement}
turn to finding the best-scoring causal ordering instead and subsequently
apply variable selection onto the result to obtain a DAG, which naturally
relieves our concerns with cycles. More particularly, CORL \citep{Wang_etal_2021Ordering}
is the first RL method operating on the ordering space, which defines
states as incomplete permutations and actions as the element to be
added next. GARL \citep{Yang_etal_2023Causal} is proposed to enhance
ordering generation by exploiting prior structural knowledge with
the help of graph attention networks. Meanwhile, RCL-OG \citep{Yang_etal_2023Reinforcement}
introduces a notion of order graph that drastically reduces the state
space size from $\mathcal{O}\left(d!\right)$ to only $\mathcal{O}\left(2^{d}\right)$.
It is also worth noting that the emerging generative flow networks
(GFlowNet, \citealp{Bengio_etal_2021Flow,Bengio_etal_2023Gflownet})
offer another technique for learning (distributions of) DAGs \citep{Deleu_etal_2022Bayesian,Deleu_etal_24Joint},
in which the generation of DAGs is also viewed as a sequential generation
problem, where edges are added one-by-one with explicit exclusions
of edges introducing cycles, and the transition probabilities are
learned via flow matching objectives. However, the generation of these
orderings and DAGs are usually formulated as a Markov decision process,
in which elements are iteratively added to the structure in a multiple-step
fashion, which prevents efficient concurrent DAG generations and requires
learning the transition functions, which is computationally involved
given the multitude of discrete state-action combinations.

\textbf{}

\section{Background\label{sec:Preliminaries}}

\subsection{Functional Causal Model}

Let $\mathbf{X}=\left(X_{1},\ldots,X_{d}\right)^{\top}$ be the $d$-dimensional
random (column) vector representing the variables of interest, $\mathbf{x}^{\left(k\right)}=\left(x_{1}^{\left(k\right)},\ldots,x_{d}^{\left(k\right)}\right)^{\top}\in\mathbb{R}^{d}$
denotes the $k$-th observation of $\mathbf{X}$, and $\mathcal{D}=\left\{ \mathbf{x}^{\left(k\right)}\right\} _{k=1}^{n}$
indicates the observational dataset containing $n$ i.i.d. samples
of $\mathbf{X}$. Assuming \textit{causal sufficiency}, that is, there
are no unobserved endogenous variables, the causal structure among
said variables can be described by a DAG $\mathcal{G}=\left(\mathcal{V},\mathcal{E}\right)$
where each vertex $i\in\mathcal{V}=\left\{ 1,\ldots,d\right\} $ corresponds
to a random variable $X_{i}$, and each edge $\left(j\rightarrow i\right)\in\mathcal{E}\subset\mathcal{V}\times\mathcal{V}$
implies that $X_{j}$ is a direct cause of $X_{i}$. We also denote
the set of all direct causes of a variable as its \textit{parents},
i.e., $\pa_{i}=\left\{ j\in\mathcal{V}\mid\left(j\rightarrow i\right)\in\mathcal{E}\right\} $.
The DAG $\mathcal{G}$ is also represented algebraically with a binary
\textit{adjacency matrix} $\mathbf{A}\in\left\{ 0,1\right\} ^{d\times d}$
where the $\left(i,j\right)$-th entry is 1 iff $\left(i\rightarrow j\right)\in\mathcal{E}$.
Then, the space of all (adjacency matrices of) DAGs of $d$ nodes
is denoted by $\mathbb{D}_{d}\subset\left\{ 0,1\right\} ^{d\times d}$.
We follow the Functional Causal Model framework \citep[FCM,][]{Pearl_09Causality}
to assume the data generation process as $X_{i}:=f_{i}\left(\mathbf{X}_{\mathrm{pa}_{i}},E_{i}\right),\ \forall i=1,\ldots,d$,
where the noises $E_{i}$ are mutually independent. In addition, we
also consider \textit{causal minimality} \citep{Peters_etal_2014Causal}
to ensure each function $f_{i}$ is non-constant to any of its arguments.
For any joint distribution over $\mathbf{E}=\left(E_{1},\ldots,E_{d}\right)$,
the functions $\left\{ f_{i}\right\} _{i}^{d}$ induce a joint distribution
over $\mathbf{X}$. The goal of causal discovery is then to recover
the acyclic graph $\mathcal{G}$ from empirical samples of $P\left(\mathbf{X}\right)$.

\subsection{DAG Scoring\label{subsec:DAG-Scoring}}

Among multiple DAG scoring functions well-developed in the literature
\citep{Schwarz_1978Estimating,Heckerman_etal_1995Learning,Rissanen_1978Modeling},
here we focus on the popular Bayesian Information Criterion (BIC)
\citep{Schwarz_1978Estimating}, which is adopted in many works \citep{Chickering_02Optimal,Zhu_etal_2020Causal,Wang_etal_2021Ordering,Yang_etal_2023Reinforcement}
for its flexibility, computational straightforwardness, and consistency.

More particularly, BIC is a parametric score that assumes a model
family for the causal model parameters, e.g., linear-Gaussian, which
comes with a set of parameters $\psi$ containing the parameters of
the causal mechanisms and noise distribution. This score is used to
approximate the likelihood of data given the model after marginalizing
out the model parameters using Laplace's approximation \citep{Schwarz_1978Estimating}:
\begin{equation}
\ln p\left(\mathcal{D}\mid\mathcal{G}\right)=\ln\int p\left(\mathcal{D}\mid\psi,\mathcal{G}\right)p\left(\psi\mid\mathcal{G}\right)\mathrm{d}\psi\approx\frac{\mathcal{S}_{\text{BIC}}}{2}.
\end{equation}
Given a DAG $\mathcal{G}$, the BIC score is defined generally as
follows:
\begin{equation}
\mathcal{S}_{\text{BIC}}\left(\mathcal{D},\mathcal{G}\right)=2\ln p\left(\mathcal{D}\mid\hat{\psi},\mathcal{G}\right)-\left|\mathcal{G}\right|\ln n,\label{eq:bic}
\end{equation}

\noindent where $\hat{\psi}$ is the maximum-likelihood estimator
of $p\left(\mathcal{D}\mid\psi,\mathcal{G}\right)$, $\left|\mathcal{G}\right|$
is the number of edges in $\mathcal{G}$, and $n$ is the number of
samples in $\mathcal{D}$.

The BIC is consistent in the sense that if a causal model is identifiable,
asymptotically, the true DAG has the highest score among all other
DAGs \citep{Haughton_88Choice}. Meanwhile, for limited samples, it
prevents overfitting by penalizing edges that do not improve the log-likelihood
significantly. More formally:
\begin{lem}
\label{lem:BIC-consistency}Let $\mathcal{G}^{\ast}$ be the ground
truth DAG of an identifiable SCM satisfying causal minimality \citep{Peters_etal_2014Causal}
(i.e., there are no redundant edges) inducing the dataset $\mathcal{D}$,
and let $n$ be the sample size of $\mathcal{D}$. Then, in the limit
of large $n$, $\mathcal{S}_{\text{BIC}}\left(\mathcal{D},\mathcal{G}^{\ast}\right)>\mathcal{S}_{\text{BIC}}\left(\mathcal{D},\mathcal{G}\right)$
for any $\mathcal{G}\neq\mathcal{G}^{\ast}$.
\end{lem}
The proof can be found in Appendix~\ref{subsec:Proof-of-BIC-consistency}.

As an example, for additive noise models (ANM) $X_{i}:=f_{i}\left(\mathbf{X}_{\mathrm{pa}_{i}}\right)+E_{i},\ \forall i=1,\ldots,d$
with Gaussian noises $E_{i}\sim\mathcal{N}\left(0;\sigma_{i}^{2}\right)$,
the BIC-based score can be specified as 
\[
\mathcal{S}_{\text{BIC-NV}}\left(\mathcal{D},\mathcal{G}\right)=-(n\sum_{i=1}^{d}\ln\frac{\mathrm{SSR}_{i}}{n}+\left|\mathcal{G}\right|\ln n),
\]
where $\mathrm{SSR}_{i}=\sum_{k=1}^{n}(\hat{x}_{i}^{\left(k\right)}-x_{i}^{\left(k\right)})^{2}$
is the sum of squared residuals after regressing $X_{i}$ on its parents
in $\mathcal{G}$, and we adopt the convention that $\left|\mathcal{G}\right|$
is the number of edges in $\mathcal{G}$. Additionally assuming equal
noise variances gives us with 
\[
\mathcal{S}_{\mathrm{BIC-EV}}\left(\mathcal{D},\mathcal{G}\right)=-(nd\ln\frac{\sum_{i=1}^{d}\mathrm{SSR}_{i}}{nd}+\left|\mathcal{G}\right|\ln n).
\]
The derivations of BIC scores are presented in Appendix~\ref{subsec:Derivation-of-BIC}.
A simpler yet widely adopted alternative is the least squares (LS)
\citep{Zheng_etal_18DAGs,Lachapelle_etal_2020Gradient,Yu_etal_2021Dags,Bello_etal_22dagma,Massidda_etal_2023Constraint}.
With an additional $l_{0}$ regularization, we define the LS score
as $\mathcal{S}_{\text{LS}}\left(\mathcal{D},\mathcal{G}\right)=-(\sum_{i=1}^{d}\mathrm{SSR}_{i}+\lambda_{0}\left|\mathcal{G}\right|)$,
where $\lambda_{0}\geq0$ is a hyper-parameter for penalizing dense
graphs. In our empirical studies, following common practices \citep{Zhu_etal_2020Causal,Wang_etal_2021Ordering,Yang_etal_2023Causal,Yang_etal_2023Reinforcement},
linear regression is used for linear data and Gaussian process regression
is adopted for nonlinear data. That being said, any valid regression
technique can be seamlessly integrated into our method.

\subsection{DAG Representations}

Typically, to search over the space of DAGs, modern causal discovery
methods either optimize over directed graphs with differentiable DAG
regularizers \citep{Zheng_etal_18DAGs,Yu_etal_2019Dag,Lee_etal_2019Scaling,Zheng_etal_20Learning,Ng_etal_2020Role,Wei_etal_2020Dags,Zhu_etal_2020Causal,Zhang_etal_2022Truncated},
or search over causal orderings and then apply variable selection
to suppress redundant edges \citep{Cundy_etal_2021Bcd,Charpentier_etal_2022Differentiable,Chen_etal_19Causal,Wang_etal_2021Ordering,Rolland_etal_22Score,Sanchez_etal_22Diffusion,Yang_etal_2023Causal}.
The former approach does not guarantee the acyclicity of the returned
graph with absolute certainty, while the latter approach faces challenges
in efficiently generating permutations. For instance, in \citet{Cundy_etal_2021Bcd}
the permutation matrix representing the causal ordering is parametrized
by the Sinkhorn operator \citep{Sinkhorn_1964Relationship} followed
by the Hungarian algorithm \citep{Kuhn_1955Hungarian} with a considerable
cost of $\mathcal{O}\left(d^{3}\right)$. Other examples include ordering-based
RL methods \citep{Wang_etal_2021Ordering,Yang_etal_2023Reinforcement}
that cost at least $\mathcal{O}\left(d^{2}\right)$ just to generate
a single ordering element, thus totaling an $\mathcal{O}\left(d^{3}\right)$
complexity for generating a DAG.

Our work takes inspiration from \citet{Yu_etal_2021Dags}, where a
novel unconstrained characterization of \textit{weighted} adjacency
matrices of DAGs is proposed. Particularly, a ``node potential''
vector $\mathbf{p}\in\mathbb{R}^{d}$ is introduced to model an \textit{implicit}
causal ordering, where $i$ precedes $j$ if $p_{j}>p_{i}$. Hence,
the weight matrix
\begin{equation}
\mathbf{A}=\mathbf{W}\odot\mathrm{ReLU}\left(\mathrm{grad}\left(\mathbf{p}\right)\right),\label{eq:nocurl}
\end{equation}

\noindent where $\mathbf{W}\in\mathbb{R}^{d\times d}$ and $\mathrm{grad}\left(\mathbf{p}\right)_{ij}:=p_{j}-p_{i}$
is the gradient flow operator \citep{Lim_20Hodge}, can be shown to
correspond to a valid DAG \citep[Theorem 2.1,][]{Yu_etal_2021Dags}.
However, this weight matrix is only applicable for linear models,
and the representation is only used as a refinement for the result
returned by a constrained optimization problem. An alternative to
this characterization is recently introduced by \citet{Massidda_etal_2023Constraint},
where $\mathbf{A}=\mathbf{W}\odot\mathrm{sigmoid}\left(\mathrm{grad}\left(\mathbf{p}\right)/\tau\right)$,
yet this approach only ensures acyclicity at the limit of the annealing
temperature $\tau\rightarrow0^{+}$, which is usually not exactly
achieved in practice. Additionally, the equivalent DAG formulation
of \citet{Annadani_etal_2023Bayesdag} uses the node potential $\mathbf{p}$
to represent a smooth \textit{explicit} permutation matrix: $\sigma\left(\mathbf{p}\right):=\lim_{\tau\rightarrow0^{+}}\mathrm{Sinkhorn}(\mathbf{p}\cdot\left[1\ldots d\right]/\tau)$,
again necessitating a temperature scheduler and the expensive Sinkhorn
operator, which reportedly requires 300 iterations to converge and
an $\mathcal{O}\left(d^{3}\right)$ complexity of the Hungarian algorithm,
for generating a single DAG.

\section{$\protect\vecdag$: Unconstrained Parametrization of DAGs\label{sec:Unconstrained-Parametrization}}

\subsection{The $\protect\vecdag$ Operator\label{subsec:Vec2DAG}}

Extending from the formulation in Eq.~(\ref{eq:nocurl}), we design
a deterministic translation from an unconstrained continuous space
to the space of general \textit{binary} adjacency matrices of all
DAGs, not restricted to linear models. To be more specific, in addition
to the node potential vector $\mathbf{p}\in\mathbb{R}^{d}$, we introduce
a strictly upper-triangular ``edge potential'' matrix $\mathbf{E}\in\mathbb{R}^{d\times d}$,
which can be described using $\frac{d\cdot\left(d-1\right)}{2}$ parameters.
We then combine them with $\mathbf{p}$ to create a unified representation
vector $\mathbf{z}\in\mathbb{S}_{d}=\mathbb{R}^{d\cdot\left(d+1\right)/2}$,
which is the parameter space of all $d$-node DAGs in our method.
Furthermore, we denote by $\mathbf{p}\left(\mathbf{z}\right)$ and
$\mathbf{E}\left(\mathbf{z}\right)$ the node and edge potential components
constituting $\mathbf{z}$, respectively. Specifically, $\mathbf{p}(\mathbf{z})$
represents the node potential vector formed by the first $d$ elements
of $\mathbf{z}$, while $\mathbf{E}(\mathbf{z})$ is the edge potential
matrix, with the elements above the main diagonal derived from the
last $\frac{d\cdot(d-1)}{2}$ elements of $\mathbf{z}$ (see our code
in Figure~\ref{fig:DAG-transformation}). Then, our unconstrained
DAG parametrization $\vecdag_{d}$ for $d$ nodes can be defined as
follows.
\begin{defn}
For all $d\in\mathbb{N}^{+}$ and $\mathbf{z}\in\mathbb{S}_{d}$:

\begin{equation}
\vecdag_{d}\left(\mathbf{z}\right):=H\left(\mathbf{E}\left(\mathbf{z}\right)+\mathbf{E}\left(\mathbf{z}\right)^{\top}\right)\odot H\left(\mathrm{grad}\left(\mathbf{p}\left(\mathbf{z}\right)\right)\right),\label{eq:vec2dag}
\end{equation}

\noindent where $H\left(x\right):=\begin{cases}
1 & \text{if }x>0,\\
0 & \text{otherwise.}
\end{cases}$ is known as the Heaviside step function and $\odot$ is the Hadamard
(element-wise) product operator. 
\end{defn}
The intuition behind Vec2DAG is that the first term in Eq.~(\ref{eq:vec2dag})
defines a symmetric binary adjacency matrix, determining whether two
nodes are connected. The directions of these connections are then
dictated by the second term in Eq.~(\ref{eq:vec2dag}), resulting
in a binary matrix that represents a directed graph. Additionally,
this directed graph is guaranteed to be acyclic due to the use of
the gradient flow operator.

\noindent The procedure to sample a DAG is then denoted as $\mathbf{z}\sim P\left(\mathbf{z}\right),\ \mathbf{A}=\vecdag_{d}\left(\mathbf{z}\right)$,
which can be implemented in a few lines of code, as illustrated in
Figure~\ref{fig:DAG-transformation} of the Appendix. The validity
of our parametrization is justified by the following theorem.
\begin{thm}
\label{thm:surjective}For all $d\in\mathbb{N}^{+}$, let $\vecdag_{d}:\mathbb{S}_{d}\rightarrow\left\{ 0,1\right\} ^{d\times d}$
be defined as in Eq.~(\ref{eq:vec2dag}). Then, $\mathrm{Im}\left(\vecdag_{d}\right)=\mathbb{D}_{d}$,
where $\mathrm{Im}\left(\cdot\right)$ is the Image operator, and
$\mathbb{D}_{d}$ is the space of all $d$-node DAGs.
\end{thm}
The proof can be found in Appendix~\ref{subsec:Proof-of-Theorem}.
Our formulation directly represents a DAG by a real-valued vector,
which is in stark contrary to existing unconstrained methods that
only aim for a DAG sampler \citep{Cundy_etal_2021Bcd,Wang_etal_2021Ordering,Charpentier_etal_2022Differentiable,Deleu_etal_2022Bayesian,Annadani_etal_2023Bayesdag,Yang_etal_2023Causal,Yang_etal_2023Reinforcement}.
More notably, this approach requires \textbf{no temperature annealing}
like in \citet{Massidda_etal_2023Constraint,Annadani_etal_2023Bayesdag},
and can generate a valid DAG in a \textbf{single step} since sampling
$\mathbf{z}$ can be done instantly in an unconstrained manner. In
addition, this merely costs \textbf{$\bm{\mathcal{O}\left(d^{2}\right)}$
parallelizable operations} compared with the $\mathcal{O}\left(d^{3}\right)$
cost of sequentially generating permutations using the Sinkhorn operator
\citep{Cundy_etal_2021Bcd,Charpentier_etal_2022Differentiable,Annadani_etal_2023Bayesdag}
and multiple-step RL methods \citep{Wang_etal_2021Ordering,Yang_etal_2023Causal,Yang_etal_2023Reinforcement}.
Moreover, our generation technique is one-step, and thus \textbf{does
not require learning any transition function}, which vastly reduces
the computational burden compared with RL methods based on sequential
decisions.

\subsection{Properties of $\protect\vecdag$}

In this section, we show that our parameterization $\vecdag$ has
some important additional properties that set it apart from past formulations.
\begin{lem}
(Scaling and Translation Invariance). \label{lemma:invariance}For
all $d\in\mathbb{N}^{+}$, let $\vecdag_{d}:\mathbb{S}_{d}\rightarrow\left\{ 0,1\right\} ^{d\times d}$
be defined as in Eq.~(\ref{eq:vec2dag}). Then, for all $\mathbf{z}\in\mathbb{S}_{d}$,
$\alpha>0$, and ${\bf \bm{\beta}}\in\mathbb{\mathbb{S}}_{d}$ such
that $\left|\mathbf{p}\left(\bm{\beta}\right)_{i}\right|<\nicefrac{1}{2}\min_{j}\left|\mathbf{p}\left(\mathbf{z}\right)_{i}-\mathbf{p}\left(\mathbf{z}\right)_{j}\right|$
and $\left|\mathbf{E}\left(\bm{{\bf \beta}}\right)_{ij}\right|<\left|\mathbf{E}\left(\mathbf{z}\right)_{ij}\right|$
$\forall i,j$, we have $\mathbb{\vecdag}_{d}\left(\mathbf{z}\right)=\vecdag_{d}\left(\alpha\cdot(\mathbf{z}+{\bf \bm{\beta}})\right)$.
\end{lem}
This insight is proven in Appendix~\ref{subsec:Proof-of-Invariance}.
Intuitively, this indicates that scaling the potential by any positive
constant $\alpha$ results in the same DAG ($\text{Vec2DAG}(\mathbf{z})=\text{Vec2DAG}(\alpha\cdot\mathbf{z})$),
and translating the potential by an amount $\beta$ (which can be
large, provided it does not change the ordering of $\mathbf{p}$ or
the element-wise positivity of $\mathbf{E}$) also results in the
same DAG ($\text{Vec2DAG}(\mathbf{z})=\text{Vec2DAG}(\mathbf{z}+\bm{{\bf \beta}})$).
In other words, any DAG can be diversely constructed by infinitely
many representations, suggesting a dense parameter space where representations
of different DAGs are close to each other. This leads us to the next
point, which shows an upper bound of the distance between an arbitrary
representation with a representation of any DAG.
\begin{lem}
(Proximity between DAGs).\label{lemma:proximity} Let $\mathbf{z}\in\mathbb{S}_{d}$.
Then, for any DAG $\mathbf{A}\in\mathbb{D}_{d}$ and $\epsilon>0$,
there exists $\mathbf{z}_{\mathbf{A}}$ in the unit ball $B\left(\infty;\left\Vert \mathbf{z}\right\Vert _{\infty}+\epsilon\right)$
around $\mathbf{z}$ such that $\vecdag_{d}\left(\mathbf{z}_{\mathbf{A}}\right)=\mathbf{A}$.
\end{lem}
We provide the proof in Appendix~\ref{subsec:Proof-of-proximity}.
This property is not straightforward in existing constrained optimization
approaches \citep{Zheng_etal_18DAGs,Lee_etal_2019Scaling,Zheng_etal_20Learning,Lachapelle_etal_2020Gradient,Bello_etal_22dagma},
and suggests that the true DAG may be found closer to the initial
position if we start from a smaller scale in our framework. We leverage
this result in our implementation by restricting $\mathbb{S}_{d}$
to a hypercube $\left[-\gamma,\gamma\right]^{d\cdot\left(d+1\right)/2}$
with a relatively small $\gamma=10$. This has the effect of regularizing
the search space but still does not invalidate our Theorem~\ref{thm:surjective},
i.e., we can still reach every possible DAGs when searching in this
restricted space.

\section{$\protect\ours$: Reinforced DAG Learning without Acyclicity Constraints\label{sec:Method}}

\subsection{Motivation for Reinforcement Learning}

Using the $\vecdag$ representation, the score-based causal discovery
problem may seem to simplify into a maximization problem: $\mathbf{z}^{*}=\underset{\mathbf{z}\in\mathbb{R}^{d\cdot\left(d+1\right)/2}}{\arg\max}\mathcal{S}_{\text{BIC}}\left(\mathcal{D},\vecdag\left(\mathbf{z}\right)\right)$,
which could, in principle, be addressed to certain extents using off-the-shelf
black-box optimization techniques such as Bayesian optimization, which
is one of the most popular blackbox optimization methods. However,
solving this optimization problem is far from straightforward due
to the high dimensionality of the search space, which grows quadratically
with the number of nodes (e.g., for 30 nodes, the space is 465-dimension).

Since RL is the only black-box optimization approach that has been
studied in the score-based causal discovery literature (at the time
writing this manuscript, to the best of our knowledge), we align with
the established line of works \citep{Zhu_etal_2020Causal,Wang_etal_2021Ordering,Yang_etal_2023Reinforcement,Yang_etal_2023Causal}
to specifically focus on leveraging RL to solve this optimization
problem. The idea is that an RL agent with a stochastic policy can
autonomously decide where to explore based on the uncertainty of the
learned policy, which is continuously updated through incoming reward
signals \citep{Zhu_etal_2020Causal}. That said, we note that, our
method, which only involves one-step trajectories as shown below,
is better perceived as a policy gradient approach rather than a general
RL method that requires a multi-step Markov decision process (MDP).

As shown below, policy gradient provides our method with built-in
exploration-exploitation capabilities and linear scalability with
respect to both dimensionality and sample size, making it a practical
choice for this problem. In addition, our policy gradient point of
view also enables the adaptability of various established methods,
such as vanilla policy gradient \citep{Sutton_etal_1999Policy}, A2C
\citep{Mnih_etal_2016Asynchronous}, and PPO \citep{Schulman_etal_2017Proximal},
to efficiently optimize our objective, effectively establishing a
clear association between our proposed approach and policy gradient
in RL.

\subsection{Policy Gradient for DAG Search\label{subsec:Policy-Gradient}}

\paragraph{Policy and Action.}

Utilizing RL, we seek for a policy $\pi$ that outputs a continuous
action $\mathbf{z}\in\mathbb{S}_{d}=\mathbb{R}^{d\cdot\left(d+1\right)/2}$,
which is the parameter space of DAGs of $d$ nodes. In this work,
we consider stochastic policies for better exploration, i.e., we parametrize
our policy by an isotropic Gaussian distribution with learnable means
and variances: $\pi_{\bm{\theta}}\left(\mathbf{z}\right)=\mathcal{N}\left(\mathbf{z};\bm{\mu}_{\bm{\theta}},\text{diag}\left(\mathbf{\bm{\sigma}}_{\bm{\theta}}^{2}\right)\right)$.
Since our policy generates a DAG representation in just one step,
every trajectory starts with the same initial state and terminates
after only one transition, so the agent does not need to be aware
of the state in our method. We note that, similar to RL-BIC \citep{Zhu_etal_2020Causal},
the one-step nature of the environment does not preclude the application
of RL in our approach. This is because a one-step environment is simply
a special case of a MDP, which remains compatible with most RL algorithms.

\paragraph{Reward.}

The reward of an action in our method is set as the graph score of
the DAG induced by that action with respect to the observed dataset
$\mathcal{D}$ (Section~\ref{subsec:DAG-Scoring}), and divided by
$n\times d$ to maintain numerical stability without modifying the
monotonicity of the score: 
\begin{equation}
\mathcal{R}\left(\mathbf{z}\right):=\frac{1}{n\times d}\mathcal{S}\left(\mathcal{D},\textbf{Vec2DAG}\left(\mathbf{z}\right)\right).\label{eq:reward}
\end{equation}

\paragraph{Policy Gradient Algorithm.}

\begin{algorithm}[t]
\caption{$\protect\ours$ with vanilla policy gradient for causal discovery.\label{alg:ALIAS}}

\begin{algorithmic}[1]

\REQUIRE{Dataset $\mathcal{D}=\left\{ \mathbf{x}^{\left(k\right)}\right\} _{k=1}^{n}$,
score function $\mathcal{S}\left(\mathcal{D},\cdot\right)$, batch
size $B$, and learning rate $\eta$.}

\ENSURE{Estimated causal DAG $\hat{\mathcal{G}}$.}

\WHILE{not terminated}

\STATE Draw a minibatch of $B$ actions from the policy: $\left\{ \mathbf{z}^{\left(k\right)}\sim\pi_{\bm{\theta}}\right\} _{k=1}^{B}$.

\STATE Collect rewards $\left\{ r^{\left(k\right)}:=\frac{1}{n\times d}\mathcal{S}\left(\mathcal{D},\textbf{Vec2DAG}_{d}\left(\mathbf{z}^{\left(k\right)}\right)\right)\right\} _{k=1}^{B}$.\hfill$\triangleright$
Sec.~\ref{subsec:Vec2DAG}

\STATE Update policy as: $\bm{\theta}:=\bm{\theta}+\eta\left(\frac{1}{B}\sum_{k=1}^{B}\nabla_{\bm{\theta}}\ln\pi_{\bm{\theta}}\left(\mathbf{z}^{\left(k\right)}\right)\cdot r^{\left(k\right)}\right)$.\hfill$\triangleright$
Sec.~\ref{subsec:Policy-Gradient}

\ENDWHILE

\STATE $\mathbf{z}\sim\pi_{{\bf \theta}}$, $\hat{\mathcal{G}}:=\mathrm{\vecdag}\left(\mathbf{z}\right)$.

\STATE Post-process $\hat{\mathcal{G}}$ by pruning if needed and
return.\hfill$\triangleright$ Sec.~\ref{subsec:Post-Processing}

\end{algorithmic}
\end{algorithm}

Since our action space is continuous, we employ policy gradient methods,
which are well established for handling continuous actions, rather
than the value-based approach as in recent RL-based techniques \citep{Wang_etal_2021Ordering,Yang_etal_2023Causal,Yang_etal_2023Reinforcement}.
The training objective is to maximize the expected return defined
as $\mathcal{J}\left(\bm{\mathbf{\theta}}\right)=\mathbb{E}_{\mathbf{z}\sim\pi_{\bm{\theta}}}\left[\mathcal{R}\left(\mathbf{z}\right)\right]$.

Under identifiable causal models, causal minimality, and a consistent
scoring function, the optimal policy obtained by maximizing this objective
will return the true DAG:
\begin{lem}
\label{pg-guarantee}Assuming causal identifiability and causal minimality,
that is, there is a unique causal model with no redundant edges that
can produce the observed dataset, and BIC score is used to define
the reward $\mathcal{R}\left(\mathbf{z}\right)$ as in Eq.~(\ref{eq:reward}).
Let $n$ be the sample size of the observed dataset $\mathcal{D}$,
$\theta^{*}\in\underset{\mathbf{\theta}\in\Theta}{\arg\max}\mathbb{E}_{\mathbf{z}\sim\pi_{\bm{\theta}}}\left[\mathcal{R}\left(\mathbf{z}\right)\right]$,
where $\pi_{\bm{\theta}}\left(\mathbf{z}\right)=\mathcal{N}\left(\mathbf{z};\bm{\mu}_{\bm{\theta}},\text{diag}\left(\mathbf{\bm{\sigma}}_{\bm{\theta}}^{2}\right)\right)$.
Then, as $n\rightarrow\infty$, $\mathcal{G}=\vecdag\left(\mathbf{z}\right)$
is the true DAG, where $\mathbf{z}\sim\pi_{\theta^{*}}$.
\end{lem}
The proof is presented in Appendix~\ref{subsec:Proof-of-pg-guarantee}.
The differential entropy of the policy can also be added to the expected
return as a regularization term to encourage exploration \citep{Mnih_etal_2016Asynchronous},
however we find in our experiments that the stochasticity offered
by the policy suffices for exploration. That said, we also investigate
the effect of entropy regularization in our empirical studies. During
training, the parameter $\bm{\theta}$ is updated in the direction
suggested by the policy gradient algorithm. For example, using vanilla
policy gradient, the gradient is given by the policy gradient theorem
\citep{Sutton_etal_1999Policy} as: $\nabla_{\bm{\theta}}\mathcal{J}\left(\bm{\theta}\right)=\mathbb{E}_{\mathbf{z}\sim\pi_{\bm{\theta}}}\left[\nabla_{\bm{\theta}}\ln\pi_{\bm{\theta}}\left(\mathbf{z}\right)\cdot\mathcal{R}\left(\mathbf{z}\right)\right].$

Note that since our trajectories are one-step and our environment
is deterministic, the state-action value function is always equal
to the immediate reward, and therefore there is no need for a critic
to estimate the value function. Hence, vanilla policy gradient works
well out-of-the-box for our framework, yet in practice our method
can be implemented with more advanced algorithms for improved training
efficiency. Our practical implementation considers the basic policy
algorithm Advantage Actor-Critic (A2C, \citealp{Mnih_etal_2016Asynchronous})
and a more advanced method Proximal Policy Optimization (PPO, \citealp{Schulman_etal_2017Proximal}).
In addition, while policy gradient only ensures local convergence
under suitable conditions \citep{Sutton_etal_1999Policy}, our empirical
evidence remarks that our method can reach the exact ground truth
DAG in notably many cases.

\subsection{Post Processing\label{subsec:Post-Processing}}

With limited sample sizes, due to overfitting, redundant edges may
still be present in the returned DAG that achieves the highest score.
One approach towards suppressing the false discovery rate is to greedily
remove edges with non-substantial contributions in the score. For
linear models, a standard approach is to threshold the absolute values
of the estimated weight matrix $\mathbf{W}$ at a certain level $\delta$
\citep{Zheng_etal_18DAGs,Ng_etal_2020Role,Bello_etal_22dagma}, i.e.,
removing all edges $\left(i\rightarrow j\right)$ with $\left|\mathbf{W}_{ij}\right|<\delta$.
For nonlinear models, the popular CAM pruning method \citep{Buhlmann_etall_14Cam}
can be employed for generalized additive models (GAMs), which performs
a GAM regression on the parents set and exclude the parents that do
not pass a predefined significance level. An alternative pruning method
that does not depend on the causal model is based on conditional independence
(CI), i.e., by imposing Faithfulness \citep{Spirtes_etal_00Causation},
for each $j\in\pa_{i}$ in the graph found so far, we remove the edge
$\left(j\rightarrow i\right)$ if $X_{i}\indep X_{j}\mid X_{\pa_{i}\setminus\left\{ j\right\} }$,
which is a direct consequence of the Faithfulness assumption and can
be realized with available CI tests like KCIT \citep{Zhang_etal_11Kernel}.
In addition, for the least squares score, we can increase the regularization
strength on the number of edges to encourage sparsity during the learning
process. Our numerical experiments investigate the effects of all
these approaches.

To summarize, Algorithm~\ref{alg:ALIAS} highlights the key steps
of our $\ours$ method for the case with vanilla gradient policy.

\section{Numerical Evaluations\label{sec:Results}}

In this section, we validate our method in the causal discovery task
across a comprehensive set of settings, including \textit{different
nonlinearities}, \textit{varying graph} \textit{types},\textit{ sizes
and densities}, \textit{varying sample sizes}, as well as \textit{different
degrees of model misspecification} on \textit{both synthetic and real
data}. In addition, we also analyze the computational efficiency of
our method, as well as demonstrate the significance of different components
of our method, especially the choice of reinforcement learning as
the optimizer, in our extensive ablation studies.

\subsection{Experiment Setup}

\begin{figure*}[t]
\centering{}\includegraphics[width=1\textwidth]{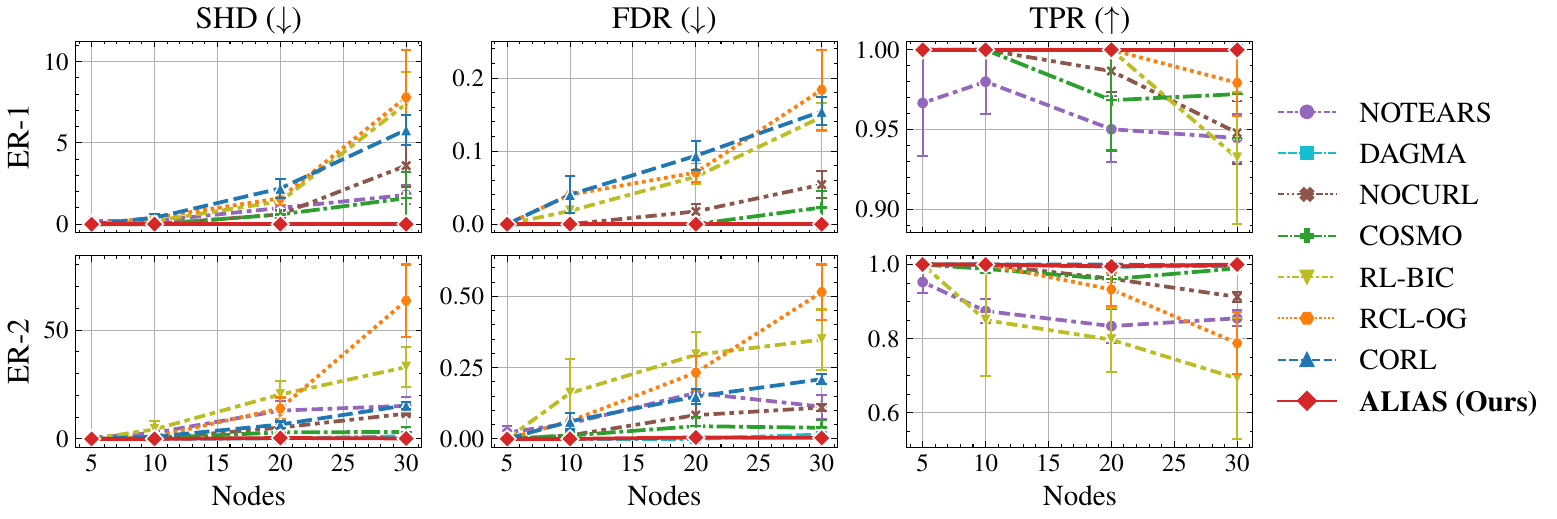}\caption{\textbf{Causal Discovery Performance on Linear-Gaussian Data.} ER-1
and ER-2 denote Erd\H{o}s-R\'{e}nyi graph models with expected in-degrees
of 1 and 2, respectively. The weight range is $\mathcal{U}\left(\left[-5,-2\right]\cup\left[2,5\right]\right)$,
which is wider than prior studies, making our setting more challenging
due to higher data variance. We compare the proposed $\protect\ours$
method with NOTEARS \citep{Zheng_etal_18DAGs}, DAGMA \citep{Bello_etal_22dagma},
NOCURL \citep{Yu_etal_2021Dags}, COSMO \citep{Massidda_etal_2023Constraint},
RL-BIC \citep{Zhu_etal_2020Causal}, CORL \citep{Wang_etal_2021Ordering},
and RCL-OG \citep{Yang_etal_2023Reinforcement}. The performance metrics
are Structural Hamming Distance (SHD, lower is better), False Detection
Rate (FDR, lower is better), and True Positive Rate (TPR, higher is
better). Shaded areas depict standard errors over 5 independent runs.\label{fig:linear-gaussian}}
\end{figure*}

We conduct extensive empirical evaluations on both simulated and real
datasets, where the ground truth DAGs are available, to compare the
efficiency of the proposed $\ours$ method with up-to-date state-of-the-arts
in causal discovery, including the constrained continuous optimization
approaches with soft DAG constraints NOTEARS \citep{Zheng_etal_18DAGs,Zheng_etal_20Learning}
and DAGMA \citep{Bello_etal_22dagma}, unconstrained continuous optimization
approaches NOCURL \citep{Yu_etal_2021Dags} and COSMO \citep{Massidda_etal_2023Constraint},
as well as three RL-based methods RL-BIC \citep{Zhu_etal_2020Causal},
CORL \citep{Wang_etal_2021Ordering}, and RCL-OG \citep{Yang_etal_2023Reinforcement}.
A brief description of these methods along with their implementation
details and hyper-parameter specifications are provided in Appendix~\ref{sec:Experiment-Details}
and the evaluation metrics are described in Appendix~\ref{subsec:Metrics}.
For the main experiments, we use the variant with BIC score and PPO
algorithm for our method, and examine other variants in the ablation
studies. We report supplementary results, including additional ablation
studies, in Appendix~\ref{sec:Additional-Experiments}.

\subsection{Linear Data with Gaussian and non-Gaussian Noises\label{subsec:Linear-Data}}

\begin{table}
\caption{\textbf{Causal discovery performance on dense graphs (30-node ER-8)
and high-dimensional graphs (200-node ER-2) with linear-Gaussian data.}
The performance metrics are Structural Hamming Distance (SHD, lower
is better), False Detection Rate (FDR, lower is better), and True
Positive Rate (TPR, higher is better). The numbers are \textit{mean
\textpm{} standard error} over 5 independent runs. \textbf{Bold}:
best performance, \uline{underline}: second-best performance. RL-BIC
\& CORL fail to run high-dimensional tasks.\label{tab:dense-highdim}}

\begin{centering}
\par\end{centering}

\centering{}\resizebox{\columnwidth}{!}{%
\begin{tabular}{ccccccc}
\toprule 
\multirow{2}{*}{\textbf{Method}} & \multicolumn{3}{c}{\textbf{Dense graphs (30 nodes, $\approx$ 240 edges)}} & \multicolumn{3}{c}{\textbf{High-dimensional graphs (200 nodes, $\approx$ 400 edges)}}\tabularnewline
\cmidrule{2-7}
 & \textbf{SHD ($\downarrow$)} & \textbf{FDR ($\downarrow$)} & \textbf{TPR ($\uparrow$)} & \textbf{SHD ($\downarrow$)} & \textbf{FDR ($\downarrow$)} & \textbf{TPR ($\uparrow$)}\tabularnewline
\midrule
NOTEARS \citep{Zheng_etal_18DAGs} & $141.2\pm11.9$ & $0.25\pm0.03$ & $0.55\pm0.03$ & $\phantom{00}53.8\pm\phantom{0}6.5$ & $0.06\pm0.01$ & $0.93\pm0.01$\tabularnewline
DAGMA \citep{Bello_etal_22dagma} & \textit{$\phantom{0}\underline{67.6\pm\mathit{\mathit{\phantom{0}8.0}}}$} & \textit{$\underline{0.14\pm\mathit{0.02}}$} & $0.82\pm0.02$ & $\phantom{000}\underline{9.6\pm\phantom{0}2.7}$ & $\underline{0.02\pm0.00}$ & $\underline{0.99\pm0.00}$\tabularnewline
NOCURL \citep{Yu_etal_2021Dags} & $147.6\pm\phantom{0}5.7$ & $0.32\pm0.01$ & $0.63\pm0.00$ & $\phantom{0}227.6\pm17.5$ & $0.20\pm0.03$ & $0.59\pm0.02$\tabularnewline
COSMO \citep{Massidda_etal_2023Constraint} & $\phantom{0}97.4\pm\phantom{0}6.8$ & $0.24\pm0.01$ & $0.80\pm0.02$ & $\phantom{0}158.0\pm19.5$ & $0.25\pm0.03$ & $0.87\pm0.02$\tabularnewline
RL-BIC \citep{Zhu_etal_2020Causal} & $180.6\pm21.7$ & $0.43\pm0.06$ & $0.42\pm0.14$ & - & - & -\tabularnewline
CORL \citep{Wang_etal_2021Ordering} & $\phantom{0}82.4\pm22.3$ & $0.23\pm0.05$ & \textit{$\underline{0.87\pm\mathit{0.04}}$} & - & - & -\tabularnewline
RCL-OG \citep{Yang_etal_2023Reinforcement} & $199.7\pm\phantom{0}7.1$ & $0.47\pm0.01$ & $0.51\pm0.04$ & $1076.6\pm28.8$ & $0.89\pm0.00$ & $0.32\pm0.01$\tabularnewline
\midrule 
$\ours$ (Ours) & \textbf{$\mathbf{\phantom{0}\phantom{0}0.2\pm\phantom{0}0.2}$} & \textbf{$\mathbf{0.00\pm0.00}$} & $\mathbf{1.00\pm0.00}$ & \textbf{$\mathbf{\phantom{000}2.0\pm\phantom{0}0.9}$} & \textbf{$\mathbf{0.00\pm0.00}$} & \textbf{$\mathbf{1.00\pm0.00}$}\tabularnewline
\bottomrule
\end{tabular}}
\end{table}

For a given number of nodes $d$, we first generate a DAG following
the Erd\H{o}s-R\'{e}nyi graph model \citep{Erdos_Renyi_60Evolution}
with an expected in-degree of $k\in\mathbb{N}^{+}$, denoted by ER-$k$.
Next, edge weights are randomly sampled from the uniform distribution
$P\left(\mathbf{W}\right)$, and the noises are drawn from the standard
Gaussian $E_{i}\sim\mathcal{N}\left(0,1\right)$. To make this setting
more challenging, we use a wider range $P\left(\mathbf{W}\right)=\mathcal{U}\left(\left[-5,-2\right]\cup\left[2,5\right]\right)$
compared with the common range of $\mathcal{U}\left(\left[-2,-0.5\right]\cup\left[0.5,2\right]\right)$
in previous studies \citep{Zheng_etal_18DAGs,Zhu_etal_2020Causal,Wang_etal_2021Ordering,Bello_etal_22dagma}.
We then sample $n=1\,000$ observations for each dataset. This causal
model is identifiable due to the equal noise variances \citep{Peters_etal_2014Causal}.
For fairness, we also apply the same pruning procedure with linear
regression coefficients thresholded at $0.3$ for all methods and
use the equal-variance BIC (Section~\ref{subsec:DAG-Scoring}) for
RL-BIC, CORL, RCL-OG, and $\ours$.

\paragraph{Small to moderate graphs.}

In Figure~\ref{fig:linear-gaussian} we report the causal discovery
performance for linear-Gaussian data with small to moderate graph
sizes and densities, showing that our method consistently achieves
near-perfect performance in all metrics, which can be expected thanks
to its ability to explore the DAG space competently. Overall, the
closest method with comparable performance to our method in this case
is DAGMA, followed by COSMO, which are among the most advanced continuous
optimization approaches, while other methods, including RL-based ones,
still struggle even in this simplest scenario.In addition, the results
on Scale-Free (SF) graphs and the common weight range $\mathcal{U}\left(\left[-2,-0.5\right]\cup\left[0.5,2\right]\right)$
can also be found in Appendix~\ref{sec:Additional-Experiments}.

\paragraph{Dense \& High-dimensional graphs.}

\begin{wrapfigure}[18]{R}{0.45\columnwidth}%
\vspace{-3mm}

\begin{centering}
\includegraphics[width=0.45\columnwidth]{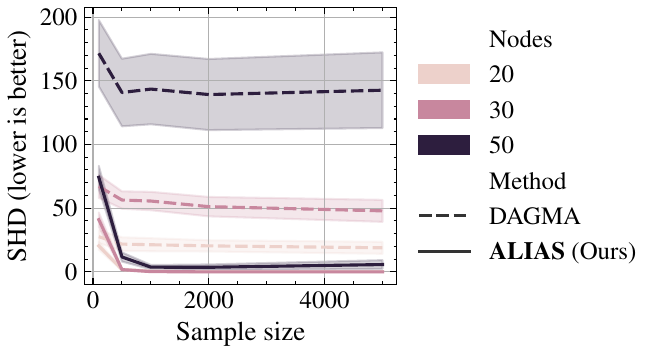}
\par\end{centering}
\centering{}\caption{\textbf{Causal Discovery performance (linear-Gaussian data on ER-8
graphs) as function of sample size ($100$ to $5\,000$)}. Shaded
areas depict standard errors over 5 independent runs.\label{fig:sample-size}}
\end{wrapfigure}%
We next test the proposed method's ability to adapt to highly complex
scenarios, including the cases with very dense graphs (ER-8 graphs)
and larger number of nodes (200-node graphs). This is to demonstrate
the advantages of our proposed method over existing approaches that
typically struggle on slightly dense graphs of ER-4 at most \citep{Zheng_etal_18DAGs,Yu_etal_2021Dags,Bello_etal_22dagma,Massidda_etal_2023Constraint}
or small graphs of only tens of nodes \citep{Zhu_etal_2020Causal,Yang_etal_2023Reinforcement}.
In this case, we use the common weight range of $\mathcal{U}\left(\left[-2,-0.5\right]\cup\left[0.5,2\right]\right)$
to avoid numerical instabilities due to more complex graphs. Table~\ref{tab:dense-highdim}
depicts that for dense graphs, $\ours$ makes almost no mistake while
the best baseline in this case, which is DAGMA, still has a significantly
large SHD of $67.6\pm8.0$. The performance gap is narrower in the
high-dimensional setting with 200 nodes, yet our method remains the
leading approach with an SHD of only 2, compared with an SHD of nearly
10 for the second-best method DAGMA.

\paragraph{Effect of sample size.}

We further investigate the behavior of $\ours$ under data scarcity
and redundancy. We again consider the difficult configuration of ER-8
graphs, and vary the sample size from very limited ($100$) to redundant
($5\,000$) in Figure~\ref{fig:sample-size}, where it is shown that
our method with just 100 samples can surpass DAGMA even with $5\,000$
samples.

\begin{table}

\caption{\textbf{Causal discovery performance under noise misspecification
on linear data with 30-node ER-2 graphs.} The numbers are \textit{mean
\textpm{} standard error} over 5 runs. \textbf{Bold}: best performance,
\uline{underline}: second-best performance.\label{tab:different-noises}}

\centering{}\resizebox{0.85\columnwidth}{!}{%
\begin{tabular}{ccccc}
\toprule 
 & \multicolumn{4}{c}{\textbf{SHD (lower is better)}}\tabularnewline
\midrule 
\textbf{Method\textbackslash Noise} & $\mathrm{Exp}\left(1\right)$ & $\mathrm{Gumbel}\left(0,1\right)$ & $\mathrm{Laplace}\left(0,1\right)$ & $\mathrm{Uniform}\left(-1,1\right)$\tabularnewline
\midrule
NOTEARS \citep{Zheng_etal_18DAGs} & $\phantom{0}6.0\pm\phantom{0}1.6$ & $\phantom{0}4.0\pm\phantom{0}1.9$ & $\phantom{0}3.0\pm\phantom{0}1.4$ & $\phantom{0}6.4\pm\phantom{0}4.3$\tabularnewline
DAGMA \citep{Bello_etal_22dagma} & \uline{\mbox{$\phantom{0}1.0\pm\phantom{0}1.0$}} & \textbf{$\mathbf{\phantom{0}0.2\pm\phantom{0}0.2}$} & \uline{\mbox{$\phantom{0}1.0\pm\phantom{0}1.0$}} & \uline{\mbox{$\phantom{0}4.8\pm\phantom{0}1.5$}}\tabularnewline
NOCURL \citep{Yu_etal_2021Dags} & $10.8\pm\phantom{0}0.8$ & $\phantom{0}6.6\pm\phantom{0}1.7$ & $\phantom{0}4.2\pm\phantom{0}1.0$ & $29.0\pm\phantom{0}3.0$\tabularnewline
COSMO \citep{Massidda_etal_2023Constraint} & $\phantom{0}5.4\pm\phantom{0}2.0$ & $\phantom{0}5.0\pm\phantom{0}2.3$ & $\phantom{0}7.6\pm\phantom{0}2.5$ & $\phantom{0}5.6\pm\phantom{0}1.8$\tabularnewline
RL-BIC \citep{Zhu_etal_2020Causal} & $66.8\pm13.2$ & $34.4\pm9.5$ & $31.4\pm10.1$ & $31.4\pm11.5$\tabularnewline
CORL \citep{Wang_etal_2021Ordering} & $12.0\pm\phantom{0}1.4$ & $15.4\pm\phantom{0}2.4$ & $15.6\pm2.2$ & $15.0\pm\phantom{0}0.9$\tabularnewline
RCL-OG \citep{Yang_etal_2023Reinforcement} & $77.3\pm13.5$ & $79.8\pm22.8$ & $41.3\pm16.4$ & $57.0\pm23.1$\tabularnewline
\midrule 
\textbf{$\ours$ (Ours)} & \textbf{$\mathbf{\phantom{0}0.4\pm\phantom{0}0.3}$} & \uline{\mbox{$\phantom{0}0.4\pm\phantom{0}0.4$}} & \textbf{$\mathbf{\phantom{0}0.8\pm\phantom{0}0.4}$} & \textbf{$\mathbf{\phantom{0}0.4\pm0.3}$}\tabularnewline
\bottomrule
\end{tabular}}
\end{table}

\paragraph{Model misspecification.}

Next, we consider the model misspecification scenarios when the Gaussian
noise assumption is violated. In Table~\ref{tab:different-noises},
we benchmark all methods on linear data with four types of non-Gaussian
noises. The results indicate that our method is still the most robust
to noise mis-specification, with an SHD of less than one in all four
cases, and is the lead performer in three out of four configurations.
In addition, we further study the performance of our method under
different model misspecification scenarios, including mismatched causal
model, noisy data, and the presence of hidden confounders, in Appendix~\ref{subsec:Model-Misspecification-Results}.

\subsection{Nonlinear Data with Gaussian Processes\label{subsec:Nonlinear-GP}}

In this section, to answer the question of whether our method can
operate beyond the standard linear-Gaussian setting, we follow the
evaluations in \citet{Zhu_etal_2020Causal,Wang_etal_2021Ordering,Yang_etal_2023Reinforcement}
to sample each causal mechanism $f_{i}$ from a Gaussian process with
an RBF kernel of unit bandwidth, and the noises follow normal distributions
with different variances sampled uniformly. We also follow \citet{Wang_etal_2021Ordering,Yang_etal_2023Reinforcement}
to apply Gaussian process regression using the RBF kernel with learnable
length scale and regularization $\alpha=1$ to calculate the BIC with
non-equal variances (Section~\ref{subsec:DAG-Scoring}) for RL-BIC,
CORL, RCL-OG, and our $\ours$ method. For NOTEARS, DAGMA, and COSMO,
we use their nonlinear versions where Multiple-layer Perceptrons (MLP)
are used to model nonlinear relationships.

\begin{table}[t]
\caption{\textbf{Causal discovery performance on nonlinear data with Gaussian
processes on 10-node ER-4 graphs. }The performance metrics are Structural
Hamming Distance (SHD, lower is better), False Detection Rate (FDR,
lower is better), and True Positive Rate (TPR, higher is better).
The numbers are \textit{mean \textpm{} standard error} over 5 runs.
\textbf{Bold}: best performance, \uline{underline}: second-best performance.
Since the graphs are dense and the noise is additive, we also study
the effect of pruning the output graphs with CAM pruning \citep{Buhlmann_etall_14Cam}.\label{tab:nonlinear}}

\begin{centering}
\par\end{centering}
\centering{}\resizebox{\columnwidth}{!}{%
\begin{tabular}{ccccccc}
\toprule 
 & \multicolumn{3}{c}{\textbf{No Pruning}} & \multicolumn{3}{c}{\textbf{CAM Pruning}}\tabularnewline
\midrule 
\textbf{Method} & \textbf{SHD ($\downarrow$)} & \textbf{FDR ($\downarrow$)} & \textbf{TPR ($\uparrow$)} & \textbf{SHD ($\downarrow$)} & \textbf{FDR ($\downarrow$)} & \textbf{TPR ($\uparrow$)}\tabularnewline
\midrule
NOTEARS \citep{Zheng_etal_20Learning} & $28.4\pm1.4$ & $0.33\pm0.06$ & $0.33\pm{0.04}$ & $28.6\pm{1.0}$ & $0.32\pm{0.06}$ & $0.32\pm{0.04}$\tabularnewline
DAGMA \citep{Bello_etal_22dagma} & $25.8\pm1.7$ & $0.32\pm0.04$ & $0.40\pm{0.05}$ & $26.0\pm{1.8}$ & $0.31\pm{0.04}$ & $0.39\pm{0.06}$\tabularnewline
NOCURL \citep{Yu_etal_2021Dags} & $35.2\pm0.9$ & $0.47\pm0.09$ & $0.15\pm{0.04}$ & $35.0\pm{0.8}$ & $0.46\pm{0.08}$ & $0.15\pm{0.04}$\tabularnewline
COSMO \citep{Massidda_etal_2023Constraint} & $26.4\pm2.0$ & $0.30\pm0.04$ & $0.39\pm{0.04}$ & $27.0\pm{2.5}$ & $0.28\pm{0.04}$ & $0.35\pm{0.05}$\tabularnewline
RL-BIC \citep{Zhu_etal_2020Causal} & $39.0\pm2.0$ & $\underline{0.06\pm0.06}$ & $0.05\pm{0.04}$ & $39.2\pm{1.8}$ & $\underline{0.05\pm{0.05}}$ & $0.04\pm{0.04}$\tabularnewline
CORL \citep{Wang_etal_2021Ordering} & $\phantom{0}8.4\pm1.8$ & $0.19\pm0.04$ & $0.90\pm{0.03}$ & $\phantom{0}9.6\pm{1.7}$ & $0.10\pm{0.04}$ & $0.82\pm{0.04}$\tabularnewline
RCL-OG \citep{Yang_etal_2023Reinforcement} & $\phantom{0}\underline{7.0\pm1.4}$ & $0.16\pm0.03$ & $\underline{0.94\pm{0.02}}$ & $\phantom{0}\underline{9.2\pm{1.3}}$ & $0.12\pm{0.04}$ & $\underline{0.84\pm{0.02}}$\tabularnewline
\midrule 
$\ours$ (Ours) & \textbf{$\mathbf{\phantom{0}0.8\pm0.4}$} & \textbf{$\mathbf{0.01\pm0.01}$} & \textbf{$\mathbf{0.99\pm{0.00}}$} & \textbf{$\mathbf{\phantom{0}4.6\pm{0.9}}$} & \textbf{$\mathbf{0.01\pm{0.01}}$} & \textbf{$\mathbf{0.89\pm{0.02}}$}\tabularnewline
\bottomrule
\end{tabular}}
\end{table}

The empirical results reported in Table~\ref{tab:nonlinear} verify
the effectiveness of our method even on nonlinear data. Our method
outperforms all other baselines in all metrics, either with or without
pruning. Most remarkably, even without pruning, our method correctly
identifies nearly every edge with an expected SHD lower than 1. However,
by using CAM pruning, there is a slight degrade in performance of
most methods, which could be due to CAM's inability to capture complex
causal mechanisms drawn from Gaussian processes.

\begin{table}

\begin{centering}
\caption{\textbf{Causal discovery performance on real-world flow cytometry
data \citep{Sachs_etall_05Causal} with 11 nodes, 17 edges, and 853
samples.} Running time is compared among RL-based methods. The figures
for RL-BIC are as originally reported. \textbf{Bold}: best performance,
\uline{underline}: second-best performance. Since the causal model
is potentially non-additive, we also consider CIT-based pruning with
KCIT \citep{Zhang_etal_11Kernel}.\label{tab:real-data}}
\par\end{centering}
\begin{centering}
\par\end{centering}
\begin{centering}
\resizebox{0.85\linewidth}{!}{%
\begin{tabular}{ccccccc}
\toprule 
 & \multicolumn{3}{c}{\textbf{CAM Pruning}} & \multicolumn{3}{c}{\textbf{CIT Pruning}}\tabularnewline
\midrule 
\multirow{2}{*}{\textbf{Method}} & \textbf{Total} & \textbf{Correct ($\uparrow$)} & \multirow{2}{*}{\textbf{SHD ($\downarrow$)}} & \textbf{Total} & \textbf{Correct ($\uparrow$)} & \multirow{2}{*}{\textbf{SHD ($\downarrow$)}}\tabularnewline
 & \textbf{edges} & \textbf{edges} &  & \textbf{edges} & \textbf{edges} & \tabularnewline
\midrule
NOTEARS \citep{Zheng_etal_20Learning} & $\phantom{0}8$ & $5$ & $13$ & $\phantom{0}7$ & $\underline{5}$ & $\underline{13}$\tabularnewline
DAGMA \citep{Bello_etal_22dagma} & $\phantom{0}6$ & $2$ & $15$ & $\phantom{0}6$ & $2$ & $15$\tabularnewline
NOCURL \citep{Yu_etal_2021Dags} & $\phantom{0}4$ & $2$ & $15$ & $\phantom{0}4$ & $2$ & $15$\tabularnewline
COSMO \citep{Massidda_etal_2023Constraint} & $\phantom{0}5$ & $2$ & $16$ & $\phantom{0}5$ & $2$ & $16$\tabularnewline
RL-BIC \citep{Zhu_etal_2020Causal} & $10$ & $\underline{7}$ & $\underline{11}$ & - & - & -\tabularnewline
CORL \citep{Wang_etal_2021Ordering} & $\phantom{0}9$ & $3$ & $14$ & $10$ & $3$ & $15$\tabularnewline
RCL-OG \citep{Yang_etal_2023Reinforcement} & $\phantom{0}9$ & $5$ & $13$ & $\phantom{0}9$ & $\underline{5}$ & $\underline{13}$\tabularnewline
\midrule 
$\ours$ (Ours) & $10$ & \textbf{8} & \textbf{$\mathbf{10}$} & $\phantom{0}9$ & \textbf{$\mathbf{8}$} & \textbf{$\phantom{0}\mathbf{9}$}\tabularnewline
\bottomrule
\end{tabular}}
\par\end{centering}
\end{table}

Furthermore, following \citet{Lachapelle_etal_2020Gradient}, we also
study the case of causal model misspecification with Post-nonlinear
models \citep{Zhang_etal_2009Identifiability} in Appendix~\ref{subsec:Model-Misspecification-Results}.
In addition, nonlinear models generated using MLPs are also studied
in Appendix~\ref{subsec:nonlinear-mlp}.

\subsection{Real Data}

Next, to confirm the validity of our method past synthetic data, we
evaluate it on the popular benchmark flow cytometry dataset \citep{Sachs_etall_05Causal},
which involves a protein signaling network based on expression levels
of proteins and phospholipids. We employ the observational partition
of the dataset with 853 samples, 11 nodes, and 17 edges.

The empirical results provided in Table~\ref{tab:real-data} show
that our method $\ours$ both achieves the best SHD and number of
correct edges among all approaches. Specifically, using CAM pruning
under the assumption of generalized additive noise models, we achieve
the lowest SHD of 10 compared with the second-best of 11 by RL-BIC.
Meanwhile, when using CIT-based pruning, we can even further reduce
the SHD to 9, with 8 out of 9 identified edges are correct. This
is a state-of-the-art level of SHD among existing studies on this
dataset.

\subsection{Runtime Analysis}

\begin{figure}

\begin{tabular}{>{\centering}m{0.35\columnwidth}>{\centering}m{0.65\columnwidth}}
\resizebox{.35\columnwidth}{!}{%
\begin{tabular}{cccc}
\toprule 
\textbf{Nodes} & \textbf{ALIAS} & \textbf{CORL} & \textbf{RCL-OG}\tabularnewline
\midrule
$\phantom{0}50$ & $0.3$ & $\phantom{0}45.7$ (${\color{violet}\phantom{0,\!}135.7}\mathbin{\color{violet}\times}$) & $\phantom{0}46.4$ (${\color{violet}\phantom{0,\!}137.7}\mathbin{\color{violet}\times}$)\tabularnewline
$100$ & $0.2$ & $146.8$ (${\color{violet}\phantom{0,\!}900.0}\mathbin{\color{violet}\times}$) & $188.6$ (${\color{violet}1,\!155.9}\mathbin{\color{violet}\times}$)\tabularnewline
$150$ & $0.3$ & $359.6$ (${\color{violet}1,\!261.5}\mathbin{\color{violet}\times}$) & $471.3$ (${\color{violet}{\color{violet}1,\!653.2}\mathbin{\color{violet}\times}}$)\tabularnewline
$200$ & $0.4$ & $640.6$ (${\color{violet}1,\!596.8}\mathbin{\color{violet}\times}$) & $915.5$ (${\color{violet}2,\!282.0}\mathbin{\color{violet}\times}$)\tabularnewline
\bottomrule
\end{tabular}} & \includegraphics[width=0.66\columnwidth]{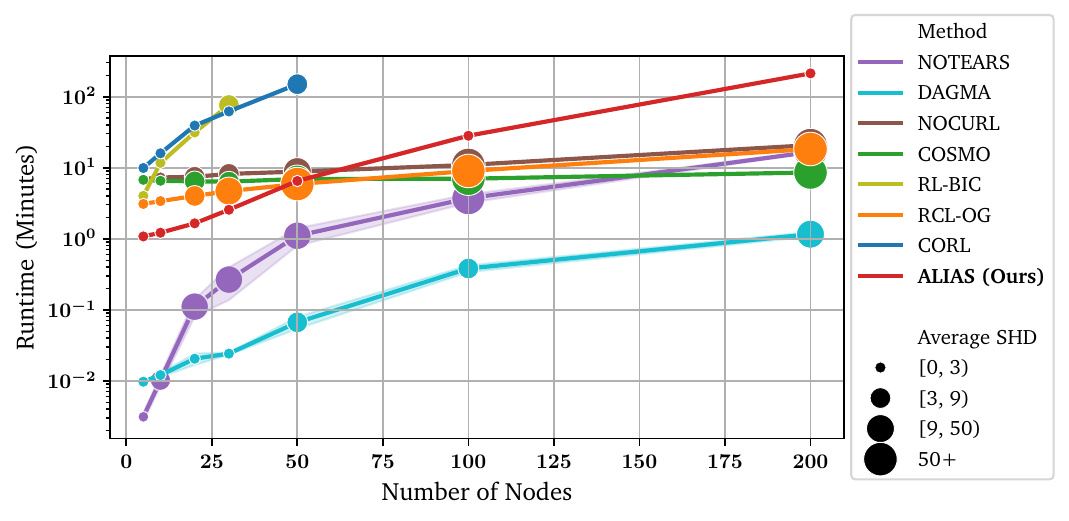}\tabularnewline
\textbf{(a) Average sampling time for each DAG in milliseconds.} & \textbf{(b) Causal Discovery Runtime.}\tabularnewline
\end{tabular}

\caption{\textbf{Runtime analysis of $\protect\ours$.} (a) We demonstrate
the efficiency of our $\mathcal{O}\left(d^{2}\right)$ sampling technique
compared with the $\mathcal{O}\left(d^{3}\right)$ approaches in CORL
and RCL-OG. (b) We study the runtime of all methods with respect to
graph size. The shaded areas depicts\textit{ }standard errors over
5 random linear-Gaussian datasets with the regular weight range $\mathcal{U}\left(\left[-2,-0.5\right]\cup\left[0.5,2\right]\right)$
on ER-2 graphs.\label{fig:Runtime-analysis}}
\end{figure}

Here, we analyze the efficiency of $\ours$ in details. First, to
show the significance of our optimal quadratic complexity for sampling
DAGs, our Figure~\ref{fig:Runtime-analysis}a compares the runtime
of our policy with the autoregressive sampling approaches in CORL
and RCL-OG with cubic complexity. It can be seen that, our DAG sampling
policy is much faster than other approaches and does not significantly
slow down with increasing graph sizes. Meanwhile, CORL and RCL-OG
are nearly 140 times slower than $\ours$ at 50 nodes, and the speedup
ratio drastically increases with the growth of the graph. Second,
in Figure~\ref{fig:Runtime-analysis}b, we detail the runtime and
performance of all methods with varying numbers of nodes. For small-
to moderate-sized graphs of up to 50 nodes, our method is even faster
than gradient-based methods NOCURL and COSMO. For larger graphs, RL-BIC
and CORL become computationally expensive very rapidly, and while
other methods become faster than $\ours$, their performance quickly
degrade with significantly larger SHDs than our method.

\subsection{Ablation Studies}

In Figure~\ref{fig:learning-curve}, we study the effect of the
choice of graph scorer, RL method, and number of training steps onto
the performance of $\ours$ compared with DAGMA and CORL as the representatives
for continuous optimization and RL-based approaches. It can be seen
that all variants of our method surpass the baselines, using as few
as $1\,000$ training steps. While all variants perform equivalently
well, PPO proves to be a better choice than A2C, with both variants
$\text{PPO}+\text{BIC}$ and $\text{PPO}+\text{LS}$ can reach very
close to zero SHD, while those of A2C are not as performant (Figure~\ref{fig:learning-curve}b).
The influence of other hyper-parameters can be found in Appendix~\ref{subsec:Ablation-Studies}.

Furthermore, to show that the effectiveness of $\ours$ is not only
thanks to the $\vecdag$ parametrization alone, but also the application
of RL, as opposed to the gradient-based optimization approach commonly
employed in the literature, we replace RL in our method with a gradient-based
optimizer, which is popular among modern causal discovery methods,
and compare the performances. However, since Vec2DAG is discrete,
which renders the objective non-differentiable as is, we make slight
modifications to make it amenable for continuous optimization, and
the adapted version for linear data is given as:

\[
\mathrm{Vec2DAG}^{\textrm{cont.}}:=(\mathbf{E(z)+E(z)^{\top}})\odot H(\mathrm{grad}(\mathbf{p(z)})),
\]

which still represents the weighted adjacency matrix of a DAG. Still,
$H(\cdot)$ is not differentiable, so we further use the Straight-Through
estimator \citep{Bengio_etal_2013Estimating} to estimate its gradients.
Specifically, we use $H(\mathrm{grad}\mathbf{(p(z))})$ for the forward
pass whereas the gradients of the inner part $\frac{\partial\mathrm{grad}\mathbf{(p(z))}}{\partial\mathbf{z}}$
is used for the backward pass. This is not done similarly for the
first term in $\vecdag$ because that would require an additional
weight matrix to represent linear coefficients, which is redundant
compared with the above. Then, since the BIC score used in our RL
approach is also non-differentiable, we adopt a likelihood-based loss
similar to BIC as follows (which is also used in, e.g., GOLEM, \citealp{Ng_etal_2020Role}):

\[
\mathcal{L}(\mathbf{z})=\ln\left(\frac{1}{n\times d}\Vert\mathbf{X}-\mathbf{X}\cdot\mathrm{Vec2DAG}^{\textrm{cont.}}(\mathbf{z})\Vert_{2}^{2}\right)+\lambda_{1}|\mathbf{z}|,
\]

where $\lambda_{1}$ is the sparsity regularization coefficient. We
minimize this loss until convergence using the Adam optimizer \citep{Kingma_2014Adam}
(same as NOCURL, COSMO, and the RL algorithm PPO in our method).

We provide the results in Table~\ref{tab:Role-of-RL} with a wide
range of hyperparameter choices for the above approach. It can be
seen that even in this simple case, the continuous optimization approach
performs poorly and cannot compete with our RL approach, confirming
that $\ours$'s effectiveness is attributed greatly by RL, not just
$\vecdag$ alone.

\begin{table}

\caption{\textbf{Role of RL in $\protect\ours$.} We replace RL with continuous
optimization using Adam \citep{Kingma_2014Adam} and compare with
the RL version. The numbers are \textit{mean \textpm{} standard error}
over 5 random datasets on 10-node ER-2 graphs.\label{tab:Role-of-RL}}

\begin{centering}
\resizebox{\columnwidth}{!}{%
\begin{tabular}{cccc}
\toprule 
\textbf{Method} & \textbf{SHD ($\downarrow$)} & \textbf{FDR ($\downarrow$)} & \textbf{TPR ($\uparrow$)}\tabularnewline
\midrule
Vec2DAG + continuous optimization ($\text{lr}=10^{-2},\lambda_{1}=10^{-3}$) & $14.0\pm{0.6}$ & $0.52\pm{0.05}$ & $0.42\pm{0.06}$\tabularnewline
Vec2DAG + continuous optimization ($\text{lr}=10^{-2},\lambda_{1}=10^{-5}$) & $\phantom{0}9.6\pm{2.6}$ & $0.36\pm{0.06}$ & $0.65\pm{0.08}$\tabularnewline
Vec2DAG + continuous optimization ($\text{lr}=10^{-2},\lambda_{1}=10^{-7}$) & $11.8\pm{1.6}$ & $0.45\pm{0.08}$ & $0.48\pm{0.06}$\tabularnewline
Vec2DAG + continuous optimization ($\text{lr}=10^{-3},\lambda_{1}=10^{-3}$) & $10.0\pm{1.7}$ & $0.37\pm{0.05}$ & $0.56\pm{0.06}$\tabularnewline
Vec2DAG + continuous optimization ($\text{lr}=10^{-3},\lambda_{1}=10^{-5}$) & $11.8\pm{2.4}$ & $0.44\pm{0.06}$ & $0.55\pm{0.08}$\tabularnewline
Vec2DAG + continuous optimization ($\text{lr}=10^{-3},\lambda_{1}=10^{-7}$) & $\phantom{0}8.6\pm{0.4}$ & $0.32\pm{0.03}$ & $0.62\pm{0.01}$\tabularnewline
Vec2DAG + continuous optimization ($\text{lr}=10^{-4},\lambda_{1}=10^{-3}$) & $14.2\pm{1.9}$ & $0.53\pm{0.06}$ & $0.49\pm{0.05}$\tabularnewline
Vec2DAG + continuous optimization ($\text{lr}=10^{-4},\lambda_{1}=10^{-5}$) & $12.0\pm{2.1}$ & $0.45\pm{0.05}$ & $0.59\pm{0.04}$\tabularnewline
Vec2DAG + continuous optimization ($\text{lr}=10^{-4},\lambda_{1}=10^{-7}$) & $12.8\pm{2.8}$ & $0.45\pm{0.08}$ & $0.55\pm{0.04}$\tabularnewline
Vec2DAG + continuous optimization ($\text{lr}=10^{-5},\lambda_{1}=10^{-3}$) & $12.6\pm{3.7}$ & $0.49\pm{0.10}$ & $0.55\pm{0.12}$\tabularnewline
Vec2DAG + continuous optimization ($\text{lr}=10^{-5},\lambda_{1}=10^{-5}$) & $13.0\pm{3.2}$ & $0.51\pm{0.09}$ & $0.51\pm{0.10}$\tabularnewline
Vec2DAG + continuous optimization ($\text{lr}=10^{-5},\lambda_{1}=10^{-7}$) & $12.2\pm{2.5}$ & $0.49\pm{0.08}$ & $0.51\pm{0.08}$\tabularnewline
\midrule 
\textbf{Vec2DAG + RL} & $\mathbf{\phantom{0}0.0\pm0.0}$ & $\mathbf{\phantom{0}0.00\pm0.00}$ & $\mathbf{1.00\pm0.00}$\tabularnewline
\bottomrule
\end{tabular}}
\par\end{centering}
\end{table}

\section{Conclusions\label{sec:Conclusions}}

In this study, a novel causal discovery method based on RL is proposed.
With the introduction of a new DAG characterization that bridges an
unconstrained continuous space to the constrained DAG space, we devise
an RL policy that can generate DAGs efficiently without any enforcement
of the acyclicity constraint, which helps improve the search for the
optimal score drastically. Experiments on a wide array of both synthetic
and real datasets confirm the effectiveness of our method compared
with state-of-the-art baselines.

Regarding limitations, the RL approaches in our study, which are online
RL methods, may be limited in sample efficiency, as exploration data
is not effectively ultilized to prioritize visiting promising DAGs,
thus potentially requiring more explorations than needed to reach
the optimal DAG. Towards this end, more sample-efficient RL approaches,
such as Optimistic PPO \citep{Cai_etal_20Provably}, or reward redesign
can be considered to enhance exploration in our method, and thus further
improve its efficiency.

Future work may involve deepening the understanding on the convergence
properties of our method and extending it to more intriguing settings
like causal discovery with interventional data and hidden variables.

\vfill{}

\pagebreak{}

\bibliographystyle{tmlr}
\bibliography{ref}
\vfill{}
\pagebreak{}

\appendix

\section*{Appendix}

\begin{figure}
\begin{centering}
\noindent\begin{minipage}[t]{1\columnwidth}%
\begin{lstlisting}[language=Python]
def Vec2DAG(z) -> np.ndarray:
    p = z[:d]					# R^d
	E = np.zeros((d, d))
	E[np.triu_indices(d, -1)] = z[d:]	# R^(d(d-1)/2)

    A = (E + E.T > 0) * (p[:, None] < p[None, :])
    return A
\end{lstlisting}%
\end{minipage}
\par\end{centering}
\caption{Unconstrained DAG parameterization. This function takes as input a
real-valued vector $\mathbf{z}\in\mathbb{R}^{d\cdot\left(d+1\right)/2}$
and deterministically transforms it into an adjacency matrix of a
$d$-node DAG.\label{fig:DAG-transformation}}
\end{figure}

\section{Details about BIC scores\label{subsec:Derivation-of-BIC}}

\subsection{Non-equal variances BIC}

Recall that the additive noise model under Gaussian noise is given
by $X_{i}:=f_{i}\left(\mathbf{X}_{\mathrm{pa}_{i}}\right)+E_{i}$,
where $E_{i}\sim\mathcal{N}\left(0,\sigma_{i}^{2}\right)$. This implies
$X_{i}\sim\mathcal{N}\left(f_{i}\left(\mathbf{X}_{\pa_{i}}\right),\sigma_{i}^{2}\right)$
and the log-likelihood of an empirical dataset $\mathcal{D}=\left\{ \mathbf{x}^{\left(k\right)}\right\} _{k=1}^{n}$
is given by
\begin{align}
\mathcal{L} & =\ln p\left(\mathcal{D}\mid f,\bm{\sigma},\mathcal{G}\right)\label{eq:log-likelihood}\\
 & =-\frac{1}{2}\sum_{k=1}^{n}\sum_{i=1}^{d}\frac{\left(x_{i}^{\left(k\right)}-f_{i}\left(x_{\pa_{i}}^{\left(k\right)}\right)\right)^{2}}{\sigma_{i}^{2}}\\
 & \phantom{=}-\frac{n}{2}\sum_{i=1}^{d}\ln\sigma_{i}^{2}+\text{const},
\end{align}

\noindent where the constant does not depend on any variable.

The maximum likelihood estimator for $f_{i}$ can be found via least
square methods, and that of $\sigma_{i}^{2}$ can be found by solving
$\frac{\partial\mathcal{L}}{\partial\sigma_{i}^{2}}=0$, which yields
\begin{equation}
\hat{\sigma}_{i}^{2}=\frac{1}{n}\underbrace{\sum_{k=1}^{n}\left(x_{i}^{\left(k\right)}-\hat{f}_{i}\left(x_{\pa_{i}}^{\left(k\right)}\right)\right)^{2}}_{\mathrm{SSR}_{i}}.
\end{equation}

Plugging this back to Eqn.~(\ref{eq:log-likelihood}) gives
\begin{equation}
\mathcal{\hat{L}}=-\frac{n}{2}\sum_{i=1}^{d}\ln\frac{\mathrm{SSR}_{i}}{n}+\text{const}.
\end{equation}

Finally, we obtain the BIC score for the non-equal variances case
by incorporating this into Eqn.~(\ref{eq:bic}):

\begin{equation}
\text{BIC}\left(\mathcal{D},\mathcal{G}\right)=n\sum_{i=1}^{d}\ln\frac{\mathrm{SSR}_{i}}{n}+\left|\mathcal{G}\right|\ln n+\text{const},
\end{equation}

\subsection{Equal variances BIC}

Similarly to the unequal variances case, by assuming $\sigma_{1}=\ldots=\sigma_{d}=\sigma$,
we solve for $\frac{\partial\mathcal{L}}{\partial\sigma^{2}}=0$ and
obtain
\begin{equation}
\hat{\sigma}^{2}=\frac{1}{nd}\sum_{i=1}^{d}\underbrace{\sum_{k=1}^{n}\left(x_{i}^{\left(k\right)}-\hat{f}_{i}\left(x_{\pa_{i}}^{\left(k\right)}\right)\right)^{2}}_{\mathrm{SSR}_{i}}.
\end{equation}

Substituting this estimate into Eqn.~(\ref{eq:log-likelihood}) yields
us with
\begin{equation}
\hat{\mathcal{L}}=-\frac{nd}{2}\ln\frac{\sum_{i=1}^{d}\mathrm{SSR}_{i}}{nd}.
\end{equation}

For the last step, the BIC score for the equal variance case is given
by substitution as

\begin{equation}
\text{BIC}\left(\mathcal{D},\mathcal{G}\right)=nd\ln\frac{\sum_{i=1}^{d}\mathrm{SSR}_{i}}{nd}+\left|\mathcal{G}\right|\ln n+\text{const}.
\end{equation}

\section{Proofs}

\subsection{Proof of Lemma~\ref{lem:BIC-consistency} \label{subsec:Proof-of-BIC-consistency}}
\begin{lem*}
Let $\mathcal{G}^{\ast}$ be the ground truth DAG of an identifiable
SCM satisfying causal minimality \citep{Peters_etal_2014Causal} (i.e.,
there are no redundant edges) inducing the dataset $\mathcal{D}$,
and let $n$ be the sample size of $\mathcal{D}$. Then, in the limit
of large $n$, $\mathcal{S}_{\text{BIC}}\left(\mathcal{D},\mathcal{G}^{\ast}\right)>\mathcal{S}_{\text{BIC}}\left(\mathcal{D},\mathcal{G}\right)$
for any $\mathcal{G}\neq\mathcal{G}^{\ast}$.
\end{lem*}
\begin{proof}
Let $\psi^{*}$ be the parameter of the causal model generating the
dataset $\mathcal{D}$. Because the causal model is identifiable and
$\left(\psi^{*},\mathcal{G}^{*}\right)$ are the true parameters generating
the data $\mathcal{D}$, the likelihood $p\left(\mathcal{D}\mid\psi^{*},\mathcal{G}^{*}\right)$
is the highest possible. For any incorrect DAG $\mathcal{G}\neq\mathcal{G}^{*}$,
there only exists parameter $\psi$ such that $p\left(\mathcal{D}\mid\psi^{*},\mathcal{G}^{*}\right)=p\left(\mathcal{D}\mid\psi,\mathcal{G}\right)$
if $\mathcal{G}^{*}\subset\mathcal{G}$, because of causal minimality.
Otherwise, the difference between the likelihoods are given by
\begin{align}
\ln p\left(\mathcal{D}\mid\psi^{*},\mathcal{G}^{*}\right)-\ln p\left(\mathcal{D}\mid\psi,\mathcal{G}\right) & =\sum_{k=1}^{n}\ln p\left(\mathbf{x}^{\left(k\right)}\mid\psi^{*},\mathcal{G}^{*}\right)-\ln p\left(\mathbf{x}^{\left(k\right)}\mid\psi,\mathcal{G}\right)\\
 & =n\cdot\underbrace{\mathrm{KL}\left(p\left(\mathcal{D}\mid\psi^{*},\mathcal{G}^{*}\right)\parallel p\left(\mathcal{D}\mid\psi,\mathcal{G}\right)\right)}_{k\left(\mathcal{G}^{*},\mathcal{G}\right)}+o\left(n\right)
\end{align}

\noindent where the second equality follows from the asymptotic behavior
of the log-likelihoods. Therefore, for any incorrect DAG $\mathcal{G}\varsupsetneq\mathcal{G}^{*}$:
\begin{equation}
\mathcal{S}_{\text{BIC}}\left(\mathcal{D},\mathcal{G}^{\mathcal{*}}\right)-\mathcal{S}_{\text{BIC}}\left(\mathcal{D},\mathcal{G}\right)=n\cdot k\left(\mathcal{G}^{*},\mathcal{G}\right)+o\left(n\right)+\left(\left|\mathcal{G}\right|-\left|\mathcal{G}^{*}\right|\right)\ln n.
\end{equation}
As $n\rightarrow\infty$, if $\mathcal{G}^{*}\nsubseteq\mathcal{G}$
then the KL divergence term grows linearly with $n$, dominating the
logarithmic growth of the penalty term. On the other hand, if $\mathcal{G}^{*}\subset\mathcal{G}$
then the difference between the likelihoods vanishes and is therefore
dominated by the penalty term. Thus, for any $\mathcal{G}\neq\mathcal{G}^{*}$,
the BIC score satisfies $\mathcal{S}_{\text{BIC}}\left(\mathcal{D},\mathcal{G}^{\ast}\right)>\mathcal{S}_{\text{BIC}}\left(\mathcal{D},\mathcal{G}\right)$
as $n\rightarrow\infty$.
\end{proof}

\subsection{Proof of Theorem~\ref{thm:surjective}\label{subsec:Proof-of-Theorem}}
\begin{thm*}
For all $d\in\mathbb{N}^{+}$, let $\vecdag_{d}:\mathbb{S}_{d}\rightarrow\left\{ 0,1\right\} ^{d\times d}$
be defined as in Eq.~\ref{eq:vec2dag}. Then, $\mathrm{Im}\left(\vecdag_{d}\right)=\mathbb{D}_{d}$.
\end{thm*}
To show that $\vecdag$ is surjective, we first show that every point
in the parameter space maps to a directed graph without any cycle,
and vice-versa, for any DAG, there always exists a vector that maps
to it.
\begin{lem*}
\label{lem:necessary}For all $d\in\mathbb{N}^{+}$, let $\vecdag_{d}:\mathbb{S}_{d}\rightarrow\left\{ 0,1\right\} ^{d\times d}$
be defined as in Eq.~(\ref{eq:vec2dag}). Then, $\vecdag_{d}\left(\mathbf{z}\right)\in\mathbb{D}_{d}$
$\forall\mathbf{z}\in\mathbb{S}_{d}$.
\end{lem*}
\begin{proof}
The proof follows the same argument with Theorem 2.1 of \citet{Yu_etal_2021Dags}.
Specifically, let $\mathbf{A}=\vecdag_{d}\left(\mathbf{z}\right)=H\left(\mathbf{E}\left(\mathbf{z}\right)+\mathbf{E}\left(\mathbf{z}\right)^{\top}\right)\odot H\left(\mathrm{grad}\left(\mathbf{p}\left(\mathbf{z}\right)\right)\right)$.
Then, the necessary condition for an edge from $i$ to $j$ to exist
is $p_{i}<p_{j}$. By contradiction, assuming there exists a cycle
$\left(i_{1}\rightarrow\ldots\rightarrow i_{k}\rightarrow i_{1}\right)$
in $\mathbf{A}$, then it follows that $p_{i_{1}}<\cdots<p_{i_{k}}<p_{i_{1}}$.
This contradicts with the total-ordering property of real values,
thus concluding our argument.
\end{proof}
\begin{lem*}
\label{lem:sufficient}For all $d\in\mathbb{N}^{+}$, let $\vecdag_{d}:\mathbb{S}_{d}\rightarrow\left\{ 0,1\right\} ^{d\times d}$
be defined as in Eq.~(\ref{eq:vec2dag}). Then, for any DAG $\mathbf{A}\in\mathbb{D}_{d}$
and $\epsilon>0$, there exists $\mathbf{z}\in\left(-\epsilon,\epsilon\right)^{d\cdot\left(d+1\right)/2}\subset\mathbb{S}_{d}$
such that $\vecdag_{d}\left(\mathbf{z}\right)=\mathbf{A}$.
\end{lem*}
\begin{proof}
We prove by construction. Let $\an_{i}=\left\{ j\neq i\mid\text{there is a path from }j\text{ to }i\right\} $
be the set of ancestors of $i$. It follows that if $i\rightarrow j$
is in $\mathbf{A}$ then every ancestor of $i$ is an ancestor of
$j$, so $\an_{i}\subset\an_{j}$ and $\an_{j}=\an_{i}\cup\left\{ i\right\} \cup\left(\an_{j}\setminus\an_{i}\setminus\left\{ i\right\} \right)$,
which means $\left|\an_{i}\right|<\left|\an_{j}\right|$. Therefore,
we can construct the unnormalized node potentials as $\tilde{p}_{i}=\left|\an_{i}\right|$
and normalize it to fit into the $\left(-\epsilon,\epsilon\right)$
range as
\begin{equation}
p_{i}:=\frac{\tilde{p_{i}}-\min\tilde{p_{i}}}{\max\tilde{p_{i}}-\min\tilde{p_{i}}}\times\epsilon-\frac{1}{2}\epsilon.
\end{equation}

We then construct the upper-triangular edge potentials matrix as
\begin{equation}
E_{ij}=\begin{cases}
\nicefrac{1}{2}\epsilon & \text{if }i<j\text{ and }A_{ij}+A_{ji}=1,\\
-\nicefrac{1}{2}\epsilon & \text{otherwise}.
\end{cases}\label{eq:construct-E}
\end{equation}

\noindent We now verify that this pair of potentials leads to the
exact binary matrix $\mathbf{A}$ via Eq.~(\ref{eq:vec2dag}). First,
$E_{ij}+E_{ji}=\nicefrac{1}{2}\epsilon>0$ if $i$ and $j$ are directly
connected in $\mathbf{A}$, and $E_{ij}+E_{ji}=-\nicefrac{1}{2}\epsilon<0$
otherwise, thus $H\left(\mathbf{E}+\mathbf{E}^{\top}\right)$ is the
undirected version of $\mathbf{A}$. Second, for any directed edge
$j\rightarrow i$ in $\mathbf{A}$, we have $\left|\an_{i}\right|<\left|\an_{j}\right|$,
which leads to $\mathrm{grad}\left(\mathbf{p}\right)_{ij}=p_{j}-p_{i}>0$,
and thus $H\left(\mathrm{grad}\left(\mathbf{p}\right)\right)$ encodes
the direction for every directed edge in $\mathbf{A}$. Therefore,
using the Hadamard product, we mask out the edges in $H\left(\mathrm{grad}\left(\mathbf{p}\right)\right)$
that do not exist in $\mathbf{A}$, and at the same time give direction
to any available edge in $H\left(\mathbf{E}+\mathbf{E}^{\top}\right)$,
resulting in exactly $\mathbf{A}$.
\end{proof}
Lemmas~\ref{lem:necessary} and \ref{lem:sufficient} then completes
our proof of Theorem~\ref{thm:surjective}.

\subsection{Proof of Lemma~\ref{lemma:invariance}\label{subsec:Proof-of-Invariance}}
\begin{lem*}
(Scaling and Translation Invariance). For all $d\in\mathbb{N}^{+}$,
let $\vecdag_{d}:\mathbb{S}_{d}\rightarrow\left\{ 0,1\right\} ^{d\times d}$
be defined as in Eq.~\ref{eq:vec2dag}. Then, for all $\mathbf{z}\in\mathbb{S}_{d}$,
$\alpha>0$, and ${\bf \bm{\beta}}\in\mathbb{\mathbb{S}}_{d}$ such
that $\left|\mathbf{p}\left(\bm{\beta}\right)_{i}\right|<\nicefrac{1}{2}\min_{j}\left|\mathbf{p}\left(\mathbf{z}\right)_{i}-\mathbf{p}\left(\mathbf{z}\right)_{j}\right|$
and $\left|\mathbf{E}\left(\bm{{\bf \beta}}\right)_{ij}\right|<\left|\mathbf{E}\left(\mathbf{z}\right)_{ij}\right|$
$\forall i,j$, we have $\mathbb{\vecdag}_{d}\left(\mathbf{z}\right)=\vecdag_{d}\left(\alpha\cdot(\mathbf{z}+{\bf \bm{\beta}})\right)$.
\end{lem*}
\begin{proof}
\noindent We first show that $\vecdag_{d}\left(\mathbf{z}\right)=\vecdag_{d}\left(\alpha\cdot\mathbf{z}\right)$
$\forall\alpha>0$. Let $\mathbf{A}=\vecdag_{d}\left(\mathbf{z}\right)$.
This follows from the fact that 
\begin{align}
\vecdag_{d}\left(\mathbf{z}\right)_{ij}=1 & \Leftrightarrow\begin{cases}
\mathbf{p}\left(\mathbf{z}\right)_{i}<\mathbf{p}\left(\mathbf{z}\right)_{j}\\
\mathbf{E}\left(\mathbf{z}\right)_{ij}+\mathbf{E}\left(\mathbf{z}\right)_{ji}>0
\end{cases}\\
 & \Leftrightarrow\begin{cases}
\alpha\mathbf{p}\left(\mathbf{z}\right)_{i}<\alpha\mathbf{p}\left(\mathbf{z}\right)_{j}\\
\alpha\mathbf{E}\left(\mathbf{z}\right)_{ij}+\alpha\mathbf{E}\left(\mathbf{z}\right)_{ji}>0
\end{cases}\\
 & \Leftrightarrow\vecdag_{d}\left(\alpha\cdot\mathbf{z}\right)_{ij}=1,\label{eq:scale}
\end{align}
or equivalently, $\vecdag_{d}\left(\mathbf{z}\right)=\vecdag_{d}\left(\alpha\cdot\mathbf{z}\right)$.

Next, we prove $\vecdag_{d}\left(\mathbf{z}\right)=\vecdag_{d}\left(\mathbf{z}+{\bf \bm{\beta}}\right)$
if $\left|\mathbf{p}\left(\bm{\beta}\right)_{i}\right|<\nicefrac{1}{2}\min_{j}\left|\mathbf{p}\left(\mathbf{z}\right)_{i}-\mathbf{p}\left(\mathbf{z}\right)_{j}\right|$
and $\left|\mathbf{E}\left(\bm{{\bf \beta}}\right)_{ij}\right|<\left|\mathbf{E}\left(\mathbf{z}\right)_{ij}\right|$
$\forall i,j$. These conditions state that translating the potentials
by an amount $\bm{\beta}$ that does not change the ordering of $\mathbf{p}$
or the element-wise positivity of $\mathbf{E}$ also leads to the
same DAG. First, it can be seen that if $p_{i}<p_{j}$ and $p_{i}^{\prime},p_{j}^{\prime}\in\left(-\nicefrac{1}{2}\left(p_{j}-p_{i}\right),\nicefrac{1}{2}\left(p_{j}-p_{i}\right)\right)$,
then $p_{i}+p_{i}^{\prime}<\nicefrac{1}{2}\left(p_{i}+p_{j}\right)<p_{j}+p_{j}^{\prime}.$
Therefore, $\mathbf{p}\left(\mathbf{z}\right)_{i}<\mathbf{p}\left(\mathbf{z}\right)_{j}\Leftrightarrow\mathbf{p}\left(\mathbf{z}\right)_{i}+\mathbf{p}\left(\bm{\beta}\right)_{i}<\mathbf{p}\left(\mathbf{z}\right)_{j}+\mathbf{p}\left(\bm{\beta}\right)_{j}$
if $\left|\mathbf{p}\left(\bm{\beta}\right)_{i}\right|<\nicefrac{1}{2}\left|\mathbf{p}\left(\mathbf{z}\right)_{i}-\mathbf{p}\left(\mathbf{z}\right)_{j}\right|$,
which also applies when $\left|\mathbf{p}\left(\bm{\beta}\right)_{i}\right|<\nicefrac{1}{2}\min_{j}\left|\mathbf{p}\left(\mathbf{z}\right)_{i}-\mathbf{p}\left(\mathbf{z}\right)_{j}\right|$.

Similarly, $\mathbf{E}\left(\mathbf{z}\right)_{ij}>0\Leftrightarrow\mathbf{E}\left(\mathbf{z}\right)_{ij}+\mathbf{E}\left(\bm{\beta}\right)_{ij}>0$
if $\left|\mathbf{E}\left(\bm{{\bf \beta}}\right)_{ij}\right|<\left|\mathbf{E}\left(\mathbf{z}\right)_{ij}\right|$.
Combining these two points entails that if $\left|\mathbf{E}\left(\bm{{\bf \beta}}\right)_{ij}\right|<\left|\mathbf{E}\left(\mathbf{z}\right)_{ij}\right|$
$\forall i,j$ then we obtain
\begin{align}
\vecdag_{d}\left(\mathbf{z}\right)_{ij}=1 & \Leftrightarrow\begin{cases}
\mathbf{p}\left(\mathbf{z}\right)_{i}<\mathbf{p}\left(\mathbf{z}\right)_{j}\\
\mathbf{E}\left(\mathbf{z}\right)_{ij}+\mathbf{E}\left(\mathbf{z}\right)_{ji}>0
\end{cases}\\
 & \Leftrightarrow\begin{cases}
\mathbf{p}\left(\mathbf{z}\right)_{i}+\mathbf{p}\left(\bm{\beta}\right)_{i}<\mathbf{p}\left(\mathbf{z}\right)_{j}+\mathbf{p}\left(\bm{\beta}\right)_{j}\\
\left(\mathbf{E}\left(\mathbf{z}\right)_{ij}+\mathbf{E}\left(\bm{\beta}\right)_{ij}\right)+\left(\mathbf{E}\left(\mathbf{z}\right)_{ji}+\mathbf{E}\left(\bm{\beta}\right)_{ji}\right)>0
\end{cases}\\
 & \Leftrightarrow\vecdag_{d}\left(\mathbf{z}+\bm{\beta}\right)_{ij}=1.\label{eq:translation}
\end{align}

\noindent Finally, combining Eqs.~(\ref{eq:scale}) and (\ref{eq:translation})
concludes our proof for Lemma~\ref{lemma:invariance}:

\[
\vecdag_{d}\left(\alpha\cdot(\mathbf{z}+{\bf \bm{\beta}})\right)=\vecdag_{d}\left(\mathbf{z}+{\bf \bm{\beta}}\right)=\vecdag_{d}\left(\mathbf{z}\right).
\]
\end{proof}

\subsection{Proof of Lemma~\ref{lemma:proximity}\label{subsec:Proof-of-proximity}}
\begin{lem*}
(Proximity between DAGs). Let $\mathbf{z}\in\mathbb{S}_{d}$. Then,
for any DAG $\mathbf{A}\in\mathbb{D}_{d}$ and $\epsilon>0$, there
exists $\mathbf{z}_{\mathbf{A}}$ in the unit ball $B\left(\infty;\left\Vert \mathbf{z}\right\Vert _{\infty}+\epsilon\right)$
around $\mathbf{z}$ such that $\vecdag_{d}\left(\mathbf{z}_{\mathbf{A}}\right)=\mathbf{A}$.
\end{lem*}
\begin{proof}
Lemma~\ref{lem:sufficient} established that there exists a representation
$\mathbf{z}_{\mathbf{A}}\in\left(-\epsilon,\epsilon\right)^{d\cdot\left(d+1\right)/2}$
for any DAG $\mathbf{A}$ and $\epsilon>0$. Therefore, the distance
between $\mathbf{z}$ and $\mathbf{z}_{\mathbf{A}}$ is bounded in
$l_{\infty}$ norm by the triangle inequality as $\left|\mathbf{z}-\mathbf{z}_{\mathbf{A}}\right|_{\infty}\leq\left|\mathbf{z}\right|_{\infty}+\left|\mathbf{z}_{\mathbf{A}}\right|<\left|\mathbf{z}\right|_{\infty}+\epsilon$.
\end{proof}

\subsection{Proof of Lemma~\ref{pg-guarantee}\label{subsec:Proof-of-pg-guarantee}}
\begin{lem*}
Assuming causal identifiability and causal minimality, that is, there
is a unique causal model with no redundant edges that can produce
the observed dataset, and BIC score is used to define the reward $\mathcal{R}\left(\mathbf{z}\right)$
as in Eq.~(\ref{eq:reward}). Let $n$ be the sample size of the
observed dataset $\mathcal{D}$, $\theta^{*}\in\underset{\mathbf{\theta}\in\Theta}{\arg\max}\mathbb{E}_{\mathbf{z}\sim\pi_{\bm{\theta}}}\left[\mathcal{R}\left(\mathbf{z}\right)\right]$,
where $\pi_{\bm{\theta}}\left(\mathbf{z}\right)=\mathcal{N}\left(\mathbf{z};\bm{\mu}_{\bm{\theta}},\text{diag}\left(\mathbf{\bm{\sigma}}_{\bm{\theta}}^{2}\right)\right)$.
Then, as $n\rightarrow\infty$, $\mathcal{G}=\vecdag\left(\mathbf{z}\right)$
is the true DAG, where $\mathbf{z}\sim\pi_{\theta^{*}}$.
\end{lem*}
\begin{proof}
Let $\mathcal{G}^{*}$ be the true DAG. By the identifiability and
minimality of the causal model and Lemma~\ref{lem:BIC-consistency},
as $n\rightarrow\infty$, $\mathbf{z}^{*}\in\arg\max\mathcal{R}\left(\mathbf{z}\right)$
if and only if $\vecdag\left(\mathbf{z}^{*}\right)=\mathcal{G}^{*}$.
This implies that the reward $\mathcal{R}\left(\mathbf{z}\right)$
is only maximized when $\mathbf{z}$ corresponds to $\mathcal{G}^{*}$.

Now, consider the objective $\theta^{*}\in\underset{\mathbf{\theta}\in\Theta}{\arg\max}\mathbb{E}_{\mathbf{z}\sim\pi_{\bm{\theta}}}\left[\mathcal{R}\left(\mathbf{z}\right)\right]$.
One solution to this maximization problem is when $\pi_{\bm{\theta}}$
places all of its probability mass on an arbitrary $\mathbf{z}^{*}\in\arg\max\mathcal{R}\left(\mathbf{z}\right)$,
i.e., $\pi_{\bm{\theta}^{*}}\left(\mathbf{z}\right)=\delta\left(\mathbf{z}-\mathbf{z}^{*}\right)$,
where $\delta\left(\cdot\right)$ is the Dirac delta function. This
is achieved when $\bm{\mu}_{\bm{\theta}}=\mathbf{z}^{*}$ and $\mathbf{\bm{\sigma}}^{2}=\mathbf{0}$.
In this case, any sample $\mathbf{z}\sim\pi_{\bm{\theta}^{*}}$ satisfies
$\mathbf{z}=\mathbf{z}^{*}$.

On the other hand, if $\pi_{\bm{\theta}}$ assigns probability mass
to multiple values of $\arg\max\mathcal{R}\left(\mathbf{z}\right)$,
then it is a regular Gaussian distribution and thus will also put
non-zero masses on the values with smaller rewards, thus strictly
reducing the value of the expected reward. Thus, $\pi_{\bm{\theta}}$
cannot achieve the maximum expected reward unless it is a Dirac delta
distribution concentrated at $\mathbf{z}^{*}$.

Finally, since $\vecdag\left(\mathbf{z}^{*}\right)=\mathcal{G}^{*}$,
we have $\mathcal{G}=\vecdag\left(\mathbf{z}\right)=\mathcal{G}^{*}$
with probability 1 when $\mathbf{z}\sim\pi_{\bm{\theta}^{*}}$. Therefore,
as $n\rightarrow\infty$, $\mathcal{G}=\vecdag\left(\mathbf{z}\right)$
is the true DAG.
\end{proof}

\section{Experiment Details\label{sec:Experiment-Details}}

\subsection{Datasets}

\subsubsection{Synthetic Linear-Gaussian Data\label{subsec:Synthetic-Linear-Gaussian-Data}}

To simulate data for a given number of nodes $d$, we first generate
a DAG following the Erd\H{o}s-R\'{e}nyi graph model \citep{Erdos_Renyi_60Evolution}
with a random ordering and an expected in-degree of $k\in\mathbb{N}^{+}$,
denoted by ER-$k$. Next, edge weights are randomly sampled from the
uniform distribution $\mathcal{U}\left(\left[-5,-2\right]\cup\left[2,5\right]\right)$,
giving a weighted matrix $\mathbf{W}\in\mathbb{R}^{d\times d}$ where
zero entries indicate no connections, then the noises are sampled
from the standard Gaussian distribution $E_{i}\sim\mathcal{N}\left(0,1\right)$.

Finally, we sample $n=1,000$ observations for each dataset following
the linear assignment $X_{i}:=\sum_{j\in\mathrm{pa}_{i}}W_{ji}X_{j}+E_{i}$.
While linear-Gaussian models are non-identifiable in general \citep{Spirtes_etal_00Causation},
the instances with equal noise variances adopted in this experiment
are known to be fully identifiable \citep{Peters_etal_2014Causal}.
This data generation process is similar to multiple other works such
as \citet{Zheng_etal_18DAGs,Zhu_etal_2020Causal,Wang_etal_2021Ordering}
and is conducted using the well-established gCastle\footnote{\url{https://github.com/huawei-noah/trustworthyAI/tree/master/gcastle},
version 1.0.3, Apache-2.0 license.} utility for causality research \citep{Zhang_etal_2021Gcastle}.

\subsubsection{Nonlinear Data with Gaussian Processes}

We evaluate the performance of the proposed method $\ours$ with competitors
on the exact 5 datasets used by \citet{Zhu_etal_2020Causal} in their
experiment, which are produced by \citet{Lachapelle_etal_2020Gradient}
(\url{https://github.com/kurowasan/GraN-DAG}, MIT license). In addition,
we also consider the first 5 datasets of the PNL-GP portion of their
datasets for the model misspecification experiment.

\subsection{Experiment Details}

\subsubsection{Evaluation Metrics\label{subsec:Metrics}}

The estimated graphs are assessed against ground truth DAGs on multiple
evaluation metrics, including the commonly employed Structural Hamming
Distance (SHD, lower is better), False Detection Rate (FDR, lower
is better), and True Positive Rate (TPR, higher is better). SHD counts
the minimal number of edge additions, removals, and reversals in order
to turn one graph into another. FDR is the ratio of incorrectly estimated
edges over all estimated edges, while TPR is the proportion of correctly
identified edges over all true edges. We also utilize gCastle for
the calculations of the aforementioned metrics.

\subsubsection{Implementations of Methods\label{subsec:Implementation-Details}}

Our set of baseline methods contains a wide range of both well-established
and recent approaches, including the constrained continuous optimization
approaches with soft DAG constraints NOTEARS \citep{Zheng_etal_18DAGs,Zheng_etal_20Learning}
and DAGMA \citep{Bello_etal_22dagma}, unconstrained continuous optimization
approaches NOCURL \citep{Yu_etal_2021Dags} and COSMO \citep{Massidda_etal_2023Constraint},
as well as three RL-based methods RL-BIC \citep{Zhu_etal_2020Causal},
CORL \citep{Wang_etal_2021Ordering}, and RCL-OG \citep{Yang_etal_2023Reinforcement},
where the implementations are publicly available. More specifically,
\begin{itemize}
\item NOTEARS \citep{Zheng_etal_18DAGs,Zheng_etal_20Learning} is a continuous
optimization method that optimizes over the continuous space of weighted
adjacency $\mathbf{W}\in\mathbb{R}^{d\times d}$ for linear models,
with a soft constraint $\mathbf{W}\text{ is DAG}\Leftrightarrow h\left(\mathbf{W}\right)=e^{\mathbf{W}\circ\mathbf{W}}-d=0$,
which is solved using augmented Lagrangian methods. For nonlinear
data, $\mathbf{W}$ is used to mask the input of the MLPs that model
the nonlinear causal mechanisms.
\item DAGMA \citep{Bello_etal_22dagma} is an alternative proposal to NOTEARS,
with the acyclicity constraint $\mathbf{W}\text{ is DAG}\Leftrightarrow h\left(\mathbf{W}\right)=\log\det\left(s\mathbf{I}-\mathbf{W}\circ\mathbf{W}\right)-d\log s=0$
for all $s>0$.
\item NOCURL \citep{Yu_etal_2021Dags} is proposed as an unconstrained continuous
optimization method that represents the weight matrix of a linear
model by $\mathbf{A}=\mathbf{W}\odot\mathrm{ReLU}\left(\mathrm{grad}\left(\mathbf{p}\right)\right),$
where $\mathbf{W}\in\mathbb{R}^{d\times d}$ and $\mathrm{grad}\left(\mathbf{p}\right)_{ij}:=p_{j}-p_{i}$.
Their optimization strategy is to first find a solution from an unconstrained
method like NOTEARS, then refine it with this new unconstrained representation.
\item COSMO \citep{Massidda_etal_2023Constraint} improves upon NOCURL by
parametrizing the adjacency matrix with $\mathbf{A}=\mathbf{W}\odot\mathrm{sigmoid}\left(\mathrm{grad}\left(\mathbf{p}\right)/\tau\right)$,
which can be shown to converge to a DAG when $\tau\rightarrow0^{+}$.
They optimize directly with this formulation instead of employing
the two-stage approach like NOCURL.
\item RL-BIC \citep{Zhu_etal_2020Causal} is the first RL-based method,
which involves an actor-critic agent that learns to output high-reward
graphs. The acyclicity is incorporated into the reward to penalize
cyclic graphs.
\item CORL \citep{Wang_etal_2021Ordering} is the first RL method for ordering-based
causal discovery, which revolves around an agent that learns to produce
causal orderings with high rewards. The policy sequentially generates
each element of the ordering, then the causal order is pruned to obtain
a DAG.
\item RCL-OG \citep{Yang_etal_2023Reinforcement} is an alternative RL approach
for ordering-based causal discovery, which improves the state space
design of CORL. It reduces the state space size from permutations
with a factorial size $\mathcal{O}\left(d!\right)$ to only $\mathcal{O}\left(2^{d}\right)$.
\end{itemize}
For NOTEARS, RL-BIC, and CORL, we also adopt the implementations from
the gCastle package. For RCL-OG, we employ the implementation provided
by the authors at \url{https://www.sdu-idea.cn/codes.php?name=RCL-OG}
(no license provided). Since RCL-OG is essentially a Bayesian method,
we use the best-scoring DAG from its $1\,000$ posterior samples as
the output for a fair comparison with other methods.

Default hyper-parameters for each implementation are used unless specifically
indicated. The detailed configurations of all methods are provided
in Appendix~\ref{subsec:Hyper-parameters}.

Our proposed $\ours$ method is implemented using the Stable-Baselines3\footnote{\url{https://github.com/DLR-RM/stable-baselines3}, MIT license.}
toolset \citep{Raffin_etal_2021Stable} with the Advantage Actor-Critic
(A2C, \citealp{Mnih_etal_2016Asynchronous}) and Proximal Policy Optimization
(PPO, \citealp{Schulman_etal_2017Proximal}) methods, and a custom
DAG environment built on top of Gymnasium\footnote{\url{https://github.com/Farama-Foundation/Gymnasium}, MIT license.}
\citep{Towers_etal_2023Gymnasium}. See Appendix~\ref{subsec:Hyper-parameters}
the hyper-parameters for our methods.

Experiments are executed on a mix of several machines running Ubuntu
20.04/22.04 with the matching Python environments, including the following
configurations:
\begin{itemize}
\item AMD EPYC 7742 CPU, 1TB of RAM, and 8 Nvidia A100 40GB GPUs.
\item Intel Xeon Platinum 8452Y CPU, 1TB of RAM, 4 Nvidia H100 80GB GPUs.
\item Intel Core i9 13900KF CPU, 128GB of RAM, 1 Nvidia 4070Ti Super 16GB
GPU.
\end{itemize}
The first configuration is used for batched executions of all methods,
while the second is for resource-intensive experiments like high-dimensional
ones with RL-based baselines, and the last configuration is for prototyping
experiments. We find that RL-based baselines heavily rely on GPU with
much slower (hours of) runtime on CPU, even on 30-node graphs, while
our method is less dependent on GPU and can even handle datasets of
up to 100 nodes in 45 minutes purely on CPU (single process).

\subsubsection{Hyper-parameters\label{subsec:Hyper-parameters}}

We provide the hyper-parameters for each method and experiment scenario
as follows:
\begin{itemize}
\item $\ours$ (Ours): Table~\ref{tab:Hyper-parameters}.
\item NOTEARS and DAGMA: Table~\ref{tab:params-NOTEARS-DAGMA}.
\item NOCURL: Table~\ref{tab:params-NOCURL}.
\item COSMO: Table~\ref{tab:params-COSMO}.
\item RL-BIC: Table~\ref{tab:params-RLBIC}.
\item CORL: Table~\ref{tab:params-CORL}.
\item RCL-OG: Table~\ref{tab:params-RCLOG}.
\end{itemize}
\begin{table}
\caption{Default hyper-parameters for our method $\protect\ours$ throughout
the experiments. Unmentioned hyper-parameters are left unchanged.\label{tab:Hyper-parameters}}

\centering{}\resizebox{\textwidth}{!}{%
\begin{tabular}{>{\centering}p{0.3\textwidth}ccc}
\toprule 
\multirow{2}{0.3\textwidth}{Experiment} & Linear data (Figures \ref{fig:learning-curve}, \ref{fig:linear-gaussian}
and \ref{fig:sample-size},  & Nonlinear GP data (Table~\ref{tab:nonlinear}  & Real data\tabularnewline
 & Tables~\ref{tab:dense-highdim}, \ref{tab:different-noises}, \ref{tab:ER-8},
\ref{tab:SF-8}, \ref{tab:learning-rate}, \ref{tab:entropy-regularization},
and \ref{tab:sparisty-regularization}) & and Table~\ref{tab:nonlinear-pnl}) & (Table~\ref{tab:real-data})\tabularnewline
\midrule
Batch size (No. parallel environments) & \multicolumn{3}{c}{64}\tabularnewline
\midrule
Training steps & \multicolumn{3}{c}{20,000}\tabularnewline
\midrule
No. steps to run for each update & \multicolumn{3}{c}{1}\tabularnewline
\midrule
Data standardization (dimension-wise) & No (variances should remain equal) & No (data std. is near unit) & Yes (data std. is large)\tabularnewline
\midrule
RL method & \multicolumn{3}{c}{PPO}\tabularnewline
\midrule
Advantage normalization & \multicolumn{3}{c}{Yes}\tabularnewline
\midrule
Learning rate & \multicolumn{3}{c}{Stable-Baseline3' default ($0.0003$ for PPO and $0.0007$ for A2C)}\tabularnewline
\midrule
Regression method & Linear Regression & \multicolumn{2}{c}{Gaussian Process Regression with RBF kernel}\tabularnewline
\midrule
Scoring method & BIC equal variances, assuming Gaussian noises & \multicolumn{2}{c}{BIC non-equal variances, assuming Gaussian noises}\tabularnewline
\midrule
$l_{0}$ regularization for LS score (not used in comparative experiments) & $0.000001$ & - & -\tabularnewline
\bottomrule
\end{tabular}}
\end{table}

\begin{table}
\caption{Default hyper-parameters for CORL \citep{Wang_etal_2021Ordering}
throughout the experiments. Unmentioned hyper-parameters are left
unchanged.\label{tab:params-CORL}}

\centering{}\resizebox{\textwidth}{!}{%
\begin{tabular}{cccc}
\toprule 
\multirow{2}{*}{Experiment} & Linear data (Figures \ref{fig:learning-curve}, \ref{fig:linear-gaussian}
and \ref{fig:sample-size}, & Nonlinear GP data (Table~\ref{tab:nonlinear} & Real data\tabularnewline
 & Tables~\ref{tab:dense-highdim}, \ref{tab:different-noises}, \ref{tab:ER-8},
\ref{tab:SF-8}, \ref{tab:learning-rate}, \ref{tab:entropy-regularization},
and \ref{tab:sparisty-regularization}) & and Table~\ref{tab:nonlinear-pnl}) & (Table~\ref{tab:real-data})\tabularnewline
\midrule
Batch size & \multicolumn{3}{c}{64}\tabularnewline
\midrule
Training steps & \multicolumn{3}{c}{10,000}\tabularnewline
\midrule
Data standardization (dimension-wise) & No (variances should remain equal) & No (data std. is near unit) & Yes (data std. is large)\tabularnewline
\midrule
Regression method & Linear Regression & \multicolumn{2}{c}{Gaussian Process Regression with RBF kernel}\tabularnewline
\midrule
Scoring method & BIC equal variances, assuming Gaussian noises & \multicolumn{2}{c}{BIC non-equal variances, assuming Gaussian noises}\tabularnewline
\bottomrule
\end{tabular}}
\end{table}

\begin{table}
\caption{Default hyper-parameters for RL-BIC \citep{Zhu_etal_2020Causal} throughout
the experiments. Unmentioned hyper-parameters are left unchanged.\label{tab:params-RLBIC}}

\centering{}\resizebox{\textwidth}{!}{%
\begin{tabular}{cccc}
\toprule 
\multirow{2}{*}{Experiment} & Linear data (Figures \ref{fig:learning-curve}, \ref{fig:linear-gaussian}
and \ref{fig:sample-size}, & Nonlinear GP data (Table~\ref{tab:nonlinear} & Real data\tabularnewline
 & Tables~\ref{tab:dense-highdim}, \ref{tab:different-noises}, \ref{tab:ER-8},
\ref{tab:SF-8}, \ref{tab:learning-rate}, \ref{tab:entropy-regularization},
and \ref{tab:sparisty-regularization}) & and Table~\ref{tab:nonlinear-pnl}) & (Table~\ref{tab:real-data})\tabularnewline
\midrule
Batch size & \multicolumn{3}{c}{64}\tabularnewline
\midrule
Training steps & \multicolumn{3}{c}{20,000}\tabularnewline
\midrule
Data standardization (dimension-wise) & No (variances should remain equal) & No (data std. is near unit) & Yes (data std. is large)\tabularnewline
\midrule
Regression method & Linear Regression & \multicolumn{2}{c}{Gaussian Process Regression with RBF kernel}\tabularnewline
\midrule
Scoring method & BIC equal variances, assuming Gaussian noises & \multicolumn{2}{c}{BIC non-equal variances, assuming Gaussian noises}\tabularnewline
\bottomrule
\end{tabular}}
\end{table}

\begin{table}
\caption{Default hyper-parameters for RCL-OG \citep{Yang_etal_2023Reinforcement}
throughout the experiments. Unmentioned hyper-parameters are left
unchanged.\label{tab:params-RCLOG}}

\centering{}\resizebox{\textwidth}{!}{%
\begin{tabular}{cccc}
\toprule 
\multirow{2}{*}{Experiment} & Linear data (Figures \ref{fig:learning-curve}, \ref{fig:linear-gaussian}
and \ref{fig:sample-size}, & Nonlinear GP data (Table~\ref{tab:nonlinear} & Real data\tabularnewline
 & Tables~\ref{tab:dense-highdim}, \ref{tab:different-noises}, \ref{tab:ER-8},
\ref{tab:SF-8}, \ref{tab:learning-rate}, \ref{tab:entropy-regularization},
and \ref{tab:sparisty-regularization}) & and Table~\ref{tab:nonlinear-pnl}) & (Table~\ref{tab:real-data})\tabularnewline
\midrule
Batch size & \multicolumn{3}{c}{32}\tabularnewline
\midrule
Training steps & \multicolumn{3}{c}{40,000}\tabularnewline
\midrule
Data standardization (dimension-wise) & No (variances should remain equal) & No (data std. is near unit) & Yes (data std. is large)\tabularnewline
\midrule
Regression method & Linear Regression & \multicolumn{2}{c}{Gaussian Process Regression with RBF kernel}\tabularnewline
\midrule
Scoring method & BIC equal variances, assuming Gaussian noises & \multicolumn{2}{c}{BIC non-equal variances, assuming Gaussian noises}\tabularnewline
\bottomrule
\end{tabular}}
\end{table}

\begin{table}
\caption{Default hyper-parameters for NOCURL \citep{Yang_etal_2023Reinforcement}
throughout the experiments. The values are the best parameters yielding
the lowest SHD over linear-Gaussian datasets with 30-node ER-4 graphs,
found via the tuning process provided by COSMO's implementation. Unmentioned
hyper-parameters are left unchanged.\label{tab:params-NOCURL}}

\centering{}\resizebox{\textwidth}{!}{%
\begin{tabular}{cccc}
\toprule 
\multirow{2}{*}{Experiment} & Linear data (Figures \ref{fig:learning-curve}, \ref{fig:linear-gaussian}
and \ref{fig:sample-size}, & Nonlinear GP data (Table~\ref{tab:nonlinear} & Real data\tabularnewline
 & Tables~\ref{tab:dense-highdim}, \ref{tab:different-noises}, \ref{tab:ER-8},
\ref{tab:SF-8}, \ref{tab:learning-rate}, \ref{tab:entropy-regularization},
and \ref{tab:sparisty-regularization}) & and Table~\ref{tab:nonlinear-pnl}) & (Table~\ref{tab:real-data})\tabularnewline
\midrule
Batch size & \multicolumn{3}{c}{64}\tabularnewline
\midrule
Inner iterations & \multicolumn{3}{c}{5000}\tabularnewline
\midrule
Data standardization (dimension-wise) & No (variances should remain equal) & No (data std. is near unit) & Yes (data std. is large)\tabularnewline
\midrule
Learning rate & \multicolumn{3}{c}{$0.0009747753554628831$}\tabularnewline
\midrule
Regularization strength & \multicolumn{3}{c}{$0.007478648909986116$}\tabularnewline
\bottomrule
\end{tabular}}
\end{table}

\begin{table}
\caption{Default hyper-parameters for COSMO \citep{Massidda_etal_2023Constraint}
throughout the experiments. The values are the best parameters yielding
the lowest SHD over linear-Gaussian datasets with 30-node ER-4 graphs
for linear data, and MLP datasets with 40-node ER-4 graphs for nonlinear
data. They are found via the tuning process provided by COSMO's implementation.
Unmentioned hyper-parameters are left unchanged.\label{tab:params-COSMO}}

\centering{}\resizebox{\textwidth}{!}{%
\begin{tabular}{cccc}
\toprule 
\multirow{2}{*}{Experiment} & Linear data (Figures \ref{fig:learning-curve}, \ref{fig:linear-gaussian}
and \ref{fig:sample-size}, & Nonlinear GP data (Table~\ref{tab:nonlinear} & Real data\tabularnewline
 & Tables~\ref{tab:dense-highdim}, \ref{tab:different-noises}, \ref{tab:ER-8},
\ref{tab:SF-8}, \ref{tab:learning-rate}, \ref{tab:entropy-regularization},
and \ref{tab:sparisty-regularization}) & and Table~\ref{tab:nonlinear-pnl}) & (Table~\ref{tab:real-data})\tabularnewline
\midrule
Batch size & \multicolumn{3}{c}{64}\tabularnewline
\midrule
Max epochs & 5000 & \multicolumn{2}{c}{2000}\tabularnewline
\midrule
Data standardization (dimension-wise) & No (variances should remain equal) & No (data std. is near unit) & Yes (data std. is large)\tabularnewline
\midrule
Learning rate & $0.004424703475697184$ & \multicolumn{2}{c}{$0.0011606444486776536$}\tabularnewline
\midrule
$l_{1}$ regularization strength & $0.0007589315865487066$ & \multicolumn{2}{c}{$0.000999704401738756$}\tabularnewline
\midrule
$l_{2}$ regularization strength & $0.029171977709975934$ & \multicolumn{2}{c}{$0.0011406751425196925$}\tabularnewline
\midrule
Priority regularization strength & $0.0007604972601205271$ & \multicolumn{2}{c}{$0.0014549388704106592$}\tabularnewline
\midrule
Temperature & $0.0008407426566089702$ & \multicolumn{2}{c}{$0.0007651953117707655$}\tabularnewline
\midrule
Shift & $0.005860546462049756$ & \multicolumn{2}{c}{$0.009324030532123762$}\tabularnewline
\midrule
MLP hidden units & $0$ & \multicolumn{2}{c}{$10$}\tabularnewline
\bottomrule
\end{tabular}}
\end{table}

\begin{table}
\caption{Default hyper-parameters for NOTEARS \citet{Zheng_etal_18DAGs,Zheng_etal_20Learning}
and DAGMA \citet{Bello_etal_22dagma} throughout the experiments.
Unmentioned hyper-parameters are left unchanged.\label{tab:params-NOTEARS-DAGMA}}

\centering{}\resizebox{\textwidth}{!}{%
\begin{tabular}{cccc}
\toprule 
\multirow{2}{*}{Experiment} & Linear data (Figures \ref{fig:learning-curve}, \ref{fig:linear-gaussian}
and \ref{fig:sample-size}, & Nonlinear GP data (Table~\ref{tab:nonlinear} & Real data\tabularnewline
 & Tables~\ref{tab:dense-highdim}, \ref{tab:different-noises}, \ref{tab:ER-8},
\ref{tab:SF-8}, \ref{tab:learning-rate}, \ref{tab:entropy-regularization},
and \ref{tab:sparisty-regularization}) & and Table~\ref{tab:nonlinear-pnl}) & (Table~\ref{tab:real-data})\tabularnewline
\midrule
Data standardization (dimension-wise) & No (variances should remain equal) & No (data std. is near unit) & Yes (data std. is large)\tabularnewline
\midrule
SEM & Linear & \multicolumn{2}{c}{MLP}\tabularnewline
\midrule
Loss type & \multicolumn{3}{c}{$l_{2}$}\tabularnewline
\bottomrule
\end{tabular}}
\end{table}

\section{Additional Results\label{sec:Additional-Experiments}}

\subsection{Results on ER graphs}

In Table~\ref{tab:ER-8}, we provide detailed numerical results for
Figure~\ref{fig:sample-size}, i.e., effect of dimensionality and
sample size on linear-Gaussian data and ER-8 graphs.

\begin{table}[H]
\caption{Causal discovery performance as function of dimensionalities and sample
size on ER-8 graphs. The numbers are \textit{mean \textpm{} standard
error} over 5 random datasets.\label{tab:ER-8}}

\centering{}\resizebox{.7\textwidth}{!}{%
\begin{tabular}{cccccc}
\toprule 
Nodes & Samples & Method & SHD ($\downarrow$) & FDR ($\downarrow$) & TPR ($\uparrow$)\tabularnewline
\midrule
\multirow{10}{*}{20} & \multirow{2}{*}{100} & DAGMA & $40.6\pm{3.5}$ & $0.09\pm{0.01}$ & $0.78\pm{0.02}$\tabularnewline
 &  & $\ours$ & \textbf{$24.8\pm{4.6}$} & \textbf{$0.04\pm{0.01}$} & \textbf{$0.86\pm{0.02}$}\tabularnewline
\cmidrule{2-6}
 & \multirow{2}{*}{500} & DAGMA & $36.8\pm{3.8}$ & $0.08\pm{0.01}$ & $0.8\pm{0.02}$\tabularnewline
 &  & $\ours$ & \textbf{$3.0\pm{1.1}$} & \textbf{$0.01\pm{0.00}$} & \textbf{$0.98\pm{0.00}$}\tabularnewline
\cmidrule{2-6}
 & \multirow{2}{*}{$1000$} & DAGMA & $35.6\pm{3.6}$ & $0.07\pm{0.01}$ & $0.8\pm{0.02}$\tabularnewline
 &  & $\ours$ & \textbf{$1.0\pm{0.6}$} & \textbf{$0.0\pm{0.00}$} & \textbf{$0.99\pm{0.00}$}\tabularnewline
\cmidrule{2-6}
 & \multirow{2}{*}{2000} & DAGMA & $31.8\pm{3.6}$ & $0.07\pm{0.01}$ & $0.82\pm{0.03}$\tabularnewline
 &  & $\ours$ & \textbf{$0.2\pm{0.2}$} & \textbf{$0.0\pm0.0$} & \textbf{$1.0\pm0.0$}\tabularnewline
\cmidrule{2-6}
 & \multirow{2}{*}{5000} & DAGMA & $31.2\pm{4.2}$ & $0.07\pm{0.01}$ & $0.83\pm{0.03}$\tabularnewline
 &  & $\ours$ & \textbf{$0.2\pm{0.2}$} & \textbf{$0.0\pm0.0$} & \textbf{$1.0\pm0.0$}\tabularnewline
\midrule
\multirow{10}{*}{30} & \multirow{2}{*}{100} & DAGMA & $87.0\pm{10.6}$ & $0.19\pm{0.03}$ & $0.78\pm{0.02}$\tabularnewline
 &  & $\ours$ & \textbf{$51.8\pm{7.9}$} & \textbf{$0.1\pm{0.02}$} & \textbf{$0.86\pm{0.02}$}\tabularnewline
\cmidrule{2-6}
 & \multirow{2}{*}{500} & DAGMA & $67.0\pm{9.5}$ & $0.13\pm{0.02}$ & $0.81\pm{0.02}$\tabularnewline
 &  & $\ours$ & \textbf{$3.8\pm{0.6}$} & \textbf{$0.01\pm0.0$} & \textbf{$0.99\pm0.0$}\tabularnewline
\cmidrule{2-6}
 & \multirow{2}{*}{1000} & DAGMA & $67.6\pm{8.0}$ & $0.14\pm{0.02}$ & $0.82\pm{0.02}$\tabularnewline
 &  & $\ours$ & \textbf{$0.2\pm{0.2}$} & \textbf{$0.0\pm0.0$} & \textbf{$1.0\pm0.0$}\tabularnewline
\cmidrule{2-6}
 & \multirow{2}{*}{2000} & DAGMA & $65.8\pm{7.9}$ & $0.13\pm{0.02}$ & $0.82\pm{0.02}$\tabularnewline
 &  & $\ours$ & \textbf{$0.0\pm0.0$} & \textbf{$0.0\pm0.0$} & \textbf{$1.0\pm0.0$}\tabularnewline
\cmidrule{2-6}
 & \multirow{2}{*}{5000} & DAGMA & $68.4\pm{10.2}$ & $0.14\pm{0.02}$ & $0.81\pm{0.03}$\tabularnewline
 &  & $\ours$ & \textbf{$0.0\pm0.0$} & \textbf{$0.0\pm0.0$} & \textbf{$1.0\pm0.0$}\tabularnewline
\midrule
\multirow{10}{*}{50} & \multirow{2}{*}{100} & DAGMA & $242.0\pm{15.6}$ & $0.34\pm{0.02}$ & $0.76\pm{0.01}$\tabularnewline
 &  & $\ours$ & \textbf{$94.2\pm{10.1}$} & \textbf{$0.14\pm{0.02}$} & \textbf{$0.9\pm{0.01}$}\tabularnewline
\cmidrule{2-6}
 & \multirow{2}{*}{500} & DAGMA & $213.2\pm{20.8}$ & $0.31\pm{0.03}$ & $0.78\pm{0.01}$\tabularnewline
 &  & $\ours$ & \textbf{$13.6\pm{5.5}$} & \textbf{$0.02\pm{0.01}$} & \textbf{$0.99\pm{0.00}$}\tabularnewline
\cmidrule{2-6}
 & \multirow{2}{*}{1000} & DAGMA & $215.2\pm{26.6}$ & $0.3\pm{0.04}$ & $0.77\pm{0.02}$\tabularnewline
 &  & $\ours$ & \textbf{$2.2\pm{1.0}$} & \textbf{$0.0\pm0.0$} & \textbf{$1.0\pm0.0$}\tabularnewline
\cmidrule{2-6}
 & \multirow{2}{*}{2000} & DAGMA & $213.2\pm{25.8}$ & $0.31\pm{0.03}$ & $0.78\pm{0.02}$\tabularnewline
 &  & $\ours$ & \textbf{$1.8\pm{0.8}$} & \textbf{$0.0\pm0.0$} & \textbf{$1.0\pm0.0$}\tabularnewline
\cmidrule{2-6}
 & \multirow{2}{*}{5000} & DAGMA & $210.2\pm{32.4}$ & $0.3\pm{0.04}$ & $0.78\pm{0.02}$\tabularnewline
 &  & $\ours$ & \textbf{$0.7\pm{0.5}$} & \textbf{$0.0\pm0.0$} & \textbf{$1.0\pm0.0$}\tabularnewline
\bottomrule
\end{tabular}}
\end{table}

\subsection{Results on SF graphs}

Table~\ref{tab:SF-8} investigates the effect of dimensionality and
sample size on linear-Gaussian data and Scale-Free (SF) graphs with
8 parents per node.

\begin{table}[H]
\caption{Causal discovery performance as function of dimensionalities and sample
size on SF-8 graphs. The numbers are \textit{mean \textpm{} standard
error} over 5 random datasets.\label{tab:SF-8}}

\centering{}\resizebox{.7\textwidth}{!}{%
\begin{tabular}{cccccc}
\toprule 
Nodes & Samples & Method & SHD ($\downarrow$) & FDR ($\downarrow$) & TPR ($\uparrow$)\tabularnewline
\midrule
\multirow{10}{*}{20} & \multirow{2}{*}{100} & DAGMA & \textbf{$14.6\pm{3.5}$} & \textbf{$0.07\pm{0.03}$} & $0.87\pm{0.02}$\tabularnewline
 &  & $\ours$ & $16.0\pm{2.8}$ & $0.09\pm{0.02}$ & \textbf{$0.89\pm{0.03}$}\tabularnewline
\cmidrule{2-6}
 & \multirow{2}{*}{500} & DAGMA & $6.8\pm{1.4}$ & $0.02\pm{0.01}$ & $0.93\pm{0.01}$\tabularnewline
 &  & $\ours$ & \textbf{$0.6\pm{0.3}$} & \textbf{$0.01\pm{0.00}$} & \textbf{$1.0\pm{0.00}$}\tabularnewline
\cmidrule{2-6}
 & \multirow{2}{*}{$1000$} & DAGMA & $7.0\pm{1.7}$ & $0.02\pm{0.01}$ & $0.93\pm{0.02}$\tabularnewline
 &  & $\ours$ & \textbf{$0.2\pm{0.2}$} & \textbf{$0.0\pm{0.00}$} & \textbf{$1.0\pm{0.00}$}\tabularnewline
\cmidrule{2-6}
 & \multirow{2}{*}{2000} & DAGMA & $9.0\pm{3.0}$ & $0.04\pm{0.02}$ & $0.91\pm{0.03}$\tabularnewline
 &  & $\ours$ & \textbf{$0.2\pm{0.2}$} & \textbf{$0.0\pm{0.00}$} & \textbf{$1.0\pm{0.00}$}\tabularnewline
\cmidrule{2-6}
 & \multirow{2}{*}{5000} & DAGMA & $6.6\pm{2.5}$ & $0.03\pm{0.02}$ & $0.94\pm{0.02}$\tabularnewline
 &  & $\ours$ & \textbf{$0.2\pm{0.2}$} & \textbf{$0.0\pm{0.00}$} & \textbf{$1.0\pm{0.00}$}\tabularnewline
\midrule
\multirow{10}{*}{30} & \multirow{2}{*}{100} & DAGMA & $47.2\pm{7.1}$ & $0.15\pm{0.04}$ & $0.84\pm{0.01}$\tabularnewline
 &  & $\ours$ & \textbf{$30.2\pm{4.4}$} & \textbf{$0.1\pm{0.01}$} & \textbf{$0.92\pm{0.01}$}\tabularnewline
\cmidrule{2-6}
 & \multirow{2}{*}{500} & DAGMA & $45.8\pm{7.9}$ & $0.16\pm{0.03}$ & $0.85\pm{0.02}$\tabularnewline
 &  & $\ours$ & \textbf{$0.2\pm{0.2}$} & \textbf{$0.0\pm0.0$} & \textbf{$1.0\pm0.0$}\tabularnewline
\cmidrule{2-6}
 & \multirow{2}{*}{1000} & DAGMA & $43.6\pm{10.0}$ & $0.14\pm{0.04}$ & $0.86\pm{0.03}$\tabularnewline
 &  & $\ours$ & \textbf{$0.2\pm{0.2}$} & \textbf{$0.0\pm0.0$} & \textbf{$1.0\pm0.0$}\tabularnewline
\cmidrule{2-6}
 & \multirow{2}{*}{2000} & DAGMA & $36.6\pm{9.7}$ & $0.12\pm{0.04}$ & $0.89\pm{0.02}$\tabularnewline
 &  & $\ours$ & \textbf{$0.0\pm0.0$} & \textbf{$0.0\pm0.0$} & \textbf{$1.0\pm0.0$}\tabularnewline
\cmidrule{2-6}
 & \multirow{2}{*}{5000} & DAGMA & $27.4\pm{4.7}$ & $0.09\pm{0.03}$ & $0.9\pm{0.01}$\tabularnewline
 &  & $\ours$ & \textbf{$0.0\pm0.0$} & \textbf{$0.0\pm0.0$} & \textbf{$1.0\pm0.0$}\tabularnewline
\midrule
\multirow{10}{*}{50} & \multirow{2}{*}{100} & DAGMA & $101.2\pm{18.3}$ & $0.18\pm{0.04}$ & $0.84\pm{0.02}$\tabularnewline
 &  & $\ours$ & \textbf{$54.2\pm{7.9}$} & \textbf{$0.11\pm{0.01}$} & \textbf{$0.92\pm{0.01}$}\tabularnewline
\cmidrule{2-6}
 & \multirow{2}{*}{500} & DAGMA & $68.4\pm{10.6}$ & $0.13\pm{0.02}$ & $0.89\pm{0.01}$\tabularnewline
 &  & $\ours$ & \textbf{$9.8\pm{5.1}$} & \textbf{$0.03\pm{0.01}$} & \textbf{$0.99\pm{0.00}$}\tabularnewline
\cmidrule{2-6}
 & \multirow{2}{*}{1000} & DAGMA & $71.8\pm{12.6}$ & $0.13\pm{0.03}$ & $0.89\pm{0.01}$\tabularnewline
 &  & $\ours$ & \textbf{$5.2\pm{3.3}$} & \textbf{$0.01\pm{0.01}$} & \textbf{$1.0\pm0.0$}\tabularnewline
\cmidrule{2-6}
 & \multirow{2}{*}{2000} & DAGMA & $65.2\pm{9.8}$ & $0.12\pm{0.02}$ & $0.9\pm{0.02}$\tabularnewline
 &  & $\ours$ & \textbf{$5.4\pm{4.4}$} & \textbf{$0.02\pm{0.01}$} & \textbf{$1.0\pm0.0$}\tabularnewline
\cmidrule{2-6}
 & \multirow{2}{*}{5000} & DAGMA & $75.0\pm{25.1}$ & $0.14\pm{0.05}$ & $0.89\pm{0.04}$\tabularnewline
 &  & $\ours$ & \textbf{$8.8\pm{5.2}$} & \textbf{$0.02\pm{0.01}$} & \textbf{$0.99\pm{0.00}$}\tabularnewline
\bottomrule
\end{tabular}}
\end{table}

\subsection{Ablation Studies\label{subsec:Ablation-Studies}}

\subsubsection{Effect of RL method and learning rate}

In Table~\ref{tab:learning-rate}, we conduct ablation studies on
the effect of the choices of RL method and learning rate on the performance
of our $\ours$ method with the BIC score.

\begin{table}[H]
\begin{centering}
\caption{Performance sensitivity of $\protect\ours$ with BIC score subjected
to the variations of learning rate. We employ linear-Gaussian datasets
with 30 nodes on ER-8 graphs and use CORL \citep{Wang_etal_2021Ordering}
and DAGMA \citep{Bello_etal_22dagma} as reference. For each row,
we compute the means and standard errors over 5 independent datasets.
\textbf{Bold}: better performance than the baselines. Unless otherwise
indicated, the remaining hyper-parameters are used according to Table~\ref{tab:Hyper-parameters}
in the Appendix.\label{tab:learning-rate}}
\par\end{centering}
\centering{}\resizebox{\textwidth}{!}{%
\begin{tabular}{cclccc}
\toprule 
Method & RL method & Learning rate & SHD ($\downarrow$) & FDR ($\downarrow$) & TPR ($\uparrow$)\tabularnewline
\midrule 
CORL & - & - & $\phantom{0}82.4\pm\phantom{0}{22.3}$ & $0.23\pm{0.05}$ & $0.87\pm{0.04}$\tabularnewline
DAGMA & - & - & $\phantom{0}67.6\pm\phantom{0}{8.0}$ & $0.14\pm{0.02}$ & $0.82\pm{0.02}$\tabularnewline
\midrule
\multirow{16}{*}{$\ours$ (Ours)} & \multirow{8}{*}{A2C} & 0.00001 & $\phantom{0}79.4\pm\phantom{0}{6.7}$ & $0.22\pm{0.02}$ & $0.86\pm{0.01}$\tabularnewline
 &  & 0.00005 & $\mathbf{\phantom{0}12.6\pm\phantom{00}{3.1}}$ & $\mathbf{0.04\pm{0.01}}$ & $\mathbf{0.98\pm{0.00}}$\tabularnewline
 &  & 0.0001 & $\mathbf{\phantom{00}1.2\pm\phantom{00}{0.4}}$ & \textbf{$\mathbf{0.00\pm{0.00}}$} & $\mathbf{1.00\pm0.00}$\tabularnewline
 &  & 0.0005 & $\mathbf{\phantom{00}2.2\pm\phantom{00}{1.3}}$ & $\mathbf{0.01\pm{0.00}}$ & $\mathbf{1.00\pm0.00}$\tabularnewline
 &  & 0.001 & $\mathbf{\phantom{0}22.8\pm\phantom{00}{3.9}}$ & $\mathbf{0.07\pm{0.01}}$ & $\mathbf{0.96\pm{0.01}}$\tabularnewline
 &  & 0.005 & $123.8\pm{49.7}$ & $0.34\pm{0.16}$ & $0.54\pm{0.21}$\tabularnewline
 &  & 0.01 & $223.4\pm\phantom{0}{5.6}$ & $0.48\pm{0.03}$ & $0.13\pm{0.04}$\tabularnewline
 &  & 0.05 & $232.6\pm\phantom{00}{3.0}$ & $0.59\pm{0.03}$ & $0.11\pm{0.01}$\tabularnewline
\cmidrule{2-6}
 & \multirow{8}{*}{PPO} & 0.00001 & $\phantom{0}77.8\pm\phantom{0}{7.5}$ & $0.21\pm{0.02}$ & $0.86\pm{0.01}$\tabularnewline
 &  & 0.00005 & $\mathbf{\phantom{00}9.0\pm\phantom{00}{2.9}}$ & $\mathbf{0.03\pm{0.01}}$ & $\mathbf{0.99\pm{0.00}}$\tabularnewline
 &  & 0.0001 & $\mathbf{\phantom{00}1.2\pm\phantom{00}{0.4}}$ & $\mathbf{0.00\pm0.00}$ & $\mathbf{1.00\pm0.00}$\tabularnewline
 &  & 0.0005 & $\mathbf{\phantom{00}0.4\pm\phantom{00}{0.3}}$ & $\mathbf{0.00\pm0.00}$ & $\mathbf{1.00\pm0.00}$\tabularnewline
 &  & 0.001 & $\mathbf{\phantom{00}4.0\pm\phantom{00}{2.0}}$ & $\mathbf{0.01\pm{0.01}}$ & $\mathbf{0.99\pm{0.00}}$\tabularnewline
 &  & 0.005 & $204.8\pm\phantom{0}{38.8}$ & $0.65\pm{0.13}$ & $0.21\pm{0.18}$\tabularnewline
 &  & 0.01 & $225.4\pm\phantom{0}{6.8}$ & $0.53\pm{0.04}$ & $0.11\pm{0.04}$\tabularnewline
 &  & 0.05 & $231.4\pm\phantom{0}{5.1}$ & $0.54\pm{0.02}$ & $0.13\pm{0.01}$\tabularnewline
\bottomrule
\end{tabular}}
\end{table}

\subsubsection{Effect of Entropy regularization}

In Table~\ref{tab:entropy-regularization}, we conduct ablation studies
on the effect of the choices of entropy regularization on the performance
of our $\ours$ method with the BIC score.

\begin{table}[H]
\caption{Performance sensitivity of $\protect\ours$ with BIC score subjected
to the variations of Entropy regularization weight. We employ linear-Gaussian
datasets with 30 nodes on ER-8 graphs and use CORL \citep{Wang_etal_2021Ordering}
and DAGMA \citep{Bello_etal_22dagma} as reference. Our scoring function
is BIC. For each row, we compute the means and standard errors over
5 independent datasets. \textbf{Bold}: better performance than the
baselines. Unless otherwise indicated, the remaining hyper-parameters
are used according to Table~\ref{tab:Hyper-parameters}.\label{tab:entropy-regularization}}

\centering{}\resizebox{\textwidth}{!}{%
\begin{tabular}{ccllccc}
\toprule 
Method & RL method & Learning rate & Entropy Coef. & SHD ($\downarrow$) & FDR ($\downarrow$) & TPR ($\uparrow$)\tabularnewline
\midrule
CORL &  & - & - & $\phantom{0}82.4\pm\phantom{0}{22.3}$ & $0.23\pm{0.05}$ & $0.87\pm{0.04}$\tabularnewline
DAGMA &  & - & - & $\phantom{0}67.6\pm\phantom{0}{8.0}$ & $0.14\pm{0.02}$ & $0.82\pm{0.02}$\tabularnewline
\midrule
\multirow{20}{*}{$\ours$ (Ours)} & \multirow{10}{*}{A2C} & \multirow{5}{*}{0.0001} & 0 & $\mathbf{\phantom{00}1.20\pm{0.4}}$ & $\mathbf{0.00\pm0.00}$ & \textbf{$\mathbf{1.00\pm0.00}$}\tabularnewline
 &  &  & 0.001 & $\mathbf{\phantom{00}1.2\pm{0.4}}$ & $\mathbf{0.00\pm0.00}$ & $\mathbf{1.00\pm0.00}$\tabularnewline
 &  &  & 0.01 & $\mathbf{\phantom{00}1.6\pm{0.5}}$ & $\mathbf{0.01\pm{0.00}}$ & $\mathbf{1.00\pm0.00}$\tabularnewline
 &  &  & 0.1 & $\mathbf{\phantom{0}28.8\pm{7.7}}$ & $\mathbf{0.09\pm{0.03}}$ & $\mathbf{0.97\pm{0.01}}$\tabularnewline
 &  &  & 1 & $127.0\pm{13.1}$ & $0.31\pm{0.02}$ & $0.78\pm{0.03}$\tabularnewline
\cmidrule{3-7}
 &  & \multirow{5}{*}{0.0005} & 0 & $\mathbf{\phantom{00}2.2\pm{1.3}}$ & $\mathbf{0.01\pm{0.00}}$ & $\mathbf{1.00\pm0.00}$\tabularnewline
 &  &  & 0.001 & $\mathbf{\phantom{00}3.4\pm{1.9}}$ & $\mathbf{0.01\pm{0.00}}$ & $\mathbf{0.99\pm{0.00}}$\tabularnewline
 &  &  & 0.01 & $\mathbf{\phantom{00}0.2\pm{0.2}}$ & $\mathbf{0.00\pm0.00}$ & $\mathbf{1.00\pm0.00}$\tabularnewline
 &  &  & 0.1 & $\mathbf{\phantom{0}48.2\pm{11.7}}$ & $0.14\pm{0.04}$ & $\mathbf{0.93\pm{0.01}}$\tabularnewline
 &  &  & 1 & $189.2\pm{8.5}$ & $0.40\pm{0.01}$ & $0.51\pm{0.02}$\tabularnewline
\cmidrule{2-7}
 & \multirow{10}{*}{PPO} & \multirow{5}{*}{0.0001} & 0 & $\mathbf{\phantom{00}1.2\pm{0.4}}$ & $\mathbf{0.00\pm0.00}$ & $\mathbf{1.00\pm0.00}$\tabularnewline
 &  &  & 0.001 & $\mathbf{\phantom{00}1.0\pm{0.3}}$ & $\mathbf{0.00\pm0.00}$ & $\mathbf{1.00\pm0.00}$\tabularnewline
 &  &  & 0.01 & $\mathbf{\phantom{00}1.6\pm{0.5}}$ & $\mathbf{0.01\pm{0.00}}$ & $\mathbf{1.00\pm0.00}$\tabularnewline
 &  &  & 0.1 & $\mathbf{\phantom{0}28.6\pm{6.8}}$ & $\mathbf{0.09\pm{0.02}}$ & $\mathbf{0.97\pm{0.00}}$\tabularnewline
 &  &  & 1 & $130.2\pm{12.1}$ & $0.32\pm{0.02}$ & $0.77\pm{0.03}$\tabularnewline
\cmidrule{3-7}
 &  & \multirow{5}{*}{0.0005} & 0 & $\mathbf{\phantom{00}0.4\pm{0.3}}$ & $\mathbf{0.00\pm0.00}$ & $\mathbf{1.00\pm0.00}$\tabularnewline
 &  &  & 0.001 & $\mathbf{\phantom{00}0.2\pm{0.2}}$ & $\mathbf{0.00\pm0.00}$ & $\mathbf{1.00\pm0.00}$\tabularnewline
 &  &  & 0.01 & $\mathbf{\phantom{00}1.6\pm{1.3}}$ & $\mathbf{0.00\pm{0.00}}$ & $\mathbf{1.00\pm{0.00}}$\tabularnewline
 &  &  & 0.1 & $\mathbf{\phantom{0}51.8\pm{11.8}}$ & $0.15\pm{0.04}$ & $\mathbf{0.92\pm{0.01}}$\tabularnewline
 &  &  & 1 & $178.4\pm{9.5}$ & $0.38\pm{0.02}$ & $0.58\pm{0.02}$\tabularnewline
\bottomrule
\end{tabular}}
\end{table}

\subsubsection{Effect of sparsity regularization}

In Table~\ref{tab:sparisty-regularization}, we conduct ablation
studies on the effect of the choices of sparsity regularization strength
of our $\ours$ method with the LS score.

\begin{table}[H]
\begin{centering}
\caption{Performance sensitivity of $\protect\ours$ with LS score subjected
to the variations of Sparsity regularization weight for the LS score.
We employ linear-Gaussian datasets with 30 nodes on ER-8 graphs and
use CORL \citealp{Wang_etal_2021Ordering} and DAGMA \citealp{Bello_etal_22dagma}
as reference. For each row, we compute the means and standard errors
over 5 independent datasets. \textbf{Bold}: better performance than
the baselines. Unless otherwise indicated, the remaining hyper-parameters
are used according to Table~\ref{tab:Hyper-parameters}.\label{tab:sparisty-regularization}}
\par\end{centering}
\centering{}\resizebox{\textwidth}{!}{%
\begin{tabular}{cclccc}
\toprule 
Method & RL method & Sparsity regularizer $\lambda_{0}$ & SHD ($\downarrow$) & FDR ($\downarrow$) & TPR ($\uparrow$)\tabularnewline
\midrule 
CORL & - & - & $\phantom{0}82.4\pm\phantom{0}{22.3}$ & $0.23\pm{0.05}$ & $0.87\pm{0.04}$\tabularnewline
DAGMA & - & - & $\phantom{0}67.6\pm\phantom{0}{8.0}$ & $0.14\pm{0.02}$ & $0.82\pm{0.02}$\tabularnewline
\midrule
\multirow{10}{*}{$\ours$ (Ours)} & \multirow{5}{*}{A2C} & 0 & $\mathbf{\phantom{00}8.4\pm\phantom{0}{4.2}}$ & $\mathbf{0.02\pm{0.01}}$ & $\mathbf{0.98\pm{0.01}}$\tabularnewline
 &  & 0.000001 & $\mathbf{\phantom{00}2.8\pm\phantom{0}{1.0}}$ & $\mathbf{0.01\pm{0.00}}$ & $\mathbf{0.99\pm{0.00}}$\tabularnewline
 &  & 0.0001 & $\mathbf{\phantom{0}50.0\pm{27.7}}$ & $\mathbf{0.07\pm{0.04}}$ & $0.83\pm{0.10}$\tabularnewline
 &  & 0.01 & $189.0\pm{14.2}$ & $0.34\pm{0.04}$ & $0.28\pm{0.07}$\tabularnewline
 &  & 1 & $226.4\pm{4.7}$ & $0.50\pm{0.03}$ & $0.08\pm{0.02}$\tabularnewline
\cmidrule{2-6}
 & \multirow{5}{*}{PPO} & 0 & $\mathbf{\phantom{00}2.0\pm\phantom{0}{0.5}}$ & $\mathbf{0.01\pm{0.00}}$ & $\mathbf{1.00\pm0.00}$\tabularnewline
 &  & 0.000001 & $\mathbf{\phantom{00}1.2\pm\phantom{0}{0.6}}$ & $\mathbf{0.00\pm{0.00}}$ & $\mathbf{1.00\pm0.00}$\tabularnewline
 &  & 0.0001 & $\mathbf{\phantom{0}42.0\pm{26.4}}$ & $\mathbf{0.04\pm{0.02}}$ & $0.84\pm{0.10}$\tabularnewline
 &  & 0.01 & $178.8\pm{19.4}$ & $0.29\pm{0.05}$ & $0.31\pm{0.08}$\tabularnewline
 &  & 1 & $221.6\pm{6.3}$ & $0.42\pm{0.04}$ & $0.09\pm{0.02}$\tabularnewline
\bottomrule
\end{tabular}}
\end{table}

\subsection{Model Misspecification Results\label{subsec:Model-Misspecification-Results}}

In Table~\ref{tab:nonlinear-pnl}, we study causal model misspecification
on nonlinear data. We assume data is generated via additive noise
models with Gaussian process similarly to \citet{Zhu_etal_2020Causal,Wang_etal_2021Ordering,Yang_etal_2023Reinforcement},
but test with datasets generated by the Post-nonlinear Gaussian Process
model $X_{i}:=\sigma\left(f_{i}\left(\mathbf{X}_{\pa_{i}}\right)+\text{Laplace}\left(0,1\right)\right)$,
which is identifiable \citep{Zhang_etal_2009Identifiability} and
produced by \citet{Lachapelle_etal_2020Gradient}.

\begin{table}[H]
\caption{Causal discovery performance on nonlinear data with PNL-GP model.
The data is generated with 10-node ER-4 graphs and post nonlinear
Gaussian processes as causal mechanisms. The performance metrics are
Structural Hamming Distance (SHD), False Detection Rate (FDR), and
True Positive Rate (TPR). Lower SHD and FDR values are preferable,
while higher values are better for TPR. The numbers are \textit{mean
\textpm{} standard errors} over 5 independent runs. Since the graphs
are dense, we also study the effect of pruning the output graphs.\label{tab:nonlinear-pnl}}

\begin{centering}
\par\end{centering}
\centering{}\resizebox{\columnwidth}{!}{%
\begin{tabular}{ccccccc}
\toprule 
 & \multicolumn{3}{c}{No Pruning} & \multicolumn{3}{c}{CAM Pruning}\tabularnewline
\midrule 
Method & SHD ($\downarrow$) & FDR ($\downarrow$) & TPR ($\uparrow$) & SHD ($\downarrow$) & FDR ($\downarrow$) & TPR ($\uparrow$)\tabularnewline
\midrule
NOTEARS \citep{Zheng_etal_20Learning} & $29.4\pm{1.1}$ & $0.47\pm{0.02}$ & $0.34\pm{0.03}$ & $29.4\pm{1.1}$ & $0.46\pm{0.02}$ & $0.34\pm{0.03}$\tabularnewline
DAGMA \citep{Bello_etal_22dagma} & $27.6\pm{3.0}$ & $0.45\pm{0.06}$ & $0.38\pm{0.06}$ & $27.4\pm{3.1}$ & $0.51\pm{0.08}$ & $0.38\pm{0.06}$\tabularnewline
NOCURL \citep{Yu_etal_2021Dags} & $34.8\pm{1.1}$ & $0.51\pm{0.08}$ & $0.15\pm{0.03}$ & $34.8\pm{1.1}$ & $0.45\pm{0.07}$ & $0.15\pm{0.03}$\tabularnewline
COSMO \citep{Massidda_etal_2023Constraint} & $29.4\pm{0.8}$ & $0.41\pm{0.04}$ & $0.33\pm{0.02}$ & $29.4\pm{0.8}$ & $0.41\pm{0.04}$ & $0.33\pm{0.02}$\tabularnewline
CORL \citep{Wang_etal_2021Ordering} & \textbf{$\mathbf{\phantom{0}6.6\pm{0.9}}$} & $0.15\pm{0.02}$ & \textbf{$\mathbf{0.95\pm{0.01}}$} & \textbf{$\mathbf{\phantom{0}4.2\pm{0.7}}$} & $0.08\pm{0.02}$ & \textbf{$\mathbf{0.94\pm{0.02}}$}\tabularnewline
RCL-OG \citep{Yang_etal_2023Reinforcement} & $\phantom{0}8.6\pm{1.4}$ & $0.19\pm{0.03}$ & \uline{\mbox{$0.90\pm{0.02}$}} & $\phantom{0}7.2\pm{1.3}$ & $0.13\pm{0.03}$ & \uline{\mbox{$0.89\pm{0.02}$}}\tabularnewline
\midrule 
$\ours$ (Ours) & \uline{\mbox{$\phantom{0}6.8\pm{1.3}$}} & \textbf{$\mathbf{0.04\pm{0.03}}$} & $0.85\pm{0.03}$ & \uline{\mbox{$\phantom{0}6.8\pm{1.3}$}} & \textbf{$\mathbf{0.03\pm{0.03}}$} & $0.84\pm{0.03}$\tabularnewline
\bottomrule
\end{tabular}}
\end{table}

In addition, we also study the robustness of our method under noisy
data. Specifically, we corrupt the data by adding Gaussian noises
($\sigma^{2}=0.1$) to $p\%$ random entries of the design matrix
$\mathcal{D}\in\mathbb{R}^{n\times d}$, and report the results for
the challenging 30-node ER-8 graphs in Table~\ref{tab:noisy-data},
showing that our method still consistently surpasses the baselines,
even with increasing noise levels.

\begin{table}[H]
\caption{Causal Discovery Performance under Noisy Data. The numbers are \textit{mean
\textpm{} standard error} over 5 random datasets on 30-node ER-8 graphs\label{tab:noisy-data}}

\begin{centering}
\resizebox{0.7\columnwidth}{!}{%
\begin{tabular}{ccccc}
\toprule 
Method \textbackslash{} $p$ & 1\% & 3\% & 5\% & 10\%\tabularnewline
\midrule 
NOTEARS & $181.4\pm{4.0}$ & $189.0\pm{2.5}$ & $190.8\pm{2.8}$ & $194.4\pm{3.9}$\tabularnewline
DAGMA & $67.0\pm{8.5}$ & $75.0\pm{9.7}$ & $79.0\pm{6.5}$ & $94.0\pm{6.1}$\tabularnewline
NOCURL & $142.2\pm{4.2}$ & $147.0\pm{4.8}$ & $146.0\pm{5.1}$ & $153.0\pm{5.3}$\tabularnewline
COSMO & $96.6\pm{6.1}$ & $112.8\pm{7.6}$ & $111.0\pm{8.5}$ & $136.2\pm{14.3}$\tabularnewline
\midrule 
$\ours$ (Ours) & $\mathbf{8.4\pm{0.8}}$ & $\mathbf{27.2\pm{3.0}}$ & $\mathbf{41.4\pm{5.0}}$ & $\mathbf{66.0\pm{4.0}}$\tabularnewline
\bottomrule
\end{tabular}}
\par\end{centering}
\end{table}

Moreover, we investigate our method's performance under the dependence
of noises, which is equivalent to the existence of hidden confounders.
Specifically, we create datasets with hidden variables by generating
datasets with $k$ additional variables then remove them. We present
the results (for 30-node ER-8 graphs) on such datasets in Table~\ref{tab:hidden-confounders},
where our method also outperforms all baselines, even with more hidden
confounders.

\begin{table}[H]

\caption{Causal Discovery Performance under Hidden Confounders. The numbers
are \textit{mean \textpm{} standard error} over 5 random datasets
on 30-node ER-8 graphs.\label{tab:hidden-confounders}}

\begin{centering}
\resizebox{0.7\columnwidth}{!}{%
\begin{tabular}{ccccc}
\toprule 
Method \textbackslash{} $k$ & 1 & 2 & 3 & 4\tabularnewline
\midrule 
COSMO & $93.0\pm{16.2}$ & $133.6\pm{11.2}$ & $132.8\pm{7.1}$ & $136.4\pm{8.5}$\tabularnewline
DAGMA & $53.0\pm{11.9}$ & $77.6\pm{9.3}$ & $96.0\pm{7.6}$ & $124.0\pm{5.3}$\tabularnewline
NOCURL & $149.2\pm{7.7}$ & $159.8\pm{8.4}$ & $153.2\pm{5.4}$ & $158.4\pm{5.3}$\tabularnewline
NOTEARS & $185.2\pm{10.4}$ & $172.6\pm{5.7}$ & $173.2\pm{5.0}$ & $188.8\pm{4.4}$\tabularnewline
\midrule 
$\ours$ (Ours) & $29.4\pm{6.2}$ & $58.2\pm{4.9}$ & $82.2\pm{6.1}$ & $106.6\pm{2.5}$\tabularnewline
\bottomrule
\end{tabular}}
\par\end{centering}
\end{table}

\subsection{Results on nonlinear data with MLPs\label{subsec:nonlinear-mlp}}

In this section, we study the performance of our ALIAS method on nonlinear
causal models with neural networks, which are popular among gradient-based
methods \citep{Zheng_etal_18DAGs,Bello_etal_22dagma,Massidda_etal_2023Constraint}.
Specifically, similar to the GP data, we consider 10-node ER-4 graphs
with multiple-layer perceptron (MLP) causal mechanisms with standard
Gaussian noise, as used in NOTEARS and DAGMA. These MLPs have one
hidden layer with 100 units and sigmoid activation, and we generate
1,000 samples per dataset.

We employ the $\ours$ variant that uses GP regressor with the exact
configuration as in the GP experiments (Section~\ref{subsec:Nonlinear-GP}),
but now using the equal variance variant of the BIC score. This is
compared against gradient-based baselines that support MLPs as is.
These include NOTEARS, DAGMA, and COSMO, all of which model nonlinearity
using MLPs with the same configuration, namely one hidden layer of
10 units with sigmoid activation.

The results are reported in Table~\ref{tab:nonlinear-mlp}, showing
that despite the potential disadvantage of using GP regression for
MLP-based SEMs, our method still achieves the lowest SHD of near zero
compared to all three baselines, which are specifically designed to
model MLP data.

\begin{table}[H]
\caption{Causal discovery performance on nonlinear data with MLP model. The
data is generated with 10-node ER-4 graphs and MLP model with one
hidden layer of 100 units and sigmoid activation as causal mechanisms.
The performance metrics are Structural Hamming Distance (SHD), False
Detection Rate (FDR), and True Positive Rate (TPR). Lower SHD and
FDR values are preferable, while higher values are better for TPR.
The numbers are \textit{mean \textpm{} standard errors} over 5 independent
runs. Since the graphs are dense, we also study the effect of pruning
the output graphs.\label{tab:nonlinear-mlp}}

\begin{centering}
\resizebox{\columnwidth}{!}{%
\begin{tabular}{ccccccc}
\toprule 
 & \multicolumn{3}{c}{No Pruning} & \multicolumn{3}{c}{CAM Pruning}\tabularnewline
\midrule
Method & SHD ($\downarrow$) & FDR ($\downarrow$) & TPR ($\uparrow$) & SHD ($\downarrow$) & FDR ($\downarrow$) & TPR ($\uparrow$)\tabularnewline
\midrule
NOTEARS \citep{Zheng_etal_20Learning} & $11.0\pm1.2$ & $0.21\pm0.01$ & $0.93\pm0.04$ & $10.0\pm1.9$ & $0.11\pm0.03$ & $0.80\pm0.04$\tabularnewline
DAGMA \citep{Bello_etal_22dagma} & $\phantom{0}7.6\pm1.7$ & $0.09\pm0.03$ & $0.84\pm0.05$ & $\phantom{0}8.8\pm1.6$ & $0.07\pm0.03$ & $0.78\pm0.04$\tabularnewline
COSMO \citep{Massidda_etal_2023Constraint} & $\phantom{0}9.4\pm1.0$ & $0.14\pm0.01$ & $0.87\pm0.03$ & $9.8\pm1.2$ & $0.09\pm0.02$ & $0.79\pm0.02$\tabularnewline
\midrule 
$\ours$ (Ours) & $\mathbf{\phantom{0}0.8\pm0.4}$ & $\mathbf{0.01\pm0.01}$ & $\mathbf{0.99\pm0.01}$ & $\mathbf{\phantom{0}3.8\pm0.4}$ & $\mathbf{0.00\pm0.00}$ & $\mathbf{0.89\pm0.01}$\tabularnewline
\bottomrule
\end{tabular}}
\par\end{centering}
\end{table}

\subsection{Results on regular weight range $\mathcal{U}\left(\left[-2,-0.5\right]\cup\left[0.5,2\right]\right)$}

In addition to our experiments on linear-Gaussian data with the large
weight range $\mathcal{U}\left(\left[-5,-2\right]\cup\left[2,5\right]\right)$
in Figure~\ref{fig:linear-gaussian}, here we also consider the regular
range $\mathcal{U}\left(\left[-2,-0.5\right]\cup\left[0.5,2\right]\right)$
for the linear weights. In Figure~\ref{fig:both-ranges} we present
the causal discovery results for both weight ranges under varying
graph sizes. It can be seen that larger weights pose significant challenges
to several methods, including NOTEARS, NOCURL, RL-BIC, RCL-OG, and
CORL. Meanwhile, our method consistently identifies true DAG in all
cases for both weight ranges, signifying its robustness to the data
variance.

\begin{figure}[H]

\begin{centering}
\includegraphics[width=1\columnwidth]{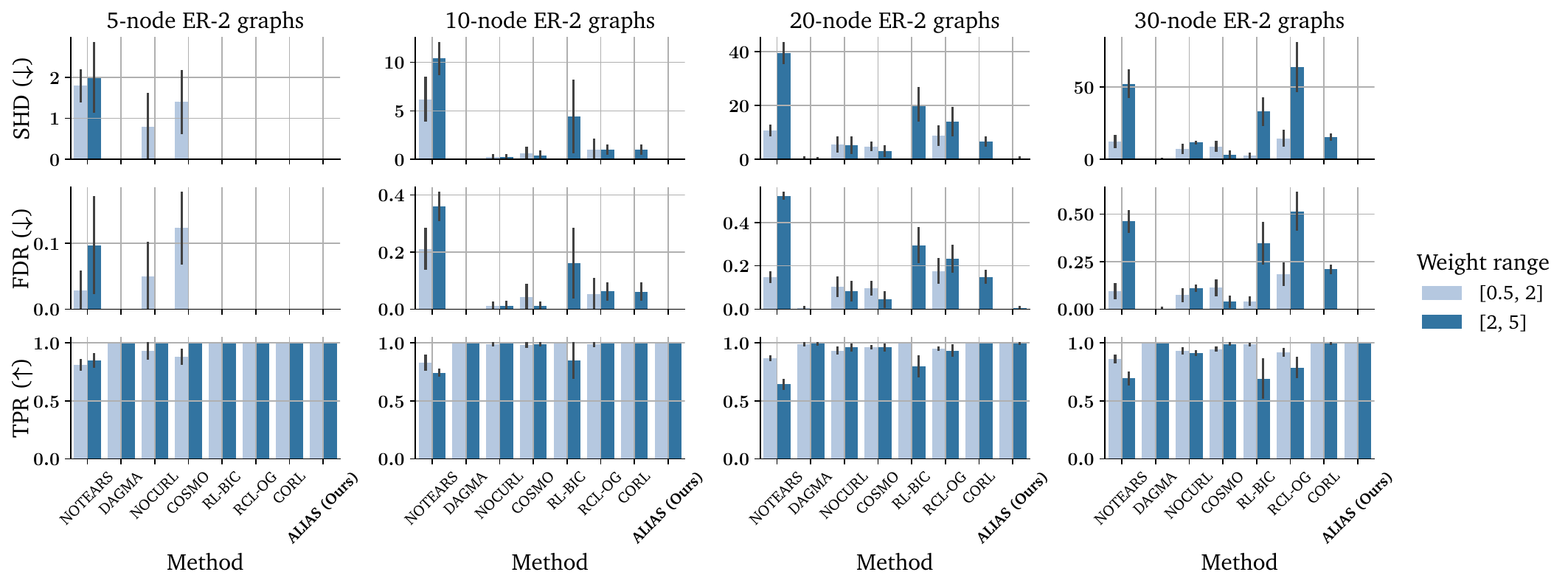}
\par\end{centering}
\caption{Causal Discovery Performance on linear-Gaussian data with small ($\mathcal{U}\left(\left[-2,-0.5\right]\cup\left[0.5,2\right]\right)$)
and large ($\mathcal{U}\left(\left[-5,-2\right]\cup\left[2,5\right]\right)$)
weight ranges. The error bars depict standard errors over 5 independent
runs.\label{fig:both-ranges}}

\end{figure}

\subsection{Comparison with Constraint-based methods}

In addition to the score-based baselines considered so far, in this
section, we also investigate the comparative performance of our $\ours$
method against constraint-based methods, which are also a prominent
approach in causal discovery. To be more specific, we consider the
classical PC method \citep{Spirtes_etal_00Causation}, as well as
more recent advances MARVEL \citep{Mokhtarian_etal_21Recursive} and
RSL \citep{Mokhtarian_etal_22Learning}. In this case, we emply gCastle's
implementation for PC and the official implementations provided at
\url{https://rcdpackage.com/} for RSL and MARVEL.

We evaluate these methods on linear-Gaussian datasets with varying
scales and densities, and the sample size is fixed to 1,000. The baselines
were configured according to their recommended settings, including
the Fisher's z test with significance levels of 0.05 for PC and $\frac{2}{d^{2}}$
for RSL and MARVEL, where $d$ is the number of nodes. Since constraint-based
methods may not orient all edges, for a fair evaluation, we compare
the undirected skeletons of the estimated and true graphs. The evaluation
metrics included SHD, precision, recall, and F1 scores. The results
presented in Table~\ref{tab:constraint-based} reveal that our method
demonstrates significantly higher accuracy than constraint-based approaches
in recovering the skeleton, achieving near-zero skeleton SHD across
all scenarios. In contrast, the baseline methods face considerable
challenges, particularly when applied to large and dense graphs.

\begin{table}[H]

\caption{Causal Discovery Performance in Comparison with Constraint-based Methods.
We compare the proposed method $\protect\ours$ with PC \citep{Spirtes_etal_00Causation},
MARVEL \citep{Mokhtarian_etal_21Recursive}, and RSL \citep{Mokhtarian_etal_22Learning}
on linear-Gaussian datasets with the regular weight range $\mathcal{U}\left(\left[-2,-0.5\right]\cup\left[0.5,2\right]\right)$.
The performance metrics are SHD (lower is better), Precision (higher
is better), Recall (higher is better), and $F_{1}$ score (higher
is better) between the skeletons of the estimated and true graphs.
The numbers are \textit{mean \textpm{} standard errors} over 5 independent
runs.\label{tab:constraint-based}}

\centering{}\resizebox{\columnwidth}{!}{%
\begin{tabular}{cccccc}
\toprule 
Data & Method & Skeleton SHD ($\downarrow$) & Skeleton Precision ($\uparrow$) & Skeleton Recall ($\uparrow$) & Skeleton $F_{1}$ ($\uparrow$)\tabularnewline
\midrule
\multirow{4}{*}{Sparse graphs (30-node ER-2)} & PC \citep{Spirtes_etal_00Causation} & $25.8\pm2.8$ & $0.88\pm0.03$ & $0.64\pm0.03$ & $0.74\pm0.03$\tabularnewline
 & MARVEL \citep{Mokhtarian_etal_21Recursive} & $15.6\pm2.1$ & $0.93\pm0.02$ & $0.79\pm0.03$ & $0.85\pm0.02$\tabularnewline
 & RSL \citep{Mokhtarian_etal_22Learning} & $13.8\pm1.5$ & $0.9\pm0.02$ & $0.86\pm0.01$ & $0.88\pm0.01$\tabularnewline
\cmidrule{2-6}
 & \textbf{ALIAS (Ours)} & $\mathbf{0.0\pm0.0}$ & $\mathbf{1.0\pm0.0}$ & $\mathbf{1.0\pm0.0}$ & \textbf{$\mathbf{1.0\pm0.0}$}\tabularnewline
\midrule
\multirow{4}{*}{Dense graphs (30-node ER-8)} & PC \citep{Spirtes_etal_00Causation} & $216.0\pm3.6$ & $0.65\pm0.02$ & $0.13\pm0.01$ & $0.21\pm0.02$\tabularnewline
 & MARVEL \citep{Mokhtarian_etal_21Recursive} & $221.8\pm3.8$ & $0.71\pm0.03$ & $0.05\pm0.0$ & $0.1\pm0.01$\tabularnewline
 & RSL \citep{Mokhtarian_etal_22Learning} & $185.4\pm5.6$ & $0.66\pm0.02$ & $0.39\pm0.01$ & $0.49\pm0.01$\tabularnewline
\cmidrule{2-6}
 & \textbf{ALIAS (Ours)} & $\mathbf{0.4\pm0.2}$ & $\mathbf{1.0\pm0.0}$ & $\mathbf{1.0\pm0.0}$ & $\mathbf{1.0\pm0.0}$\tabularnewline
\midrule
\multirow{4}{*}{Large graphs (200-node ER-2)} & PC \citep{Spirtes_etal_00Causation} & $216.6\pm10.5$ & $0.87\pm0.01$ & $0.57\pm0.02$ & $0.69\pm0.01$\tabularnewline
 & MARVEL \citep{Mokhtarian_etal_21Recursive} & $167.2\pm8.4$ & $0.91\pm0.01$ & $0.67\pm0.01$ & $0.77\pm0.01$\tabularnewline
 & RSL \citep{Mokhtarian_etal_22Learning} & $167.2\pm8.8$ & $0.87\pm0.01$ & $0.7\pm0.01$ & $0.78\pm0.01$\tabularnewline
\cmidrule{2-6}
 & \textbf{ALIAS (Ours)} & $\mathbf{0.8\pm0.6}$ & $\mathbf{1.0\pm0.0}$ & $\mathbf{1.0\pm0.0}$ & $\mathbf{1.0\pm0.0}$\tabularnewline
\bottomrule
\end{tabular}}
\end{table}

\subsection{On Varsortability}

\global\long\def\varsort{\texttt{Varsortability}}%

\global\long\def\sortnregress{\texttt{sortnregress}}%

Recently, it was suggested by \citet{Reisach_etall_21Beware} that
marginal variances share high dependencies with the topological ordering
of the true DAG for simple additive noise models. As such, sorting
the variables according to the variances may be sufficient to recover
the DAG, voiding the need for more sophisticated methods. More particularly,
the metric $\varsort\in\left[0,1\right]$ is proposed to assess the
correlation between the ordering of marginal data variances and the
causal ordering of the true DAG. Essentially, $\varsort$ calculates
the portion of directed paths in the ground-truth DAG that have increasing
marginal variances, and thus, if it is close to 1 then there is a
high chance that simply sorting the variances would reveal the true
DAG, and vice versa. It was shown that simple additive noise data
commonly employed in existing studies, starting with \citet{Zheng_etal_18DAGs},
exhibit high values of $\varsort$.

However, we show here that this is not entirely the case for our study.
Particularly, while the majority of existing studies employ \textit{sparse}
graphs for evaluation, where the densest graphs have an expected in-degree
of about 4 at most \citep{Zheng_etal_18DAGs,Yu_etal_2021Dags,Bello_etal_22dagma,Massidda_etal_2023Constraint},
in this study, we also consider much \textit{denser} graphs with an
expected in-degree of up to 8. This distinction is of high importance
because we find that $\texttt{Varsortability}$ decreases rapidly
with the graph density and reaches near-zero at ER-8 graphs (Table~\ref{tab:varsortability}),
indicating that our data settings is much more nontrivial, and cannot
be resolved simply by sorting the variances. Furthermore, since $\varsort$
is close to 0 in our data, one may expect that simply sorting the
nodes by \textit{decreasing} variances may work. In Table~\ref{tab:sortnregress}
we show that this is not the case, where the $\sortnregress$ algorithm
proposed in conjunction to $\varsort$ \citep{Reisach_etall_21Beware}---which
first sorts the nodes by marginal variances and then perform linear
feature selection to convert the ordering to a DAG---cannot recover
the true DAG regardless of the sorting direction. Nevertheless, under
this intricate scenarios, our method still robustly recovers the true
DAG with a near-zero SHD.

\begin{table}[h]
\caption{\textbf{$\mathbf{\protect\varsort}$ as function of graph density}.
We compute $\texttt{Varsortability}$ \citep{Reisach_etall_21Beware}
for linear-Gaussian data on 30-node ER graphs (Section.~\ref{subsec:Linear-Data})
with varying expected in-degrees. The numbers are \textit{mean \textpm{}
standard errors} over 100 simulations.\label{tab:varsortability}}

\begin{centering}
\begin{tabular}{cc}
\toprule 
\textbf{Expected in-degree} & \textbf{$\texttt{\textbf{Varsortability}}$}\tabularnewline
\midrule
1 & $0.97\pm0.01$\tabularnewline
2 & $0.87\pm0.03$\tabularnewline
4 & $0.35\pm0.06$\tabularnewline
6 & $0.05\pm0.01$\tabularnewline
8 & $0.01\pm0.00$\tabularnewline
\bottomrule
\end{tabular}
\par\end{centering}
\end{table}

\begin{table}[h]
\caption{\textbf{Comparison with $\protect\sortnregress$ on dense data}. We
consider linear-Gaussian datasets on dense 30-node ER-8 graphs. The
numbers are \textit{mean \textpm{} standard errors} over 5 simulations.
\label{tab:sortnregress} }

\begin{centering}
\begin{tabular}{cccc}
\toprule 
\textbf{Method} & \textbf{SHD ($\downarrow$)} & \textbf{FDR ($\downarrow$)} & \textbf{TPR ($\uparrow$)}\tabularnewline
\midrule
$\sortnregress$ (decreasing variances) & $360.2\pm\phantom{0}2.2$ & 0.97\textpm 0.01 & 0.04\textpm 0.01\tabularnewline
$\sortnregress$ (increasing variances) & $\phantom{0}88.2\pm17.7$ & 0.27\textpm 0.05 & 0.94\textpm 0.01\tabularnewline
\midrule 
\textbf{$\ours$ (ours)} & \textbf{$\mathbf{\phantom{0}\phantom{0}0.2\pm\phantom{0}0.2}$} & \textbf{$\mathbf{0.00\pm0.00}$} & $\mathbf{1.00\pm0.00}$\tabularnewline
\bottomrule
\end{tabular}
\par\end{centering}
\end{table}

\subsection{Effect of Data Standardization}

Given the issues raised by \citet{Reisach_etall_21Beware} as discussed
in the previous section, it is interesting to study the performance
of causal discovery methods under data standardization, where the
variance ordering information is completely removed.

\textbf{Linear-Gaussian data.} For linear-Gaussian data, the causal
graph is identifiable under the equal-noise variance assumption \citep{Peters_etal_2014Causal}.
Hence, standardizing our linear-Gaussian data node-wise would render
the noise variances unequal and reduce the causal model to a general
linear-Gaussian model, where the causal graph is known to be \textit{unrecoverable}
\citep{Spirtes_etal_00Causation}. As such, it is expected that causal
discovery performance on standardized linear data is worse than its
non-standardized counterparts. However, in Table.~\ref{tab:linear-standardized}
we show that our $\ours$ method can still significantly outperform
the baselines under this setting with the lowest SHD, which is only
half of that of the second-best method, highlighting $\ours$'s robustness
against data standardization.

\textbf{Nonlinear data.} On the contrary with linear causal models,
nonlinear additive noise models are generally identifiable regardless
of unequal noise variances \citep{Hoyer_etal_2008Nonlinear}, meaning
standardizing data should have minimal influence on the causal discovery
performance. Indeed, in Table.~\ref{tab:linear-standardized} we
show that our $\ours$ method still performs equally well when the
data is standardized compared to when the data is not standardized
(Table.~\ref{tab:nonlinear}) with a very\textbf{ }low SHD of only
around 1, which is much lower relative to other baselines.

\begin{table}[h]

\caption{\textbf{Causal Discovery performance on standardized data}. We consider
linear-Gaussian data and nonlinear data with Gaussian processes. Observed
data is standardized dimension-wise to remove variance information.
For $\protect\ours$, we use the non-equal variances\textbf{ }BIC
version to reflect the data condition. The numbers are \textit{mean
\textpm{} standard errors} over 5 independent runs. \textbf{Bold}:
best performance, \uline{underline}: second-best performance.\label{tab:linear-standardized}}

\centering{}\resizebox{\columnwidth}{!}{%
\begin{tabular}{ccccccc}
\toprule 
 & \multicolumn{3}{c}{\textbf{Linear data (10-node ER-2 graphs)}} & \multicolumn{3}{c}{\textbf{Nonlinear data (10-node ER-4 graphs)}}\tabularnewline
\midrule 
\textbf{Method} & \textbf{SHD ($\downarrow$)} & \textbf{FDR ($\downarrow$)} & \textbf{TPR ($\uparrow$)} & \textbf{SHD ($\downarrow$)} & \textbf{FDR ($\downarrow$)} & \textbf{TPR ($\uparrow$)}\tabularnewline
\midrule
NOTEARS \citep{Zheng_etal_20Learning} & $\underline{10.2\pm2.0}$ & $\underline{0.39\pm0.04}$ & $\mathbf{0.84\pm0.02}$ & $7.8\pm2.3$ & $0.14\pm0.07$ & $0.91\pm0.01$\tabularnewline
DAGMA \citep{Bello_etal_22dagma} & $17.4\pm2.8$ & $0.70\pm0.08$ & $0.39\pm0.07$ & $29.8\pm5.5$ & $0.4\pm0.16$ & $0.28\pm0.15$\tabularnewline
COSMO \citep{Massidda_etal_2023Constraint} & $16.4\pm1.9$ & $0.74\pm0.06$ & $0.29\pm0.06$ & $30.8\pm4.9$ & $0.5\pm0.16$ & $0.27\pm0.12$\tabularnewline
\midrule 
\textbf{$\ours$ (ours)} & $\mathbf{\phantom{0}6.0\pm1.8}$ & $\mathbf{0.30\pm0.07}$ & $\underline{0.79\pm0.05}$ & $\mathbf{1.4\pm0.9}$ & $\mathbf{0.0\pm0.01}$ & $\mathbf{0.97\pm0.03}$\tabularnewline
\bottomrule
\end{tabular}}
\end{table}

\end{document}